\documentclass{article}

 \usepackage[main, final]{neurips_2026}

\usepackage[utf8]{inputenc} 
\usepackage[T1]{fontenc}    
\usepackage{hyperref}       
\usepackage{url}            
\usepackage{booktabs}       
\usepackage{amsfonts}       
\usepackage{nicefrac}       
\usepackage{microtype}      
\usepackage{xcolor}         
\usepackage{amsmath}
\usepackage[utf8]{inputenc} 
\usepackage[T1]{fontenc}    
\usepackage{hyperref}       
\usepackage{adjustbox}
\usepackage{url}            
\usepackage{booktabs}       
\usepackage{algorithm}
\usepackage{algorithmic}
\usepackage{amsthm}
\newtheorem{proposition}{Proposition}
\newtheorem{theorem}{Theorem}

\usepackage{amsfonts}       
\usepackage{nicefrac}       
\usepackage{microtype}      
\usepackage{xcolor}         
\usepackage{amsmath}
\usepackage{amssymb}
\usepackage{subcaption}
\usepackage{siunitx}
\usepackage{multirow}
\usepackage{wrapfig}
\usepackage{array}
\usepackage{amsthm}
\usepackage{enumitem}
\usepackage{amsthm}
\usepackage{booktabs}   
\usepackage{makecell}   

\newtheorem{corollary}{Corollary}

\usepackage{caption}

\title{Semantic Optimal Transport for Sparse Autoencoder Feature Matching and Circuit Compression}

\author{%
  \textbf{Tue M. Cao}$^{\dagger,\ddagger}$
  \hspace{0.7cm}
  \textbf{Nguyen Do}$^{\dagger,\ddagger}$
  \hspace{0.7cm}
  \textbf{My T. Thai}$^{\dagger,*}$\\
  $^{\dagger}$University of Florida, FL, USA\\
  $^{\ddagger}$Equal contribution
  \qquad
  $^{*}$Correspondence to: \texttt{mythai@cise.ufl.edu}
}

\begin{document}

\maketitle

\begin{abstract}
\label{abstract}

Sparse autoencoders (SAEs) have become a central tool for interpreting language models. However, two key SAE analyses that remain difficult to scale are (1) matching semantically similar features across multi-layers and (2) compressing large feature circuits into interpretable supernodes. Although these have been treated as separate problems, we show that both are instances of a more fundamental challenge, which we frame as the estimation of semantic distances between SAE features that lie on different activation manifolds. We introduce a distributional framework for this problem, in which each feature is represented not by a single decoder vector like in the literature, but by an activation-weighted distribution over the hidden states that express it. By projecting these distributions into a shared reference space and comparing them with Wasserstein distance, our method provides a unified semantic metric for cross-layer feature comparison. 
We prove that our representation is invariant to activation rescaling, stable under perturbations, and recovers true matches under finite-sample margin conditions. Empirically, our method outperforms decoder-vector and LLM-based baselines and captures subtle functional distinctions between related features. Notably, our method compresses large feature circuits into interpretable supernodes. Code available at \url{https://github.com/tue147/semantic-optimal-transport}.

\end{abstract}

\section{Introduction}
\label{sec:introduction}
Sparse Autoencoder (SAE) has shown great promises in understanding large language models (LLM) \cite{cunningham2023sparse, topk} by breaking down the latent space into interpretable features. Many tools have been proposed, such as Feature Matching to understand feature persistence and evolution between layers \cite{mech_permute, featflow}, or tracing the ``thoughts" of the LLM using Feature Circuits \cite{feature_circuit, transcoder_circuit}. 

A standard tool of inspecting the thinking of a model is Feature Circuit \cite{feature_circuit}, a causal graph that connects features from SAEs spanned across all layers of a LLM. 
However, feature circuits are difficult to inspect at scale. A single circuit may contain hundreds feature nodes, each with distinct activation patterns and functional roles \cite{on_biology}. Manually inspecting every node is impractical, while overly pruning can remove important intermediate mechanisms \cite{feature_circuit}. In practice, researchers often compress circuits by manually grouping semantically similar feature nodes into higher-level ``supernodes" \cite{on_biology}. This makes the circuit easier to understand, but it relies on human labeling and hence not scalable. This motivates the need for an automatic method that can group SAE features by their semantic roles.

A separate line of works is Feature Matching \cite{mech_permute, featflow}, of which the goal is to understand feature transform across layers. These methods aim to match feature pairs with similar meaning from different layers, showing how the feature persist or evolve across layers. A major flaws of the existing works is that their matching rely solely on feature decoder vectors of the SAEs trained on two different layers. This flaw is critical because the decoder vectors learned to reconstruct two different distributions. Therefore, their methods only work well on matching consecutive layers \cite{mech_permute} where the distributions are similar; and are inaccurate when the matching across dissimilar distributions, such as matching across multiple layers.


\begin{figure}
    \centering
    \vspace{-6mm}
    \includegraphics[width=1\linewidth]{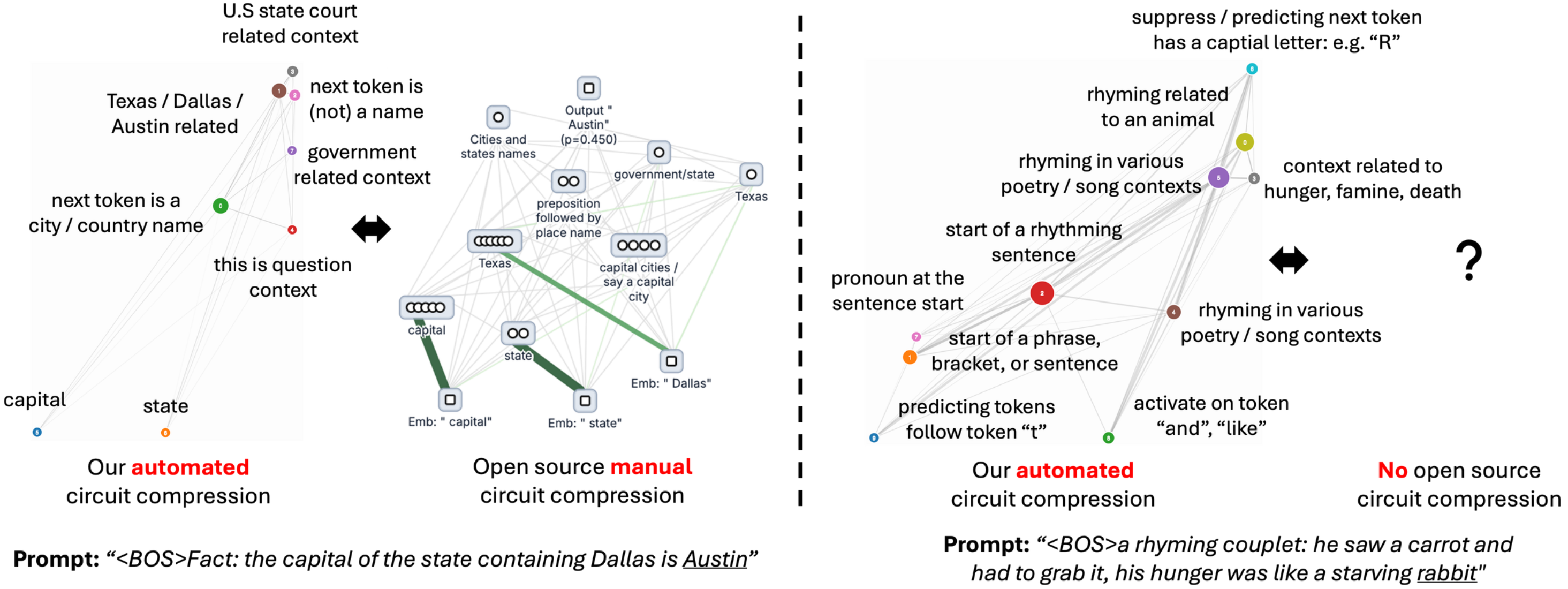}
    \caption{This figure compares our automated circuit compression and open source manual compression \cite{neuronpedia}. Our method identifies similar supernodes as open source version for ``Dallas" circuit and scalable to other circuits that have no open source implementation. More details in Section \ref{application:circuit_compression}.}
    \vspace{-5mm}
    \label{fig:circuit_compression_comparison}
\end{figure}

Although these two have been treated as separatedly, we identify that the core problem of both feature matching and circuit compression is approximating ``semantic distance" of features across multi-manifolds. The ``semantic distances" show how close the meaning of the features are, and thus, are the direct solution to both feature matching and circuit compressing. However, there are two key challenges to solve this problem. The first challenge is the different nature of the manifolds, where each manifold can have different support set and metrics, thus hindering the computation of across manifolds. The second problem is how to embed the ``semantic" of features into the manifold. The main assumption of the field is to embed a feature by one vector (e.g. decoder vectors of SAEs) \cite{mech_permute, monosemanticity}. However, this approach lose information of the feature meaning and cannot be used with multi-manifolds. A comprehensive discussion of related literature is provided in the related works (Appendix~\ref{sec:related_works_extend}).

To address these challenges, we abandon the standard vector-centric paradigm and instead frame feature comparison as an Optimal Transport (OT) problem \cite{beier2025joint, imfeldtransformer, melnyk2024distributional}. By representing feature as activation-weighted distribution rather than one vector, we prevent the loss of semantic information that occurs when collapsing complex feature behaviors into a single point. Furthermore, OT provides a mathematically rigorous foundation for cross-manifold mapping, allowing us to estimate semantic distances between features even when they reside in different latent spaces. Hence, this distributional approach naturally facilitates both accurate cross-layer feature matching and feature circuit compression.
Our theoretical analysis and empirical results show that:
\begin{itemize}[itemsep=4pt, topsep=2pt, parsep=0pt]


    \item \textbf{Theoretical guarantees:} We prove that this distributional representation is invariant to activation rescaling and its Wasserstein distances are robust to measure perturbations. Furthermore, we establish explicit margin conditions guaranteeing that empirical top-$K$ samples perfectly recover the true feature matches and semantic groups.

    \item \textbf{Cross Manifold Feature Matching:} Our method is more accurate than existing methods as well as LLM in matching features for both consecutive and across many layers. Notably, ours can identify the subtle differences in the meaning of many ``digit addition" features \cite{on_biology} while other failed.
    
    \item \textbf{Automatic circuit compression:} We can \textit{automatically} compress feature circuits into intuitive supernodes, enhancing circuits' interpretability. We apply to automatically find modular features groups that have the same functional role in across many circuits.
\end{itemize}

\section{Background and Problem Formulation}
\label{sec:problem}

\textbf{Sparse Autoencoders.}
Let $x^{(\ell)}\in\mathbb R^{d_\ell}$ denote the hidden state of a language model
at layer $\ell$. An SAE trained on layer $\ell$ represents $x^{(\ell)}$ using an
overcomplete sparse feature vector $a^{(\ell)}(x^{(\ell)})\in\mathbb R^{F_\ell}, 
F_\ell\gg d_\ell$. The encoder and decoder are given by
\begin{equation}
    a^{(\ell)}(x^{(\ell)})
=
\sigma\!\left(
W_{\mathrm{enc}}^{(\ell)}
\left(x^{(\ell)}-b_{\mathrm{enc}}^{(\ell)}\right)
\right),
\qquad
\widehat{x}^{(\ell)}
=
W_{\mathrm{dec}}^{(\ell)}
a^{(\ell)}(x^{(\ell)})
+
b_{\mathrm{dec}}^{(\ell)}.
\end{equation}
Here, $W_{\mathrm{enc}}^{(\ell)}\in\mathbb R^{F_\ell\times d_\ell}, 
W_{\mathrm{dec}}^{(\ell)}\in\mathbb R^{d_\ell\times F_\ell}$
are the encoder and decoder weights, and $b_{\mathrm{enc}}^{(\ell)}\in\mathbb R^{d_\ell}, 
b_{\mathrm{dec}}^{(\ell)}\in\mathbb R^{d_\ell}$
are the encoder and decoder biases. The function $\sigma$ is a
sparsity-inducing activation function, such as TopK \cite{topk}. We denote $a_i^{(\ell)}(x^{(\ell)})$ for the activation of feature $i\in\mathcal F^{(\ell)}$ on hidden state $x^{(\ell)}$. Note that, throughout this work, we consider nonnegative SAE activations $a_i^{(\ell)}(x^{(\ell)})\ge 0.$

\textbf{Empirical activation distributions.}
Let $\mathcal S_T
=
\left\{
(x_t^{(1)},\ldots,x_t^{(L)})
\right\}_{t=1}^{T}$ denote a corpus of $T$ token positions, where $x_t^{(\ell)}$ is the hidden state
of token position $t$ at layer $\ell$ and $L$ is the total number of layers in the language model. For a feature
$i\in\mathcal F^{(\ell)}$, define its top-$K$ activation index set by $\mathcal I_{i,T}^{(\ell),K}
=
\operatorname{TopK}_{t\in[T]}
a_i^{(\ell)}(x_t^{(\ell)})$. Throughout the paper, $\operatorname{TopK}$ denotes a deterministic selector of $K$ best candidates. For each selected token position
$t\in\mathcal I_{i,T}^{(\ell),K}$, define the normalized activation weight $\widehat w_{i,t}^{(\ell)}
=
a_i^{(\ell)}(x_t^{(\ell)})/
\sum_{s\in\mathcal I_{i,T}^{(\ell),K}}
a_i^{(\ell)}(x_s^{(\ell)}).$ The activation-weighted empirical distribution of feature $i$ at layer $\ell$ is $\widehat{\mu}_{i,T}^{(\ell)}
=
\sum_{t\in\mathcal I_{i,T}^{(\ell),K}}
\widehat w_{i,t}^{(\ell)}
\delta_{x_t^{(\ell)}}$.
Here, $\delta_{x_t^{(\ell)}}$ is the Dirac measure at the hidden state
$x_t^{(\ell)}$. Thus, $\widehat{\mu}_{i,T}^{(\ell)}$ records both the token
positions where feature $i$ activates most strongly and the relative activation
strengths over those positions.

\textbf{Problem Definition: Cross-Manifold Semantic Distance.}
\label{sec:problem_definition}
A fundamental problem is to estimate the semantic distance between SAE features, once the distance is available, downstream tasks such as: cross-layer feature matching can be solved by selecting nearest neighbors, or circuit compression can be solved by grouping circuit nodes are close. 

However, there are two key challenges: (1) features from different layers reside on distinct manifolds the meaning, compute the distance using decoder vectors trained of different manifolds \cite{mech_permute, featflow} would yield inaccurate result, (2) many works \cite{cunningham2023sparse, sae_steering, featflow} consider the semantic of the features as one vector (e.g. decoder vectors), this approach loses the information of activation context. 

To solve these challenges, we first define the problem as computing the \textit{cross-manifold} semantic distance of \textit{probability distributions} of features. The following formulation requires approximating true cross-manifold distance and preserving the semantic activation context of the features:

Let $\{\widehat{\mu}_{i,T}^{(\ell)} : i\in\mathcal F^{(\ell)}\}$ denote the empirical activation distributions of SAE features at layer $\ell$ lying on a feature–semantic manifold $\mathcal M^{(\ell)}\subseteq\mathcal P(\mathcal X^{(\ell)})$; where $\mathcal{P}(\mathcal{X}^{(\ell)})$ denotes the space of probability distributions over the hidden-state space $\mathcal X^{(\ell)}\subseteq\mathbb R^{d_\ell}$. 
Ideally, there exists a true but unknown semantic distance between any two features, which we refer to as the oracle distance:
$
D^\star_{\ell q}(i,j),
\:
i\in\mathcal F^{(\ell)},\: j\in\mathcal F^{(q)}.
$
This oracle distance is small when two features play similar semantic or
functional roles and large when their activation behavior differs. Since
$D^\star_{\ell q}$ is not directly observable, we seek a computable discrepancy 
$
\widehat D_{\ell q}:
\mathcal M^{(\ell)}\times\mathcal M^{(q)}
\to
\mathbb R_{\ge 0}
$
that approximates this oracle semantic distance. Conceptually, this corresponds
to minimizing:
\begin{equation}
\min_{\widehat D}
\sum_{\ell,q=1}^{L}
\sum_{i\in\mathcal F^{(\ell)}}
\sum_{j\in\mathcal F^{(q)}}
\omega_{ij}^{(\ell,q)}
\left[
\widehat D_{\ell q}
\left(
\widehat\mu_{i,T}^{(\ell)},
\widehat\mu_{j,T}^{(q)}
\right)
-
D^\star_{\ell q}(i,j)
\right]^2,
\end{equation}
where $\omega_{ij}^{(\ell,q)}\ge 0$ optionally weights feature pairs of interest.

\section{Methodology}
\label{sec:method}


Given the formulation in Section \ref{sec:problem_definition}, to compute distance across manifolds, we first project these layer-specific empirical distributions onto a shared reference space $\mathcal R$. Furthermore, to preserve semantic information, we compute the semantic distance between the features' empirical distributions using the Wasserstein distance. This unified pairwise distance then serves as the foundation for downstream tasks, including cross-layer feature matching and circuit compression.

\subsection{Shared Reference Space and Wasserstein Distance}
\label{sec:shared_reference_and_lifting}

To make feature activation distributions comparable across layers, we assign each token position a representation in a shared reference space $\mathcal R$. 
For a feature $i\in\mathcal F^{(\ell)}$ with top-$K$ index set
$\mathcal I_{i,T}^{(\ell),K}$ and normalized activation weights
$\widehat w_{i,t}^{(\ell)}$, we define its projected feature distribution $\widehat\mu_{i,T}^{(\ell\to\mathcal R)}
=
\sum_{t\in\mathcal I_{i,T}^{(\ell),K}}
\widehat w_{i,t}^{(\ell)}
\delta_{z_t^{(\mathcal R)}}$ in
the shared reference space $\mathcal R$ by replacing each layer-specific support
point $x_t^{(\ell)}\in\mathcal X^{(\ell)}$ with the shared-space representation
$z_t^{(\mathcal R)}\in\mathcal R$.
This new empirical distribution on shared space still defines the feature's activation profile: it chooses the contexts and their relative weights. The practical choice of reference space $\mathcal{R}$ are detailed in Section \ref{sec:method_matching_compression}.
For readability, when comparing a target feature $i\in\mathcal F^{(A)}$ with a
source feature $j\in\mathcal F^{(B)}$, we write
$\widehat\mu_{i,T}^{(A\to\mathcal R)}$ for the projected target distribution and
$\widehat\nu_{j,T}^{(B\to\mathcal R)}$ for the projected source distribution.
\label{sec:wasserstein_feature_geometry}

Once two features are represented as probability distributions over the shared
reference space $\mathcal R$, we compare them using OT. Let
$c:\mathcal R\times\mathcal R\to\mathbb R_{\ge0}$ be a generic ground cost on
$\mathcal R$. For a target feature $i\in\mathcal F^{(A)}$ and a source feature
$j\in\mathcal F^{(B)}$, we define their semantic distance as:
\begin{equation}
\widehat D_{\mathcal R}(i,j)
=
W_c\!\left(
\widehat\mu_{i,T}^{(A\to\mathcal R)},
\widehat\nu_{j,T}^{(B\to\mathcal R)}
\right),
\qquad
W_c(\mu,\nu)
=
\inf_{\gamma\in\Pi(\mu,\nu)}
\int_{\mathcal R\times\mathcal R}
c(z,z')\,d\gamma(z,z').
\end{equation}
Here, $\Pi(\mu,\nu)$ is the set of couplings between $\mu$ and $\nu$. The ground
cost $c(z,z')$ measures how expensive it is to match two selected context
representations $z,z'\in\mathcal R$, while the coupling $\gamma$ specifies how
probability mass from one feature distribution is transported to the other.

For the empirical distributions used in our method, this Wasserstein problem
becomes a finite discrete transport problem. For a fixed feature pair $(i,j)$,
the pairwise cost matrix over their selected top-$K$ contexts has entries
$C_{pq}^{(i,j)}=c(z_p^{(\mathcal R)},z_q^{(\mathcal R)})$, where
$p\in\mathcal I_{i,T}^{(A),K}$ indexes a selected context of the target feature
and $q\in\mathcal I_{j,T}^{(B),K}$ indexes a selected context of the source
feature. A discrete transport plan is a nonnegative matrix whose row and column
sums match the normalized activation weights:
$
\Pi_{ij}
=
\left\{
\pi\ge 0:
\sum_{q\in\mathcal I_{j,T}^{(B),K}}\pi_{pq}
=
\widehat w_{i,p}^{(A)}
\ \forall p\in\mathcal I_{i,T}^{(A),K},
\quad
\sum_{p\in\mathcal I_{i,T}^{(A),K}}\pi_{pq}
=
\widehat w_{j,q}^{(B)}
\ \forall q\in\mathcal I_{j,T}^{(B),K}
\right\}.
$
Therefore, the empirical Wasserstein feature distance is
$
\widehat D_{\mathcal R}(i,j)
=
\min_{\pi\in\Pi_{ij}}
\sum_{p\in\mathcal I_{i,T}^{(A),K}}
\sum_{q\in\mathcal I_{j,T}^{(B),K}}
\pi_{pq}
C_{pq}^{(i,j)}.
$
This objective quantifies the minimum cost required to align the activation-weighted supports of the two features in $\mathcal R$. It serves as a purely geometric measure of semantic similarity, however, our overall framework remains entirely agnostic to the specific choice of the ground cost $c$.
In Section~\ref{sec:theory_stability}, we analyze the stability of the Wasserstein score for a general cost $c$ satisfying coordinate-wise Lipschitz conditions. Appendix~\ref{app:metric_specific_stability} instantiates this general result for several concrete score choices and compares their stability constants.

\subsection{Feature Matching and Circuit Compression}
\label{sec:method_matching_compression}
We provide our choice of reference space $\mathcal{R}$ for feature matching and circuit compression problem, the full procedure for each application is provided in Algorithm \ref{alg:matching} and \ref{alg:compression} respectively:

\textbf{Cross-layer feature matching:} Given a target feature set $\mathcal{F}^{(A)}$ at layer $A$ and source feature set $\mathcal{F}^{(B)}$ at layer $B$, we match each target feature to its nearest source feature:
\begin{equation}
    \widehat M_{A\to B}(i)
=
\arg\min_{j\in\mathcal F^{(B)}}
\widehat D_{\mathcal R}(i,j).
\end{equation}
In practice, we choose the reference space $\mathcal{R}$ as the space of target layer $A$, where for each activation token of feature at layer $B$, we extract the representation at the same token in layer $A$. This method leverages the natural alignment via tokens between the layers without any computational overhead.

\textbf{Circuit compression:} Let $G=(V,E)$ be a feature circuit, where each node $u=(\ell,i)\in V$ corresponds to feature $i$ at layer $\ell$. We want to find a set of supernodes $\mathbb{C}=\{\mathcal C_1,\ldots,\mathcal C_m\}$, each supernode contains feature nodes of $V$, and $C_i\cap C_j=\emptyset, \: i\neq j$. We choose the partition by minimizing within-supernode semantic distance:
\begin{equation}
    {\widehat{\mathbb{C}}}
= \arg\min_{\mathbb C}
\sum_{\mathcal C\in\mathbb C}
\frac{1}{|\mathcal C|(|\mathcal C|-1)}
\sum_{u,v\in \mathcal C, u \neq v}
\widehat D_{\mathcal{R}}(u,v).
\end{equation}
In practice, this objective can be optimized with standard clustering procedures using the pairwise distance matrix over circuit nodes. We choose the shared space $\mathcal{R}$ as the concatenation of hidden state of all layers. Specifically,
for a feature activates on token $t$, the representation on $\mathcal{R}$ is 
$
z_t^{(\mathcal R)}
=
\operatorname{Concat}\!\left(x_t^{(1)},x_t^{(2)},\ldots,x_t^{(L)}\right).
$ 
This concatenation help the $\mathcal{R}$ to preserve the semantic across layers, leading to more accurate circuit compression.

\section{Theoretical Analysis}
\label{sec:theory}

We establish three theoretical guarantees for our framework (proofs in Appendix~\ref{sec:additional_theory}). First, the projected feature distribution is strictly invariant to feature activation scaling. Second, the Wasserstein distances are stable under measure perturbations, with variations bounded by the Lipschitz constants of the ground cost, justifying the stability of the distances. Finally, we provide margin-based recovery conditions, proving that true feature matches and circuit compression are perfectly preserved whenever the relevant Wasserstein or Voronoi margin exceeds the estimation error.

\subsection{Representation and Score Stability}
\label{sec:theory_stability}

Although features are compared in the shared reference space $\mathcal R$, their projected distributions are constructed from layer-specific SAE activations. The representation should therefore depend on which contexts a feature selects and their relative activation weights, not on the arbitrary numerical scale of the activations.

\begin{proposition}[Activation rescaling invariance]
\label{prop:activation_rescaling_invariance}
Let $i\in\mathcal F^{(A)}$ have activation function $a_i^{(A)}:\mathcal X^{(A)}\to\mathbb R_{\ge0}$. For any $\alpha>0$, define $\widetilde a_i^{(A)}(x)=\alpha a_i^{(A)}(x)$. Under deterministic $\operatorname{TopK}$ selection, $\widetilde\mu_{i,T}^{(A\to\mathcal R)}=\widehat\mu_{i,T}^{(A\to\mathcal R)}$. Consequently, for any source feature $j\in\mathcal F^{(B)}$ and any ground cost $c:\mathcal R\times\mathcal R\to\mathbb R_{\ge0}$, $W_c(\widetilde\mu_{i,T}^{(A\to\mathcal R)},\widehat\nu_{j,T}^{(B\to\mathcal R)})=W_c(\widehat\mu_{i,T}^{(A\to\mathcal R)},\widehat\nu_{j,T}^{(B\to\mathcal R)})$.
\end{proposition}

This proposition guarantees our representation is invariant to global activation scales, which naturally vary across layers due to differing hidden-state norms. Because our construction relies entirely on the relative ranking and proportional distribution of activations, it is unaffected by these arbitrary scalar shifts. Consequently, the method compares where a feature activates and how its activation mass is distributed across those contexts, rather than comparing raw activation magnitudes.

Let $\rho$ be a base metric on $\mathcal R$. For
$\mu,\widetilde\mu,\nu,\widetilde\nu\in\mathcal P(\mathcal R)$, define
$\mathcal S_\mu=\operatorname{supp}(\mu)\cup\operatorname{supp}(\widetilde\mu)$
and
$\mathcal S_\nu=\operatorname{supp}(\nu)\cup\operatorname{supp}(\widetilde\nu)$.
We say that $c:\mathcal R\times\mathcal R\to\mathbb R_{\ge0}$ is
coordinate-wise Lipschitz on the relevant support if there exist constants
$L_x,L_y<\infty$ such that
$|c(z,z')-c(\bar z,z')|\le L_x\rho(z,\bar z)$ for all
$z,\bar z\in\mathcal S_\mu$ and $z'\in\mathcal S_\nu$, and
$|c(z,z')-c(z,\bar z')|\le L_y\rho(z',\bar z')$ for all
$z\in\mathcal S_\mu$ and $z',\bar z'\in\mathcal S_\nu$.

\begin{theorem}[Stability of Wasserstein feature scores]
\label{thm:wasserstein_stability}
Let $c:\mathcal R\times\mathcal R\to\mathbb R_{\ge0}$ be coordinate-wise Lipschitz on the relevant support with constants $L_x,L_y$. Then $|W_c(\mu,\nu)-W_c(\widetilde\mu,\widetilde\nu)|\le L_x W_\rho(\mu,\widetilde\mu)+L_y W_\rho(\nu,\widetilde\nu)$.
\end{theorem}

Theorem~\ref{thm:wasserstein_stability} guarantees that the feature distance is stable under smooth distributional shifts. This robustness is critical because top-$K$ activation sets naturally fluctuate across corpora, SAE retrainings, or representation noise. The theorem ensures the Wasserstein score changes substantially only if a nontrivial probability mass moves significantly within $\mathcal R$, or under a highly sensitive ground cost. Thus, our method removes the need for exact token overlap, requiring only that activation distributions remain proximate in the shared reference geometry. Metric-specific constants (Euclidean, Mahalanobis, cosine, flow-aware) are detailed in Appendix~\ref{app:metric_specific_stability}.

\subsection{Margin-Based Recovery of Feature Matching and Circuit Compressing}
\label{sec:theory_recovery}

Score stability alone does not guarantee that the final discrete decision is unchanged. A feature match can change if the best and second-best source features have nearly identical scores. We therefore state margin conditions under which stable scores imply stable matching or grouping decisions. For target feature $i\in\mathcal F^{(A)}$ and source feature $j\in\mathcal F^{(B)}$, define $D_{\mathcal R}(i,j)=W_c(\mu_i^{(A\to\mathcal R)},\nu_j^{(B\to\mathcal R)})$ and $\widehat D_{\mathcal R}(i,j)=W_c(\widehat\mu_{i,T}^{(A\to\mathcal R)},\widehat\nu_{j,T}^{(B\to\mathcal R)})$. Let $j_i^\star=\arg\min_{j\in\mathcal F^{(B)}}D_{\mathcal R}(i,j)$ and $\Delta_i=\min_{j\neq j_i^\star}[D_{\mathcal R}(i,j)-D_{\mathcal R}(i,j_i^\star)]$.

\begin{theorem}[Matching recovery under Wasserstein margin]
\label{thm:matching_recovery}
Fix $i\in\mathcal F^{(A)}$ and suppose $j_i^\star$ is unique with $\Delta_i>0$. If $\sup_{j\in\mathcal F^{(B)}}|\widehat D_{\mathcal R}(i,j)-D_{\mathcal R}(i,j)|\le\epsilon_i$ and $\Delta_i>2\epsilon_i$, then $\arg\min_{j\in\mathcal F^{(B)}}\widehat D_{\mathcal R}(i,j)=j_i^\star$.
\end{theorem}

This theorem ensures the matching decision are reliable and not merely artifacts of representation noise by converting score stability into matching stability. It introduces the matching margin $\Delta_i$, which measures how much closer the true source is compared to the nearest competing source. If empirical estimation perturbs every score by at most $\epsilon_i$, the worst-case scenario shrinks this gap by at most $2\epsilon_i$. Consequently, whenever $\Delta_i > 2\epsilon_i$, the empirical nearest-neighbor rule is mathematically guaranteed to preserve the true match. In practice, this allows us to confidently accept large-margin matches as highly robust to finite-sample perturbations, while identifying and flagging inherently unstable matches.


For circuit compression, each semantic group can be represented by a semantic
center in the shared reference space. A semantic center is a representative point
of one semantic cell or supernode. Let
$\xi_1,\ldots,\xi_m\in\mathcal R$ be semantic centers, one for each supernode.
For a projected distribution $\mu\in\mathcal P(\mathcal R)$, define the
assignment score
$\mathcal Q_k(\mu)=\mathbb E_{z\sim\mu}[\rho(z,\xi_k)]$.
For feature $i$, let
$k_i^\star=\arg\min_{k\in[m]}\mathcal Q_k(\mu_i^{(A\to\mathcal R)})$,
$\widehat k_{i,T}=\arg\min_{k\in[m]}\mathcal Q_k(\widehat\mu_{i,T}^{(A\to\mathcal R)})$,
and
$\Gamma_i=\min_{r\neq k_i^\star}
[\mathcal Q_r(\mu_i^{(A\to\mathcal R)})
-
\mathcal Q_{k_i^\star}(\mu_i^{(A\to\mathcal R)})]$.

\begin{theorem}[Semantic supernode recovery]
\label{thm:voronoi_recovery}
Suppose the population semantic assignment $k_i^\star$ is unique with margin
$\Gamma_i>0$. If $W_\rho\!\left(
\widehat\mu_{i,T}^{(A\to\mathcal R)},
\mu_i^{(A\to\mathcal R)}
\right)
< \Gamma_i/2, $
then the empirical semantic assignment recovers the population assignment:
$\widehat k_{i,T}=k_i^\star$.
\end{theorem}


To confidently compress complex feature circuits into semantic supernodes, we need guarantees that the grouping decision is stable under finite-sample noise. Theorem~\ref{thm:voronoi_recovery} provides such a guarantee through the Voronoi margin $\Gamma_i$, which measures the assignment-score gap between the correct semantic center and the nearest competing semantic center in $\mathcal R$. This margin induces a recovery radius $\Gamma_i/2$: if the empirical projected distribution satisfies
$W_\rho(\widehat\mu_{i,T}^{(A\to\mathcal R)},\mu_i^{(A\to\mathcal R)})<\Gamma_i/2$, then the empirical semantic assignment is guaranteed to match the population assignment. Thus, large-margin features can be merged into supernodes with confidence, while small-margin features are close to semantic boundaries and should be left ungrouped or marked as uncertain. Moreover, this recovery condition requires the empirical projected distribution
to be sufficiently close to its population counterpart, which naturally raises
the question of how large the sample size $T$ must be for the condition to hold
with high probability. To answer this question, we refer readers to
Appendix~\ref{app:voronoi_recovery}
(Table~\ref{tab:voronoi_convergence_rates}) for high-probability recovery
variants and sufficient sample complexity guarantees, and to
Appendix~\ref{app:matching_recovery} for combined distribution-and-cost
perturbation bounds.

\section{Experiments}
\label{exps}
We conduct the experiments on open source JumpReLU SAEs \cite{gemma_scope} trained on the output of MLP block of Gemma-2-2b \cite{gemma_2} and our own TopK SAEs \cite{topk} trained on Residual stream of GPT2 \cite{gpt2}, the training details are in Appendix \ref{sec:exp_details}.
We implement the following baselines, implementation details are in Appendix \ref{sec:baseline_implementation}:
\begin{itemize} [leftmargin=*, itemsep=2pt, topsep=2pt, parsep=0pt]
    \item SAE Match \cite{mech_permute} (\textbf{Match}): computes the distance between two features via $L_2$ of the decoder vectors. This method contains two setups: directly matching two layers, or ``composition permutation" where we recursively match two consecutive layers till we reach the target layer, we denote this setup as \textbf{Match ML}.
    \item Feature Flow \cite{featflow} (\textbf{FeatFlow}): uses cosine similarity of the decoder vectors as the feature distance. This method contains also contains two setups of direct matching and ``composition permutation", we denote this setup as \textbf{FeatFlow ML}.
    \item Attribution Patching \cite{attrib, feature_circuit} (\textbf{Attrib}): selects important feature for a target feature via using attribution patching.
    \item LLM selection (\textbf{LLM Sel}): we prompt LLM (gpt-oss-120b \cite{oss}) to select which feature is the most relevant to a target feature. Note that prompting LLM to select among all pairs of features are very costly, thus we do not evaluate this method on feature matching (Section \ref{exps: feature matching}).
\end{itemize}

Lastly, Appendix \ref{sec:feat_as_dist} illustrates that representing feature as distribution is outperforms representing as one vector for majority of main experiments, showing that activation distribution preserves more semantic information.

\subsection{Feature Matching Experiments}
\label{exps: feature matching}
We test the feature matching under three setups, matching SAEs from: consecutive layers, middle to last layer, and first to last layer detailed in Table \ref{tab:feature_matching_setup}. Each setup we randomly sample 100 feature pairs to evaluate, these pairs are shared across all methods to ensure fairness. Since LLM rating fluctuates, we evaluate 5 times and report the 90\% confident interval. Other details about LLM prompt and implementation are in Appendix \ref{sec:setup_feature_matching}. Given the target SAE at layer $A$ and source SAE at layer $B$, we follow \cite{mech_permute} and evaluate using:
\begin{itemize} [leftmargin=*, itemsep=2pt, topsep=2pt, parsep=0pt]
    \item LLM evaluation and Accuracy (\textbf{LLM Eval} and \textbf{Acc}): We use a language model to caption two features separately given their activation tokens. Given the captions, for the LLM Eval metric, we ask LLM if the two features are ``similar", ``maybe", or ``different"; and we assign 3, 2 or 1 point for the judgment respectively, the higher the point the better. For the Acc metric, we report the percentage of accurate matching (LLM judge the feature pair as ``similar"). We choose gpt-oss-120b \cite{oss} for this experiment.


    \item Cross Entropy loss and Variance Explained (\textbf{CE} and \textbf{VE}): If we correctly match similar feature, we can replace the target feature in layer $A$ with source feature in $B$. We denote the reconstruction of SAE $A$ given the feature replacement as $\widehat x^{(A)}_{\widehat M_{A\to B}}$. The CE measures Cross Entropy loss of the language model when replacing the original hidden state $x^{(A)}$ with $\widehat x^{(A)}_{\widehat M_{A\to B}}$, while the VE measures how well $x^{(A)}_{\widehat M_{A\to B}}$ captures $x^{(A)}$ via variance explained $1-Var(x^{(A)}-\widehat x^{(A)}_{\widehat M_{A\to B}}) / Var(x^{(A)})$.

\end{itemize}

The results are shown in Table \ref{tab:feature_matching}. Our method performs well on LLM Eval and Matching Accuracy, while maintaining competitive performance on downstream cross entropy loss and variance explained. Most noticeably, our method has substantially more similar matched features across all setups compared to other baselines, suggesting the effective of using reference space. Moreover, consistently reaching strong results on CE and VE shows that our method identifies the correct feature pairs with similar functional role, not just mere similar in activated tokens. On the other hand, other methods are less consistent in CE and VE metrics as well as are less accurate in multi-layer matching.

\begin{table}[h]
    \footnotesize
    \centering
    \vspace{-5mm}
    \caption{Feature matching experiment results on GPT2 and Gemma-2-2B.}
    \begin{tabular}{l|c|c|c|c|c|c|c|c}
        \toprule
        & \multicolumn{4}{c|}{\textbf{GPT2}} & \multicolumn{4}{c}{\textbf{Gemma-2-2B}} \\
        \textbf{Method} & \textbf{LLM Eval} $\uparrow$ & \textbf{Acc} $\uparrow$ & \textbf{CE} $\downarrow$ & \textbf{VE} $\uparrow$ & \textbf{LLM Eval} $\uparrow$ & \textbf{Acc} $\uparrow$ & \textbf{CE} $\downarrow$ & \textbf{VE} $\uparrow$ \\
        \midrule
        \multicolumn{1}{c|}{} & \multicolumn{4}{c|}{\textbf{Layer 5 match with Layer 6}} & \multicolumn{4}{c}{\textbf{Layer 12 match with Layer 13}} \\
        \midrule
        \textbf{Our} & \textbf{2.53} $\pm$ \textbf{0.04} & \textbf{68.2} $\pm$ \textbf{1.9} & \underline{3.775} & \underline{0.9866} & \textbf{2.32} $\pm$ \textbf{0.04} & \textbf{58.6} $\pm$ \textbf{2.0} & 2.74784 & \textbf{0.7196} \\
        FeatFlow & \underline{2.46} $\pm$ \underline{0.04} & \underline{66.2} $\pm$ \underline{1.8} & \textbf{3.774} & \underline{0.9866} & 1.51 $\pm$ 0.03 & 20.4 $\pm$ 1.5 & \textbf{2.74774} & 0.7181 \\
        Match & 2.25 $\pm$ 0.04 & 55.6 $\pm$ 2.6 & 3.776 & \textbf{0.9867} & 1.41 $\pm$ 0.01 & 17.2 $\pm$ 1.5 & \underline{2.74781} & \underline{0.7188} \\
        Attrib & 1.21 $\pm$ 0.01 & 8.6 $\pm$ 0.7 & 3.785 & 0.9863 & \underline{1.65} $\pm$ \underline{0.04} & \underline{27.2} $\pm$ \underline{1.8} & 2.74804 & 0.7105 \\
        \midrule
        \multicolumn{1}{c|}{} & \multicolumn{4}{c|}{\textbf{Layer 5 match with Layer 11}} & \multicolumn{4}{c}{\textbf{Layer 12 match with Layer 25}} \\
        \midrule
        \textbf{Our} & \textbf{1.56} $\pm$ \textbf{0.04} & \textbf{23.2} $\pm$ \textbf{2.2} & \underline{3.806} & \underline{0.9127} & \textbf{1.83} $\pm$ \textbf{0.04} & \textbf{33.6} $\pm$ \textbf{2.9} & \textbf{2.74809} & \textbf{0.9115} \\
        FeatFlow & \underline{1.49} $\pm$ \underline{0.02} & 20.0 $\pm$ 1.7 & 3.810 & 0.9126 & 1.26 $\pm$ 0.01 & 10.0 $\pm$ 1.0 & 2.74813 & \underline{0.9114} \\
        Match & \underline{1.49} $\pm$ \underline{0.03} & \underline{22.2} $\pm$ \underline{1.9} & \textbf{3.800} & \textbf{0.9130} & 1.18 $\pm$ 0.03 & 6.6 $\pm$ 1.1 & \underline{2.74811} & \underline{0.9114} \\
        FeatFlow ML & 1.18 $\pm$ 0.01 & 7.8 $\pm$ 1.1 & 3.808 & \underline{0.9127} & 1.23 $\pm$ 0.04 & 9.4 $\pm$ 2.3 & \underline{2.74811} & \underline{0.9114} \\
        Match ML & 1.17 $\pm$ 0.04 & 6.8 $\pm$ 1.9 & 3.808 & \underline{0.9127} & 1.26 $\pm$ 0.04 & 11.2 $\pm$ 1.2 & 2.74812 & \underline{0.9114} \\
        Attrib & 1.26 $\pm$ 0.02 & 10.6 $\pm$ 1.2 & 3.955 & 0.8911 & \underline{1.34} $\pm$ \underline{0.04} & \underline{13.0} $\pm$ \underline{1.3} & 2.74820 & 0.9109 \\
        \midrule
        \multicolumn{1}{c|}{} & \multicolumn{4}{c|}{\textbf{Layer 0 match with Layer 11}} & \multicolumn{4}{c}{\textbf{Layer 0 match with Layer 25}} \\
        \midrule
        \textbf{Our} & \textbf{1.39} $\pm$ \textbf{0.03} & \textbf{17.2} $\pm$ \textbf{2.3} & \textbf{3.828} & \textbf{0.9111} & \textbf{1.83} $\pm$ \textbf{0.05} & \textbf{34.8} $\pm$ \textbf{2.1} & \underline{2.74866} & \textbf{0.9094} \\
        FeatFlow & \underline{1.34} $\pm$ \underline{0.02} & \underline{12.4} $\pm$ \underline{1.0} & 3.855 & 0.9095 & 1.23 $\pm$ 0.04 & 8.4 $\pm$ 1.5 & 2.74884 & 0.9091 \\
        Match & 1.27 $\pm$ 0.02 & 10.0 $\pm$ 1.3 & 3.833 & \underline{0.9110} & 1.21 $\pm$ 0.01 & 8.0 $\pm$ 0.8 & 2.74874 & \textbf{0.9094} \\
        FeatFlow ML & 1.16 $\pm$ 0.04 & 5.4 $\pm$ 2.4 & \underline{3.831} & \underline{0.9110} & 1.11 $\pm$ 0.03 & 4.4 $\pm$ 1.2 & 2.74889 & 0.9085 \\
        Match ML & 1.14 $\pm$ 0.03 & 6.0 $\pm$ 2.1 & 3.911 & 0.9063 & 1.13 $\pm$ 0.04 & 5.2 $\pm$ 1.8 & \textbf{2.74863} & \textbf{0.9094} \\
        Attrib & 1.18 $\pm$ 0.01 & 6.2 $\pm$ 1.3 & 3.837 & 0.9105 & \underline{1.28} $\pm$ \underline{0.05} & \underline{11.6} $\pm$ \underline{2.6} & 2.74885 & \underline{0.9092} \\
        \bottomrule
    \end{tabular}
    \vspace{-3mm}
    \label{tab:feature_matching}
\end{table}

\subsection{Circuit Compressing Experiments}
\label{sec:circuit_compression}
To evaluate, we first construct feature circuits following \cite{feature_circuit} with two setups: 50 nodes and compress down to 5 supernodes, and 200 nodes compress to 20 supernodes. Then, for each supernode, we randomly sample half of the nodes to prompt an LLM to caption the shared meaning and the other half are mixed with random features from the remaining supernodes to create a test set. We then ask the LLM to pick the node in the test set that is from the original supernode given the caption. The more accurate prediction, the more intuitive the compression. We evaluate over 200 circuits and report 90\% confident interval. Further implementation details are in Appendix \ref{sec:circuit_compression_setup}. Note that, in the setup of 200 nodes, we cannot evaluate LLM Sel because the circuits cannot fit into the context window of the clustering LLM (gpt-oss-120b \cite{oss}).

The results are in table \ref{tab:circuit_compression}. We find that the our method compresses circuit more accurately for both GPT2 and Gemma-2-2B than other methods. On 50 nodes, our method perform the best on Gemma-2-2B while being equally well with FeatFlow on GPT2-small. On 200 nodes, we outperform other methods on both models. While compressing using LLM selection is competitive, it is more costly in large circuits and still less accurate than our method. This shows that our method are the best choice for automated circuit compressing.

\begin{table}[h]
    \centering
    \small
    \caption{Circuit compressing experiment results on GPT2 and Gemma-2-2B. LLM Sel is not run on 200-node circuits because the circuits exceed the context length of the LLM.} 
    \vspace{2mm}
    \begin{tabular}{l|c|c|c|c}
        \toprule
        & \multicolumn{2}{c|}{50 nodes, 5 supernodes} & \multicolumn{2}{c}{200 nodes, 20 supernodes} \\ 
        \textbf{Method} & \textbf{Gemma-2-2B} & \textbf{GPT2-small} & \textbf{Gemma-2-2B} & \textbf{GPT2-small} \\
        \midrule
        \textbf{Our} & \textbf{0.6047} $\pm$ \textbf{0.0144} & \underline{0.6883} $\pm$ \underline{0.0209} & \textbf{0.6404} $\pm$ \textbf{0.0069} & \textbf{0.7368} $\pm$ \textbf{0.0100} \\ 
        FeatFlow & 0.5190 $\pm$ 0.0138 & \textbf{0.6895} $\pm$ \textbf{0.0220} & 0.5524 $\pm$ 0.0061 & 0.7183 $\pm$ 0.0103 \\ 
        Match & 0.5488 $\pm$ 0.0218 & 0.5567 $\pm$ 0.0255 & \underline{0.6203} $\pm$ \underline{0.0139} & \underline{0.7195} $\pm$ \underline{0.0115} \\ 
        LLM Sel & \underline{0.5820} $\pm$ \underline{0.0142} & 0.6536 $\pm$ 0.0222 & -- & -- \\
        Attrib & 0.5028 $\pm$ 0.0128 & 0.6393 $\pm$ 0.0210 & 0.5265 $\pm$ 0.0062 & 0.6544 $\pm$ 0.0099 \\
        \bottomrule
    \end{tabular}
    \label{tab:circuit_compression}
\end{table}

\begin{figure}[h]
    \centering
    \includegraphics[width=0.45\linewidth]{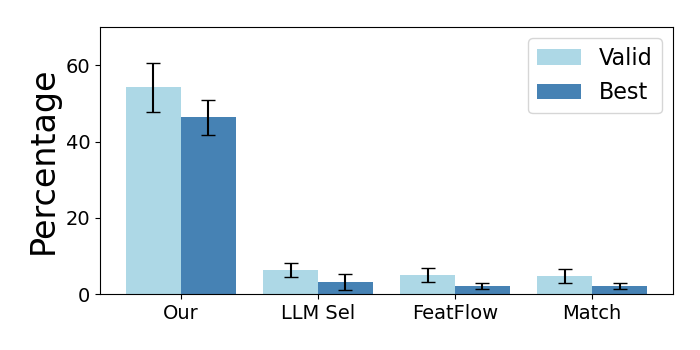}
    \caption{Digit addition experiment results. ``Valid" means the user agree with the matching, and ``Best" means the method has the best match among all methods.}
    \label{fig:digit_addition}
\end{figure}

\subsection{Matching Feature With Subtle Meaning}
\label{sec:digit_addition}
We test the methods ability to recognize subtle meanings in the features, where we know the ``ground truth" meanings. Specifically, we use ``two number addition" feature \cite{on_biology} in the format ``$p_1+p_2=$". We limit $p_1, p_2 \in [0,9]$ for a less costly experiment to run LLM Selection method. We extract the features' activation on the ``=" token over 100 datapoints ($[0,9]\times[0,9]$) and plot the activation pattern of each feature. The patterns tell the ground truth of feature meaning, i.e. feature: ``first number is 5" would have activation pattern of ``only row $p_1=5$ has positive values", one such pattern is in Figure \ref{fig:digit_addition_qualitative}. Based on the patterns as the ground truth, we can identify if the matching correctly identify the functional meaning or merely ``matching in math" context only.

We create blind ranking API for the methods and ask volunteers to rank which candidate feature is better given a target feature, randomly ranking 100 target features. We test on matching feature of the same layer (Layer 19 MLP from Gemma-2-2b) to ensure other methods can capture on the same manifold. Details are in Appendix \ref{sec:digit_addition_setup}.

The quantitative results in Figure \ref{fig:digit_addition} and qualitative in Figure \ref{fig:digit_addition_qualitative}. Our methods are ranked substantially higher than other matching methods by the people in the survey. The qualitative figure \ref{fig:digit_addition_qualitative} shows that our method correctly capture the hidden meaning of the features while others cannot.

\begin{figure}
    \centering
    \includegraphics[width=0.9\linewidth]{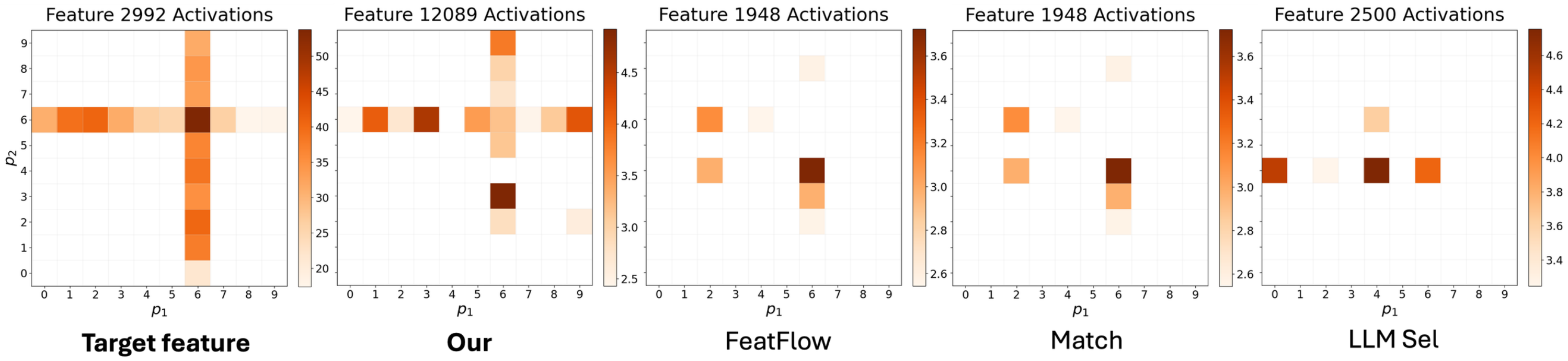}
    \caption{Qualitative example of ``digit addition" (Section \ref{sec:digit_addition}). Our method correctly identifies the similar feature while other methods fail.}
    \label{fig:digit_addition_qualitative}
\end{figure}

\section{Applications}
We demonstrate two applications of our framework. We show qualitative results of automatic circuit compression in Section \ref{application:circuit_compression}. Furthermore, we can identify feature groups that consistently perform similar function role in many circuits, Section \ref{application:modular}.

\begin{figure}
    \centering
    \includegraphics[width=1\linewidth]{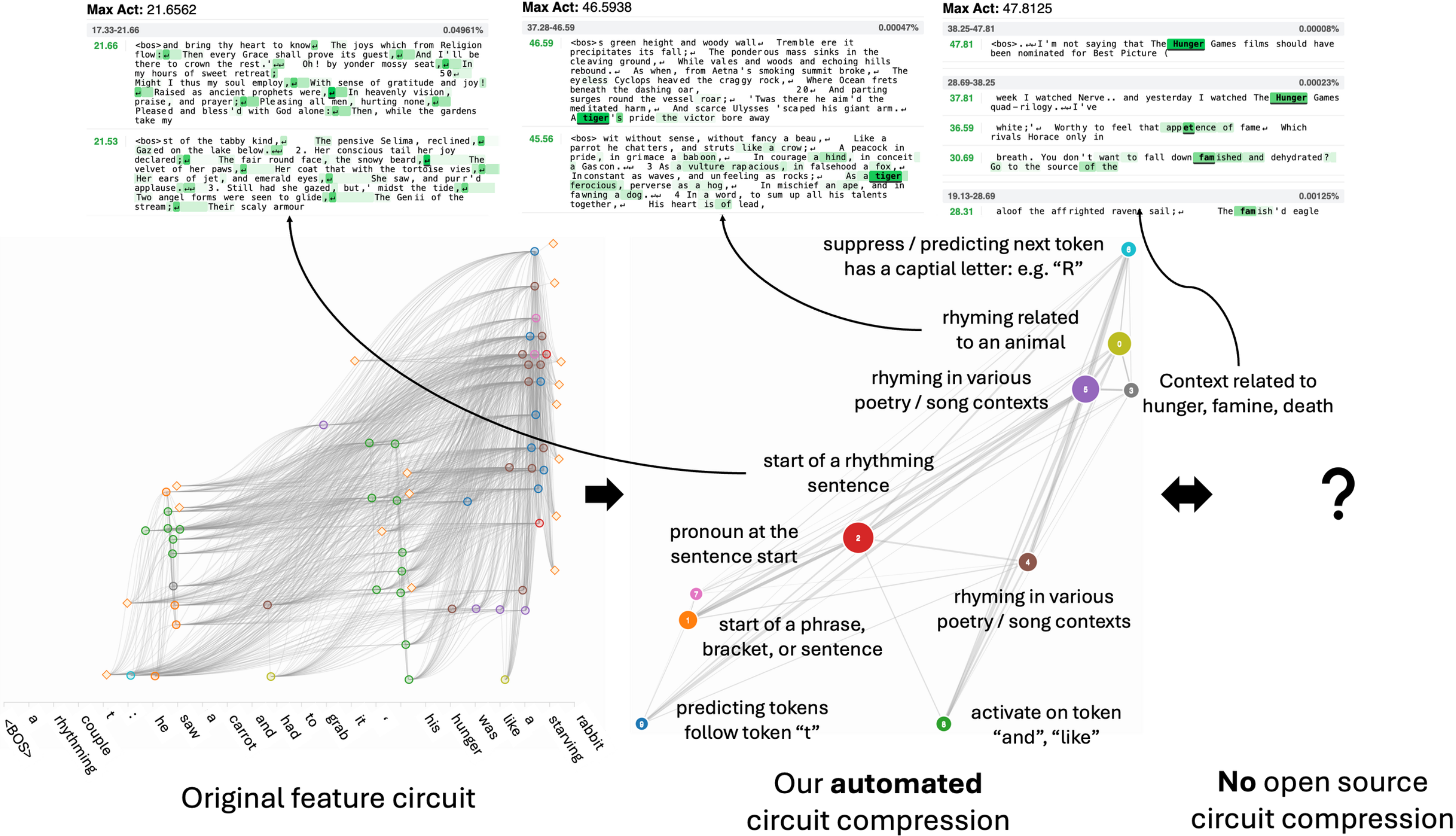}
    \caption{This figures compares our automated circuit compression and open source manual compression \cite{neuronpedia} on ``Rhythm" circuit. We can provide circuit compression automatically while there are no manual open source compression for this circuit.}
    \label{fig:rhythm_circuit}
\end{figure}

\subsection{Case Study: Automated Circuit Compression}
\label{application:circuit_compression}
We demonstrate qualitative results in \textit{automatically} compressing feature circuits to reveal the mechanisms of the model. We follow the circuit examples from Neuropedia \cite{neuronpedia} and \citet{on_biology} in Figure \ref{fig:dallas_circuit}, \ref{fig:rhythm_circuit}, \ref{fig:big_small_circuit}. Specifically, we found that our compression of ``Dallas circuit" (Figure \ref{fig:dallas_circuit}) found similar supernodes as in open source \textit{manual} implementation \cite{neuronpedia}. We find that the model retrieve information about Dallas as the Texas state, it then retrieve various government / U.S court related contexts and output the correct answer ``Austin". We identify compressed circuit that \textit{do not have open source compression} in ``rhyming circuit" (Figure \ref{fig:rhythm_circuit}), we find that the model first recognizes the poem context and thinks about a noun that rhythm and related to ``hunger". It then thinks about rhythm in ``animal" context, ultimately outputting ``rabbit". Details are in Appendix \ref{sec:automated_circuit_compress_setup}.

\subsection{Case Study: Finding Modular Feature Group}
\label{application:modular}

\begin{figure}
    \centering
    \vspace{-5mm}
    \includegraphics[width=1\linewidth]{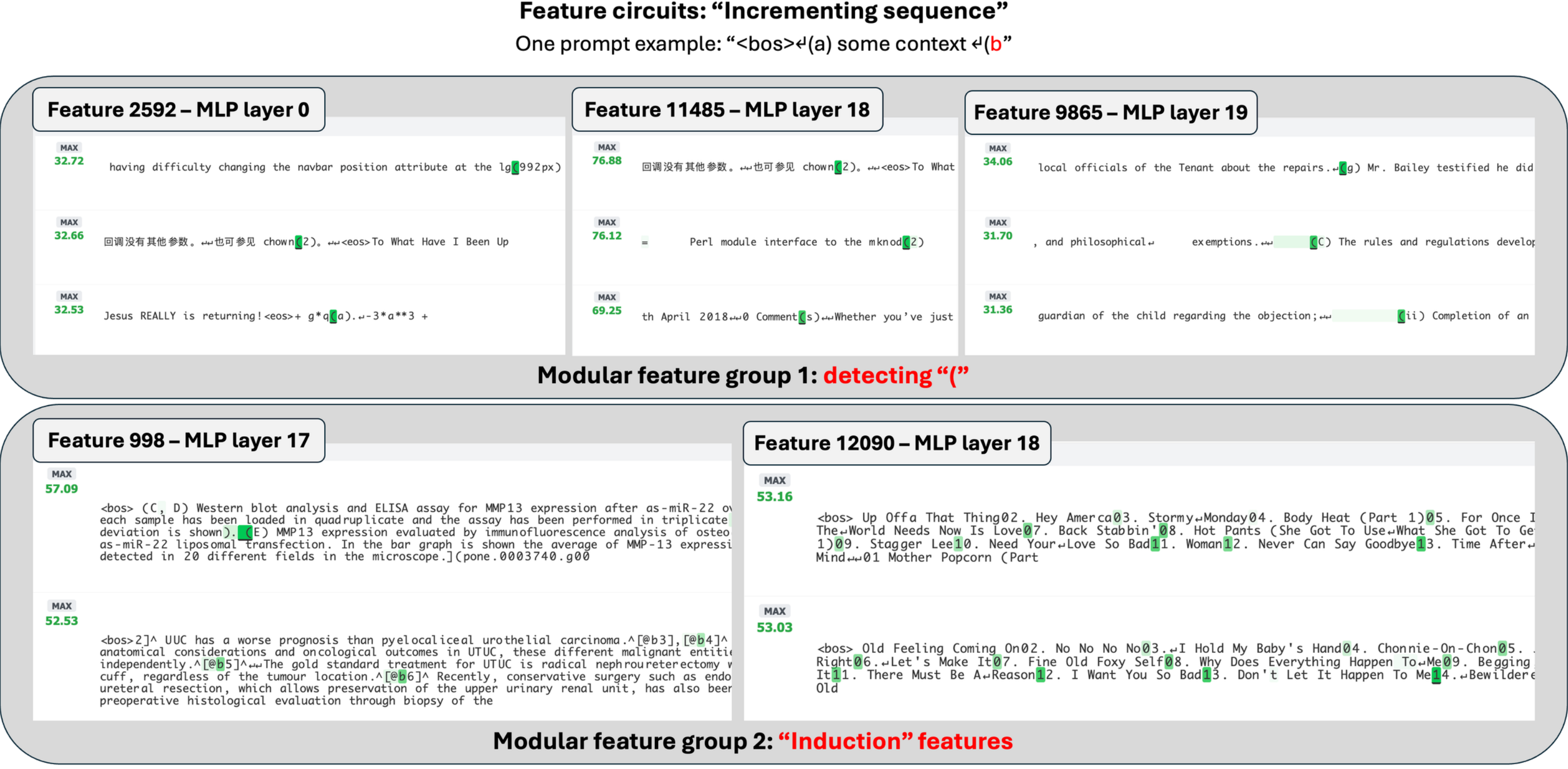}
    \caption{Two modular circuits identified form ``incrementing sequences" circuits \cite{feature_circuit}. We find two distinct modular function among the features, namely detecting ``(" features and ``Induction" features \cite{induction}. For example, feature 12090 layer 18 sees ``02" of previous text (without activate on ``02"), it recognizes this is increment sequence and activates on ``03", etc.}
    \label{fig:incrementing_seq_modular}.
    \vspace{-4mm}
\end{figure}

Beyond unsupervisedly finding circuits that perform similar task \cite{feature_circuit}, we can further automatically detect feature group that the features consistently perform the same functional role and appear in similar circuits.
Specifically, we hypothesize that features that are frequently assigned into the same supernodes have the same functional role. To check this hypothesis, we propose the following procedure: we use method in \cite{feature_circuit} to identify circuits that perform similar task. We then perform circuit compression on all circuits. Lastly, we apply clustering based on how frequent the features are in the same supernodes throughout all compressions, creating groups of features that we term as ``modular feature groups". Details on the procedure are in Appendix \ref{sec:modular_setup}.

Figure \ref{fig:incrementing_seq_modular}, \ref{fig:html_tag_modular}, \ref{fig:apostrophe_modular} show the modular feature groups. For the interesting behavior of ``incrementing sequences" in \cite{feature_circuit}, we can automatically identify two interesting groups (Figure \ref{fig:incrementing_seq_modular}). Group 1 help recognize token ``("; and group 2 contain ``induction features" - the core mechanism in recognizing repeated context \cite{induction}. This shows that our circuit compression correctly identifies the subtle functional role of the features, demonstrating the capability of unsupervisedly exploring complex behavior. Similar examples can be seen in Figure \ref{fig:html_tag_modular} and \ref{fig:apostrophe_modular}.

\section{Conclusion}
\label{sec: discussion}
In this work, we introduced a distributional framework for comparing Sparse Autoencoder features across multi-layer language models. We framed cross-manifold feature comparison as an Optimal Transport problem. Our approach achieves strong multi-layer feature matching and enables the automated feature circuits compression into interpretable semantic supernodes.

\begin{ack}
This work is partially supported by National Science Foundation
(NSF) grants SCH-2123809. Any opinions, findings, and conclusions or recommendations expressed in this paper, however, are those of the authors and do not necessarily reflect the views of the funding agency.
\end{ack}






{
\small

\bibliography{cite}
\bibliographystyle{unsrtnat}
}

\newpage
\section*{NeurIPS Paper Checklist}

The checklist is designed to encourage best practices for responsible machine learning research, addressing issues of reproducibility, transparency, research ethics, and societal impact. Do not remove the checklist: {\bf The papers not including the checklist will be desk rejected.} The checklist should follow the references and follow the (optional) supplemental material.  The checklist does NOT count towards the page
limit. 

Please read the checklist guidelines carefully for information on how to answer these questions. For each question in the checklist:
\begin{itemize}
    \item You should answer \answerYes{}, \answerNo{}, or \answerNA{}.
    \item \answerNA{} means either that the question is Not Applicable for that particular paper or the relevant information is Not Available.
    \item Please provide a short (1--2 sentence) justification right after your answer (even for \answerNA). 
\end{itemize}

{\bf The checklist answers are an integral part of your paper submission.} They are visible to the reviewers, area chairs, senior area chairs, and ethics reviewers. You will also be asked to include it (after eventual revisions) with the final version of your paper, and its final version will be published with the paper.

The reviewers of your paper will be asked to use the checklist as one of the factors in their evaluation. While \answerYes{} is generally preferable to \answerNo{}, it is perfectly acceptable to answer \answerNo{} provided a proper justification is given (e.g., error bars are not reported because it would be too computationally expensive'' or ``we were unable to find the license for the dataset we used''). In general, answering \answerNo{} or \answerNA{} is not grounds for rejection. While the questions are phrased in a binary way, we acknowledge that the true answer is often more nuanced, so please just use your best judgment and write a justification to elaborate. All supporting evidence can appear either in the main paper or the supplemental material, provided in appendix. If you answer \answerYes{} to a question, in the justification please point to the section(s) where related material for the question can be found.

IMPORTANT, please:
\begin{itemize}
    \item {\bf Delete this instruction block, but keep the section heading ``NeurIPS Paper Checklist"},
    \item  {\bf Keep the checklist subsection headings, questions/answers and guidelines below.}
    \item {\bf Do not modify the questions and only use the provided macros for your answers}.
\end{itemize}


\begin{enumerate}

\item {\bf Claims}
    \item[] Question: Do the main claims made in the abstract and introduction accurately reflect the paper's contributions and scope?
    \item[] Answer: \answerYes{} 
    \item[] Justification: See Abstract and Introduction (Section \ref{sec:introduction}). 

\item {\bf Limitations}
    \item[] Question: Does the paper discuss the limitations of the work performed by the authors?
    \item[] Answer: \answerYes{} 
    \item[] Justification: See Appendix \ref{sec:limitation}.

\item {\bf Theory assumptions and proofs}
    \item[] Question: For each theoretical result, does the paper provide the full set of assumptions and a complete (and correct) proof?
    \item[] Answer: \answerYes{} 
    \item[] Justification: See Appendix \ref{sec:additional_theory}.

    \item {\bf Experimental result reproducibility}
    \item[] Question: Does the paper fully disclose all the information needed to reproduce the main experimental results of the paper to the extent that it affects the main claims and/or conclusions of the paper (regardless of whether the code and data are provided or not)?
    \item[] Answer: \answerYes{} 
    \item[] Justification: See Appendix \ref{sec:exp_details}.

\item {\bf Open access to data and code}
    \item[] Question: Does the paper provide open access to the data and code, with sufficient instructions to faithfully reproduce the main experimental results, as described in supplemental material?
    \item[] Answer: \answerNo{} 
    \item[] Justification: We will provide the code once the paper get published.

\item {\bf Experimental setting/details}
    \item[] Question: Does the paper specify all the training and test details (e.g., data splits, hyperparameters, how they were chosen, type of optimizer) necessary to understand the results?
    \item[] Answer: \answerYes{} 
    \item[] Justification: See Appendix \ref{sec:exp_details}.

\item {\bf Experiment statistical significance}
    \item[] Question: Does the paper report error bars suitably and correctly defined or other appropriate information about the statistical significance of the experiments?
    \item[] Answer: \answerYes{} 
    \item[] Justification: We report 90\% confident interval, see Section \ref{exps}.

\item {\bf Experiments compute resources}
    \item[] Question: For each experiment, does the paper provide sufficient information on the computer resources (type of compute workers, memory, time of execution) needed to reproduce the experiments?
    \item[] Answer: \answerYes{} 
    \item[] Justification: See Appendix \ref{sec:sae_training}.
    
\item {\bf Code of ethics}
    \item[] Question: Does the research conducted in the paper conform, in every respect, with the NeurIPS Code of Ethics \url{https://neurips.cc/public/EthicsGuidelines}?
    \item[] Answer: \answerYes{} 
    \item[] Justification: We follow code of ethics guidelines and only use open source data and models.

\item {\bf Broader impacts}
    \item[] Question: Does the paper discuss both potential positive societal impacts and negative societal impacts of the work performed?
    \item[] Answer: \answerNA{} 
    \item[] Justification: Our method do not tie to any deployments.
    
\item {\bf Safeguards}
    \item[] Question: Does the paper describe safeguards that have been put in place for responsible release of data or models that have a high risk for misuse (e.g., pre-trained language models, image generators, or scraped datasets)?
    \item[] Answer: \answerNA{} 
    \item[] Justification: We do not release new data or model.

\item {\bf Licenses for existing assets}
    \item[] Question: Are the creators or original owners of assets (e.g., code, data, models), used in the paper, properly credited and are the license and terms of use explicitly mentioned and properly respected?
    \item[] Answer: \answerYes{} 
    \item[] Justification: We use open source dataset and models that are sufficiently cited.

\item {\bf New assets}
    \item[] Question: Are new assets introduced in the paper well documented and is the documentation provided alongside the assets?
    \item[] Answer: \answerNA{} 
    \item[] Justification: We do not release new assets.

\item {\bf Crowdsourcing and research with human subjects}
    \item[] Question: For crowdsourcing experiments and research with human subjects, does the paper include the full text of instructions given to participants and screenshots, if applicable, as well as details about compensation (if any)? 
    \item[] Answer: \answerYes{} 
    \item[] Justification: See Appendix \ref{sec:digit_addition_setup}.

\item {\bf Institutional review board (IRB) approvals or equivalent for research with human subjects}
    \item[] Question: Does the paper describe potential risks incurred by study participants, whether such risks were disclosed to the subjects, and whether Institutional Review Board (IRB) approvals (or an equivalent approval/review based on the requirements of your country or institution) were obtained?
    \item[] Answer: \answerYes{} 
    \item[] Justification: Our survey poses no risk for study participants.

\item {\bf Declaration of LLM usage}
    \item[] Question: Does the paper describe the usage of LLMs if it is an important, original, or non-standard component of the core methods in this research? Note that if the LLM is used only for writing, editing, or formatting purposes and does \emph{not} impact the core methodology, scientific rigor, or originality of the research, declaration is not required.
    \item[] Answer: \answerNA{} 
    \item[] Justification: We do not use LLM for core methodology.

\end{enumerate}

\clearpage
\noindent\rule{\textwidth}{1.5pt}  

\begin{center}
    {\LARGE\bfseries Supplementary for {Semantic Optimal Transport for Sparse Autoencoder Feature Matching and Circuit Compression}}
\end{center}

\noindent\rule{\textwidth}{1.5pt}  

\addcontentsline{toc}{section}{Appendix}

\textbf{This supplementary material serves as the appendix to the main paper to provide full detailed pseudocode of algorithms, derivations and proofs of key theoretical results, comprehensive ablation studies and hyperparameter settings, as well as expanded discussions on implementation to reproduce results, scalability, and generalization. It is organized as follows:}

\begin{itemize}[leftmargin=*, label=]

    \item \phantomsection \textbf{\ref{sec:limitation}. \nameref{sec:limitation}} 
    \dotfill \pageref{sec:limitation}

    \item \phantomsection \textbf{\ref{sec:related_works_extend}. \nameref{sec:related_works_extend}} 
    \dotfill \pageref{sec:related_works_extend}

    \item \phantomsection \textbf{\ref{sec:exp_details}. \nameref{sec:exp_details}} 
    \dotfill \pageref{sec:exp_details}
    \begin{itemize}
        \item[] \ref{sec:main_algo}. \nameref{sec:main_algo} 
        \dotfill \pageref{sec:main_algo}
        
        \item[] \ref{sec:baseline_implementation}. \nameref{sec:baseline_implementation} 
        \dotfill \pageref{sec:baseline_implementation}

        \item[] \ref{sec:setup_feature_matching}. \nameref{sec:setup_feature_matching} 
        \dotfill \pageref{sec:setup_feature_matching}

        \item[] \ref{sec:circuit_compression_setup}. \nameref{sec:circuit_compression_setup} 
        \dotfill \pageref{sec:circuit_compression_setup}

        \item[] \ref{sec:digit_addition_setup}. \nameref{sec:digit_addition_setup} 
        \dotfill \pageref{sec:digit_addition_setup}

        \item[] \ref{sec:automated_circuit_compress_setup}. \nameref{sec:automated_circuit_compress_setup} 
        \dotfill \pageref{sec:automated_circuit_compress_setup}

        \item[] \ref{sec:modular_setup}. \nameref{sec:modular_setup} 
        \dotfill \pageref{sec:modular_setup}

        \item[] \ref{sec:sae_training}. \nameref{sec:sae_training} 
        \dotfill \pageref{sec:sae_training}

        \item[] \ref{sec:prompt}. \nameref{sec:prompt} 
        \dotfill \pageref{sec:prompt}
    \end{itemize}

    \item \phantomsection \textbf{\ref{sec:feat_as_dist}. \nameref{sec:feat_as_dist}} 
    \dotfill \pageref{sec:feat_as_dist}

    \item \phantomsection \textbf{\ref{sec:additional_theory}. \nameref{sec:additional_theory}} 
    \dotfill \pageref{sec:additional_theory}
    \begin{itemize}
        \item[] \ref{app:activation_rescaling_invariance}. \nameref{app:activation_rescaling_invariance} 
        \dotfill \pageref{app:activation_rescaling_invariance}

        \item[] \ref{app:flow_cost_refinement}. \nameref{app:flow_cost_refinement} 
        \dotfill \pageref{app:flow_cost_refinement}

        \item[] \ref{app:wasserstein_stability}. \nameref{app:wasserstein_stability} 
        \dotfill \pageref{app:wasserstein_stability}

        \item[] \ref{app:metric_specific_stability}. \nameref{app:metric_specific_stability} 
        \dotfill \pageref{app:metric_specific_stability}

        \item[] \ref{app:ground_cost_perturbation}. \nameref{app:ground_cost_perturbation} 
        \dotfill \pageref{app:ground_cost_perturbation}

        \item[] \ref{app:matching_recovery}. \nameref{app:matching_recovery} 
        \dotfill \pageref{app:matching_recovery}

        \item[] \ref{app:voronoi_recovery}. \nameref{app:voronoi_recovery} 
        \dotfill \pageref{app:voronoi_recovery}

        \item[] \ref{sec:voronoi_synthetic_verification}. \nameref{sec:voronoi_synthetic_verification} 
        \dotfill \pageref{sec:voronoi_synthetic_verification}
    \end{itemize}

\end{itemize}

\vspace{1em}
\hrule
\vspace{1em}

\pagebreak

\section{Appendix}

\appendix

\begin{figure}
    \centering
    \includegraphics[width=1\linewidth]{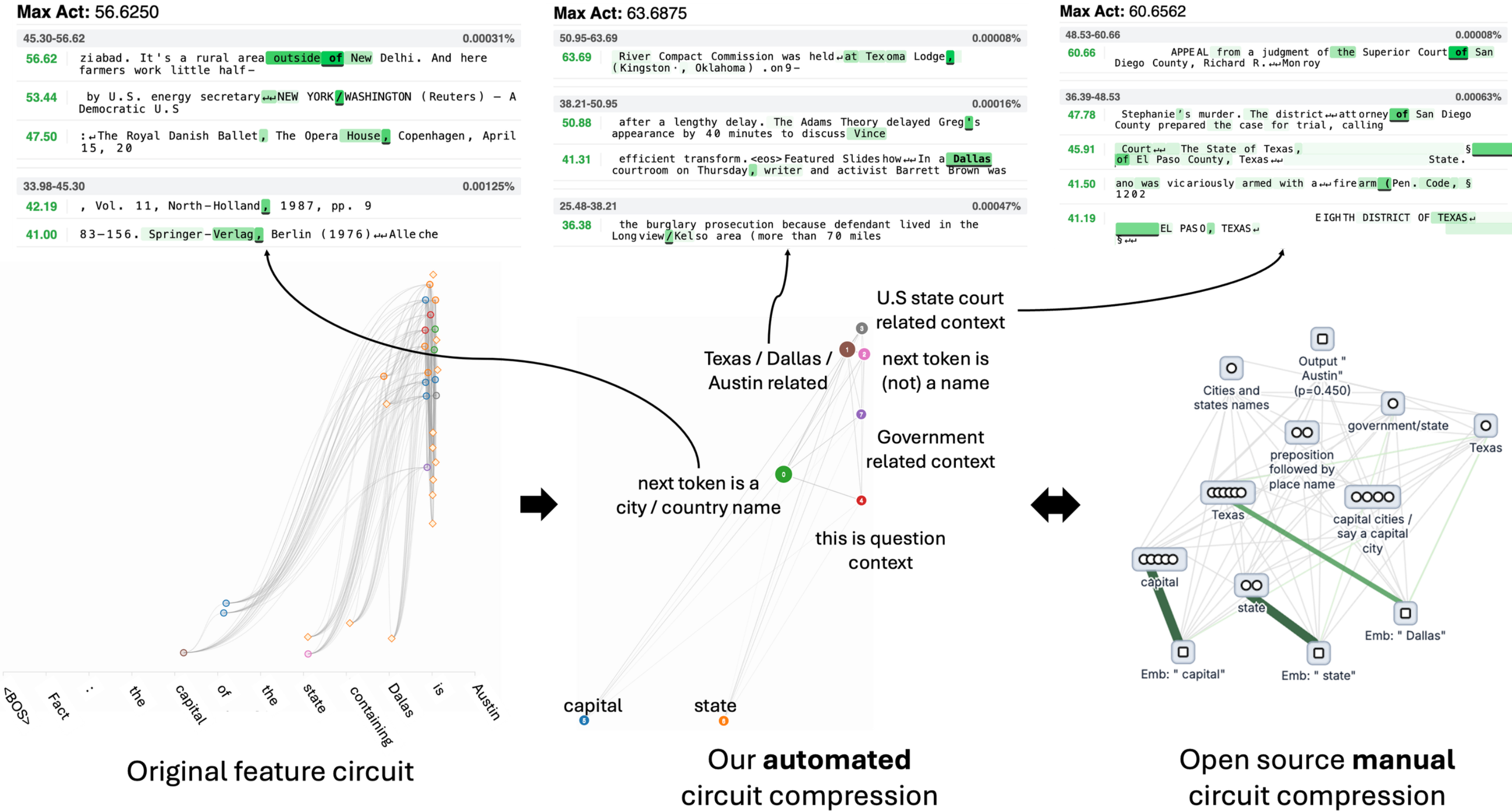}
    \caption{This figures compares our automated circuit compression and open source manual compression \cite{neuronpedia} on ``Dallas" circuit. Our method correctly finds supernodes as the open source.}
    \label{fig:dallas_circuit}
\end{figure}

\begin{figure}
    \centering
    \includegraphics[width=1\linewidth]{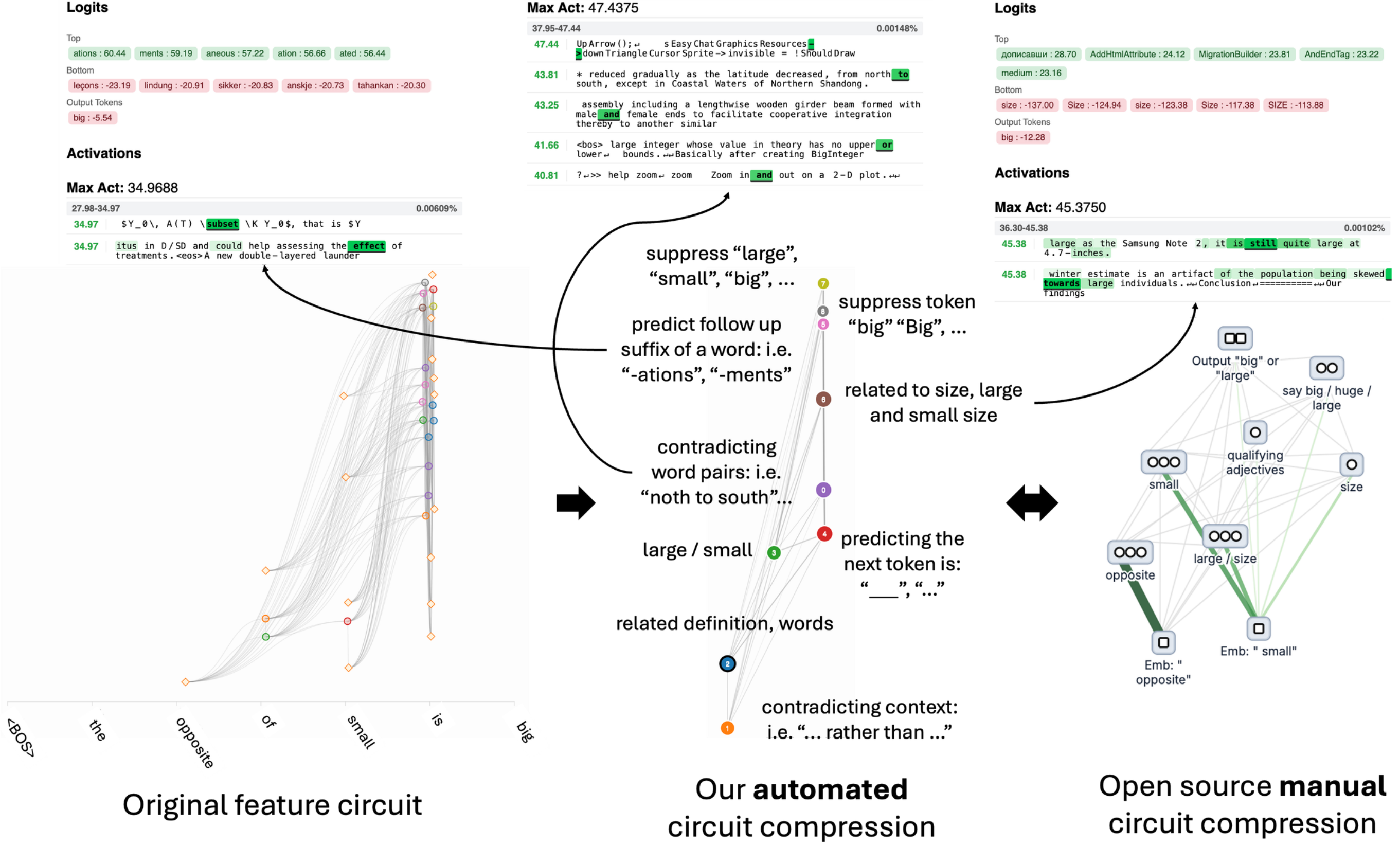}
    \caption{This figures compares our automated circuit compression and open source manual compression \cite{neuronpedia} on ``Opposite" circuit. Our method finds supernodes that are not in open source implementation such as ``predict next token is '$\_\_\_$' or '\dots'.}
    \label{fig:big_small_circuit}
\end{figure}

\begin{figure}
    \centering
    \includegraphics[width=1\linewidth]{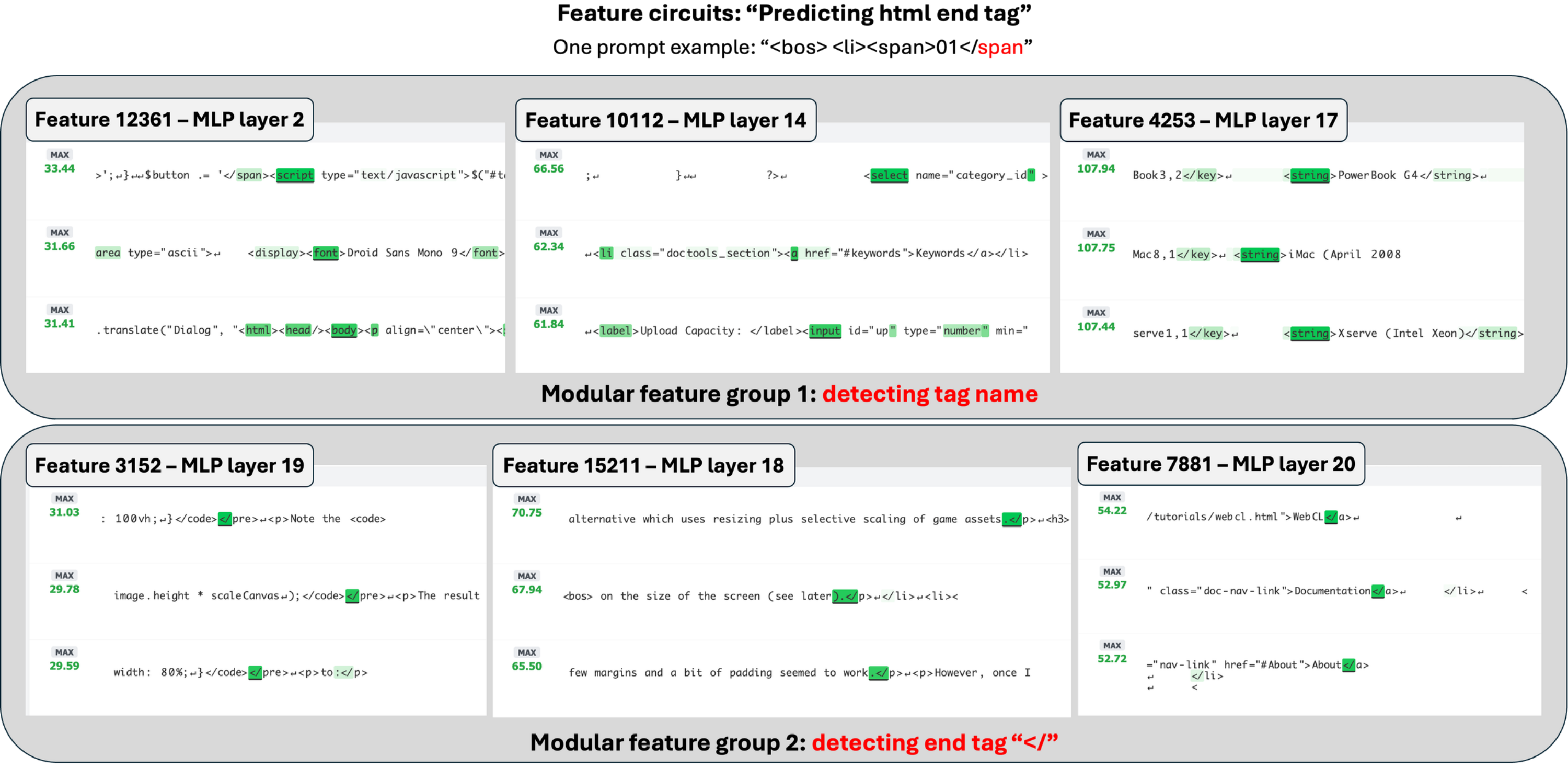}
    \caption{Two modular circuits identified form ``predicting html end tags" circuits. We find two distinct modular function among the features, namely detecting tag names and detecting end tags.}
    \label{fig:html_tag_modular}
\end{figure}

\begin{figure}
    \centering
    \includegraphics[width=1\linewidth]{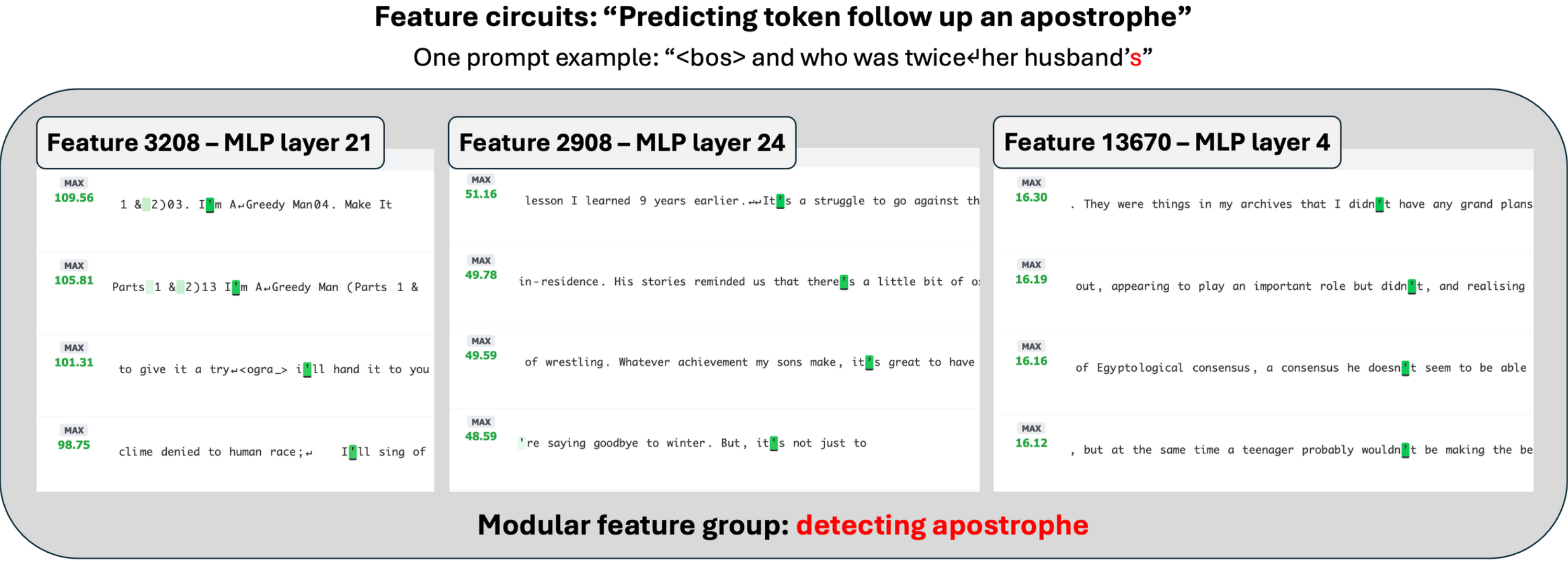}
    \caption{Two modular circuits identified form ``predict token follow up an apostrophe" circuits. We find a group of features consistently detecting apostrophes.}
    \label{fig:apostrophe_modular}
\end{figure}

\section{Limitations}
\label{sec:limitation}
A primary limitation of our current framework lies in the choice of the ground cost $c(\cdot, \cdot)$ used in our Optimal Transport formulation. Currently, we compute the Wasserstein distance using the standard $L_2$ distance in the shared reference space. This implicitly assumes a flat, Euclidean geometry, which may not fully capture the complex, non-linear manifold structure of language model representations. A promising direction for future work is to replace the $L_2$ metric with a more general, geometry-aware ground cost. For instance, leveraging the velocity field derived from flow matching models could define a principled, geodesic-based transport cost, allowing the feature distributions to be aligned along the true continuous dynamics of the latent space.

\section{Related Works}
\label{sec:related_works_extend}
\textbf{Sparse Autoencoder:} SAE is proposed to extract interpretable features from a latent space of one layer of a deep model \cite{jumprelu, topk, batchtopk, cunningham2023sparse}. It shows to be effective in many applications, including model steering \cite{sae_steering}, model diffing \cite{diffing}, study hallucination \cite{do_i_know_this_entity}, study reasoning \cite{interpreting_reasoning_feature} etc. 

\textbf{Feature Matching:} To understands how features evolve across layers and decision making of a model, recent works come up with feature matching \cite{mech_permute, featflow}. The feature matching works aim to find a mapping between SAE features at one layer and features of a SAE at another layer so that the mapped features have similar meaning with the target feature. \citet{mech_permute} formulate the mapping as a 1-to-1 linear assignment problem that minimizes the mean squared error between target feature's and mapped feature's decoder vectors. \cite{featflow} further propose mapping with a feature that induces the highest cosine similarity between decoder vectors. While these methods can be effective in matching feature from consecutive layers where the decoder vectors are trained on more similar latent space distributions, in matching across multi-layer, their methods are less accurate \cite{mech_permute} due to the distribution difference. 

\noindent\textbf{Distributional and Heterogeneous Modeling:}
Prior work has explored modeling complex data distributions
through mixtures of latent variables and generators.
Gaussian-mixture latent models enable generative clustering
\cite{ijcai2017p273}, while multiple generators can improve
the coverage of diverse data modes \cite{hoang2018mgan, bui2026harmoniainterpretablegraphlearning}.
Energy-based models further provide flexible approaches
to distribution modeling \cite{du2019implicit}, including
energy-guided conditional generative mixtures for complex
optimization solution spaces \cite{NEURIPS2025_6836e35b}. Related studies have investigated heterogeneous data
representations through autoencoder-based density estimation
\cite{zong2018iclr-deep} and stochastic latent modeling
of temporal patterns \cite{su2019omnianomaly}.
Sparse mixture-of-experts architectures enable specialized
computation across different inputs
\cite{fedus2022switch,zhou2022expertchoice},
and have also been combined with state-space models to
capture diverse feature distributions for anomaly detection
\cite{ICLR2025_8da04a60}. Unlike these approaches, which learn generative distributions,
density estimators, or specialized predictors, our framework
constructs an activation-weighted empirical distribution
for each SAE feature and measures semantic similarity
between these distributions using optimal transport
in a shared reference space.

\textbf{Feature Circuit and Circuit Compression:} Feature Circuit \cite{feature_circuit, transcoder_circuit, on_biology} is a causal graph with nodes as features and edges show causal relation with a sink feature node, thereby revealing how features interact that affect the decision of a model. \cite{feature_circuit} construct the circuit by first filter out important features to model prediction via attribution patching \cite{attrib}. For each filtered features, they find the edge weights by approximate the causal relations between the remaining features with attribution patching, only the strong connections are retained. \cite{transcoder_circuit} propose Transcoder, a variant of SAE, to approximate the language model as a linear interaction between features that is input independent, which aids the understanding feature interaction generalized across any input. Based on their work, \cite{on_biology} identify various interesting circuits and feature interations. The main downside of circuits is the complexity of manually interpreting hundreds of features and thousands of connections. Researchers have to group circuits' nodes that share the same meaning into supernodes, helping to reveal high level mechanisms of the model that are often drown by the circuit complexity. However, all of this effort is done manually \cite{on_biology, neuronpedia}, hence, unscalable to large circuits. 

\textbf{Optimal Transport (OT):}
OT provides a geometry for comparing probability distributions by measuring the minimum cost of moving mass between them \cite{halmoshierarchical}. It has been widely used in machine learning when the object of interest is more naturally represented as a distribution than as a single vector, including domain adaptation \cite{naram2025theoretical}, representation learning \cite{liu2025hgot, lv2024wasserstein}, dataset comparison \cite{nguyen2025lightspeed, bonet2025flowing}, graph matching \cite{ratnayakalearning}, shape correspondence \cite{le2024integrating}, and NLP similarity \cite{kishino2025quantifying, truong2025emo}. These works show that OT is useful when the geometry of the support matters and pointwise vector comparison fails to capture important distributional structure.

Most OT-based alignment methods compare distributions over external objects, such as samples from different domains, graphs, shapes, documents, or embedding spaces. In our setting, the object being compared is an internal SAE feature of language model. Rather than treating a feature as a decoder vector, we build a distribution from the contexts on which the feature activates and the relative strength of those activations. These activation-induced distributions are then projected into a shared reference space, allowing features from different layers to be compared in a common geometry. This makes Wasserstein distance a semantic distance between feature activation patterns, and supports both cross-layer feature matching and automatic grouping of circuit nodes into semantic supernodes.

\section{Experiment details}
\label{sec:exp_details}
In this section we detail feature matching, circuit compression, the baseline implementations, our experiment setups, SAE training, and LLM prompts.

\subsection{Algorithm for Feature Matching and Circuit Compression}
\label{sec:main_algo}
\textbf{Feature Matching:} In algorithm \ref{alg:matching}, we compute the distance for feature matching by first project the distributions of both target and source features set on the reference space. We then compute the centroids of the distribution on the space. The top-25 source features that are the closest to the target feature are retained, we discard the rest. This step is crucial to speed up the algorithm by not having to compute most of faraway source feature distributions. We then compute the top-1 feature using Wasserstein distance with Sinkhorn algorithm \cite{sinkhorn}.

\begin{algorithm}[H]
    \caption{Feature Matching Algorithm.}
    \label{alg:matching}
    \textbf{Input}: Target feature $i \in \mathcal{F}^{(A)}$, source features $ \mathcal{F}^{(B)}$, and empirical distributions $\widehat\mu_{i,T}^{(A)}, \: i \in \mathcal{F}^{(A)}$ and $\widehat\mu_{j,T}^{(B)}, \: \forall j \in \mathcal{F}^{(B)}$.
    \begin{algorithmic}[1]
        \STATE Compute reference space projected distributions $\widehat\mu_{i,T}^{(A \rightarrow \mathcal{R})}, \: i \in \mathcal{F}^{(A)}$ and $\widehat\mu_{j,T}^{(B\rightarrow \mathcal{R})}, \: \forall j \in \mathcal{F}^{(B)}$.
        \STATE Compute feature distributions' centroids $\bar\mu_{i,T}^{(A \rightarrow \mathcal{R})} = \sum_{t\in\mathcal{I}_{i, T}^{(A), K}}\widehat w_{i, t}^{(A)}z_t^{(A)}, \: i \in \mathcal{F}^{(A)}$ and $\bar\mu_{j, T}^{(B\rightarrow \mathcal{R})} = \sum_{t\in\mathcal{I}_{j, T}^{(B), K}}\widehat w_{j, t}^{(B)}z_t^{(B)}, \: \forall j \in \mathcal{F}^{(B)}$.
        \STATE Extract top-50 source features that have the closest centroids to target feature centroid $\mathrm{Top50}_{j\in\mathcal{F}^{(B)}}\left(-c\left(\bar\mu_{j,T}^{(B \rightarrow \mathcal{R})}, \bar\mu_{i,T}^{(A \rightarrow \mathcal{R})}\right)\right)$ \hfill $\triangleright$/ \textit{We do this to filter out faraway distributions, speeding up the algorithm. The choice of top-50 is arbitrary.}
        \STATE Compute the Wasserstein distance between target feature projected distribution $\widehat\mu_{i,T}^{(A \rightarrow \mathcal{R})}, \: i \in \mathcal{F}^{(A)}$ and source features projected distributions within the top-50 $\widehat\mu_{p,T}^{(B\rightarrow \mathcal{R})}, \: \forall p \in \mathrm{Top50}_{j\in\mathcal{F}^{(B)}}\left(-c\left(\bar\mu_{j,T}^{(B \rightarrow \mathcal{R})}, \bar\mu_{i,T}^{(A \rightarrow \mathcal{R})}\right)\right)$ via Sinkhorn algorithm \cite{sinkhorn}.
        \STATE $\widehat M_{A\to B}(i)$ is the closest source feature.
    \end{algorithmic}
    \textbf{Return}: $\widehat M_{A\to B}(i)$.
\end{algorithm}

\textbf{Circuit Compression:} We detail our circuit compression as follows: in algorithm \ref{alg:compression}, we first project the distribution onto reference space $\mathcal{R}$. We then compute the Wasserstein distance between any node pairs in the circuit using Sinkhorn \cite{sinkhorn}. Lastly, we apply Agglomerative clustering \cite{clustering} to find the supernodes.

\begin{algorithm}[H]
    \caption{Circuit Compression Algorithm.}
    \label{alg:compression}
    \textbf{Input}: Circuit $G=(V,E)$, number of supernodes $m \leq |V|$, and empirical distributions $\widehat\mu_{i,T}^{(\ell)}, \: (i, \ell) \in V$.
    \begin{algorithmic}[1]
        \STATE Compute reference space projected distributions $\widehat\mu_{i,T}^{(\ell\rightarrow \mathcal{R})}, \: (i, \ell) \in V$
        \STATE Compute Wasserstein distance via Sinkhorn \cite{sinkhorn} between any two pairs $u, v \in V$ resulting in $\widehat{\mathcal{D}}_{\mathcal{R}}(u,v), \:\forall u,v\in V, u\neq v$.
        \STATE Apply Agglomerate clustering \cite{clustering} to find supernodes set: $\mathbb{C}=\{\mathcal{C}_1, \dots, \mathcal{C}_m\}$.
    \end{algorithmic}
    \textbf{Return}: $\mathbb{C}$.
\end{algorithm}

\subsection{Baseline Implementations}
\label{sec:baseline_implementation}

We compute the distance matrix $\widehat{D}$ for each baseline (except for LLM Selection) as follows:

\textbf{SAE Match \cite{mech_permute}.} This method computes the feature distance via the $L_2$ norm of the threshold-normalized decoder vectors. We define threshold $a_{\min, i}^{(A)}$ as the minimum positive activation of feature $i$ across a corpus of $T$ tokens. We then compute the threshold-normalized decoder vector \cite{mech_permute} $\widetilde{W}_{dec}^{(A)}[:, i]$ by scaling the raw decoder direction by this minimum activation:
\begin{equation}
    \widetilde{W}_{dec}^{(A)}[:, i] = a_{\min, i}^{(A)} W_{dec}^{(A)}[:, i].
\end{equation}
The entries of the distance matrix are given by:
\begin{equation}
    \left(\widehat{D}_{\text{Match}}\right)_{i,j} = \left\| \widetilde{W}_{dec}^{(A)}[:, i] - \widetilde{W}_{dec}^{(B)}[:, j] \right\|_2
\end{equation}

\textbf{Feature Flow \cite{featflow}.} This method relies on the cosine similarity of the decoder vectors. Because cosine similarity is invariant to positive scalar multiplication, we can compute it using the raw decoder vectors. To convert this similarity into a valid distance metric for nearest-neighbor matching, we compute:
\begin{equation}
    \left(\widehat{D}_{\text{FeatFlow}}\right)_{i,j} = 1 - \frac{\left\langle W_{dec}^{(A)}[:, i], W_{dec}^{(B)}[:, j] \right\rangle}{\left\| W_{dec}^{(A)}[:, i] \right\|_2 \left\| W_{dec}^{(B)}[:, j] \right\|_2}
\end{equation}

\textbf{Attribution Patching \cite{attrib, feature_circuit}.} We consider method that find important features via gradient. For two feature $i \in \mathcal{F}^{(A)}$ and $j \in \mathcal{F}^{(B)}$ and $A > B$, feature $i$ has the gradient over $j$. We approximate the contribution of feature $a_j^{(B)}$ to feature $a_i^{(A)}$ via gradient over a dataset of tokens. The higher the contribution, the more important feature $j$ is to $i$. This method is used to find relevant features from previous layers \cite{feature_circuit}. Concretely, we define the distance between the two features as:
\begin{equation}
    \left(\widehat{D}_{\text{Attrib}}\right)_{i,j} = \frac{1}{T}\sum^T_{t=1}\nabla_{a_j^{(B)}} a_i^{(A)}(x^{(A)}_t) \times a_j^{(B)}(x^{(B)}_t)
\end{equation}

\textbf{LLM Selection.} Given a set of source features $\mathcal{F}^{(B)}$, we prompt the LLM to caption for each feature using LLM Captioning Prompt Appendix \ref{sec:prompt}. For each target feature $i \in \mathcal{F}^{(A)}$, we ask the LLM, from all captions of feature set $\mathcal{F}^{(B)}$, to pick the caption that most resembles the activation of feature $i$, using LLM Selection Prompt Appendix \ref{sec:prompt}.

\subsection{Feature Matching Setups}
\label{sec:setup_feature_matching}
We randomly sample 100 pairs shared across all baselines where the features have at least 30 non-zero activation over 1.28M tokens for Gemma-2-2B (10,000 sequences of 128 tokens from the Pile) and 128K tokens for GPT2 (1,000 sequences) to avoid dead features. We evaluate three main setups: consecutive layers, middle to last layers, first to last layers, shown in Table \ref{tab:feature_matching_setup}.
For our method, each feature distribution keeps its $K=500$ highest-activation hidden states, and for each target feature we compute the OT distance to the $M=25$ source features with the closest centroids in the reference space. 
We use the judge gpt-oss-120b \cite{oss} for this experiment.

\begin{table}[h]
    \caption{Layer setups in feature matching experiment.}
    \centering
    \begin{tabular}{|l|c|c|c|c|}
        \toprule
        \textbf{Models} &  \multicolumn{2}{c|}{\textbf{GPT2}} &  \multicolumn{2}{c|}{\textbf{Gemma-2-2b}} \\
        \midrule
        & Target layer & Source layer & Target layer & Source Layer \\
        \midrule
        Consecutive & 6 & 5 & 13 & 12 \\
        Middle to last & 11 & 5 & 25 & 12 \\
        First to last & 11 & 0 & 25 & 0 \\
        \bottomrule
    \end{tabular}
    \label{tab:feature_matching_setup}
\end{table}

\subsection{Circuit Compression Setups}
\label{sec:circuit_compression_setup}
\textbf{Circuit construction and compressing:} We follow \cite{feature_circuit} to compute feature circuits, the node scores are approximated using attribution patching instead of integrated gradient to speed up the running process. We gather prompts by chunking the validation set of The PILE dataset \cite{thepile} into 1000 sequences, each with 128 tokens. We extract the token sequences that the model correctly predicts the next tokens with Cross Entropy loss less than 0.1 nats. We then randomly sample 200 prompt examples shared across all baselines and compute feature circuits with 50 nodes. For each method except for LLM Sel, we compute the distance matrix between each feature pair, then use agglomerate clustering \cite{clustering} to form 5 supernodes. For LLM Sel, we show all the activations of features in the circuit and ask the model to caption each feature (using LLM Captioning Prompt Appendix \ref{sec:prompt}), we then ask the LLM to cluster the features that have similar caption into 5 supernodes (using LLM Selection Prompt Appendix \ref{sec:prompt}).
For the larger-circuit setup, we use the same 200 prompts, compute feature circuits with 200 nodes and form 20 supernodes; all other settings are unchanged. 

\textbf{Evaluation:} To evaluate the compression, for each supernode, we extract two equal sets of features. The first set is used to prompt a judge LLM - we use Llama-3.3-70b-instruct \cite{meta_llama}, where we ask the judge LLM to caption all features of the first set and give the caption for the supernode. We then use the second set, randomly mixed with equal number of features that are not in the supernode, and ask the judge to classify which feature is in the original supernode. The better the prediction, the more intuitive the compression. In our experiment, we extract 6 random features from the supernode, splitted into two sets, and 3 random features not from the supernode. We skip supernodes that do not have at least 2 features.
Note that, for circuits in which no supernode has at least 2 features, we cannot grade these circuits and they are excluded from the average. 

\section{Ablation Study: Number of Hidden States per Feature}
\label{sec:topk_ablation}
Our method represents each feature by the $K$ hidden states with the highest activation. We vary $K \in \{1, 16, 256, 500\}$ on the same 100 target features as Table \ref{tab:feature_matching}, keeping every other setting fixed, and evaluate with the same LLM judge over 5 runs. $K=500$ is the setting used in the main experiments; it reproduces the published matches of our method for 86\%--98\% of the target features. With $K=1$, each feature is a single hidden state and the OT distance reduces to the distance between two points.

The results are in Table \ref{tab:topk_ablation}. On the consecutive-layer setups and on every Gemma-2-2B setup, LLM Eval and Acc increase from $K=1$ and level off by $K=16$--$256$; for example, on Gemma-2-2B layer 12 to 25 the LLM Eval rises from 1.59 at $K=1$ to 1.92 at $K=256$. On the two GPT2 setups with target layer 11, the scores change little with $K$.

\begin{table}[h]
    \footnotesize
    \centering
    \caption{Feature matching results of our method with $K$ hidden states per feature distribution.}
    \begin{tabular}{l|c|c|c|c}
        \toprule
        & \multicolumn{2}{c|}{\textbf{GPT2}} & \multicolumn{2}{c}{\textbf{Gemma-2-2B}} \\
        \textbf{K} & \textbf{LLM Eval} $\uparrow$ & \textbf{Acc} $\uparrow$ & \textbf{LLM Eval} $\uparrow$ & \textbf{Acc} $\uparrow$ \\
        \midrule
        \multicolumn{1}{c|}{} & \multicolumn{2}{c|}{\textbf{Layer 5 match with Layer 6}} & \multicolumn{2}{c}{\textbf{Layer 12 match with Layer 13}} \\
        \midrule
        1 & 2.24 $\pm$ 0.04 & 51.8 $\pm$ 1.8 & 1.82 $\pm$ 0.02 & 37.6 $\pm$ 1.7 \\
        16 & \textbf{2.59} $\pm$ \textbf{0.04} & \underline{71.6} $\pm$ \underline{2.2} & \textbf{2.39} $\pm$ \textbf{0.05} & \textbf{61.0} $\pm$ \textbf{2.8} \\
        256 & 2.57 $\pm$ 0.02 & 71.4 $\pm$ 1.4 & \underline{2.33} $\pm$ \underline{0.05} & \underline{59.4} $\pm$ \underline{2.3} \\
        500 & \underline{2.58} $\pm$ \underline{0.03} & \textbf{72.2} $\pm$ \textbf{1.7} & \underline{2.33} $\pm$ \underline{0.06} & 59.2 $\pm$ 3.2 \\
        \midrule
        \multicolumn{1}{c|}{} & \multicolumn{2}{c|}{\textbf{Layer 5 match with Layer 11}} & \multicolumn{2}{c}{\textbf{Layer 12 match with Layer 25}} \\
        \midrule
        1 & 1.53 $\pm$ 0.02 & 21.0 $\pm$ 1.4 & 1.59 $\pm$ 0.04 & 23.4 $\pm$ 2.1 \\
        16 & \underline{1.56} $\pm$ \underline{0.03} & \underline{25.2} $\pm$ \underline{1.4} & \underline{1.72} $\pm$ \underline{0.05} & 29.2 $\pm$ 2.2 \\
        256 & 1.47 $\pm$ 0.03 & 19.6 $\pm$ 1.5 & \textbf{1.92} $\pm$ \textbf{0.03} & \textbf{38.6} $\pm$ \textbf{2.5} \\
        500 & \textbf{1.61} $\pm$ \textbf{0.05} & \textbf{26.2} $\pm$ \textbf{2.4} & \textbf{1.92} $\pm$ \textbf{0.04} & \underline{37.6} $\pm$ \underline{1.1} \\
        \midrule
        \multicolumn{1}{c|}{} & \multicolumn{2}{c|}{\textbf{Layer 0 match with Layer 11}} & \multicolumn{2}{c}{\textbf{Layer 0 match with Layer 25}} \\
        \midrule
        1 & \textbf{1.49} $\pm$ \textbf{0.01} & \textbf{19.2} $\pm$ \textbf{0.7} & 1.58 $\pm$ 0.02 & 22.2 $\pm$ 2.3 \\
        16 & \underline{1.46} $\pm$ \underline{0.05} & \underline{19.0} $\pm$ \underline{2.1} & 1.78 $\pm$ 0.06 & 32.0 $\pm$ 3.3 \\
        256 & \underline{1.46} $\pm$ \underline{0.02} & \textbf{19.2} $\pm$ \textbf{0.7} & \underline{1.80} $\pm$ \underline{0.03} & \underline{34.0} $\pm$ \underline{2.5} \\
        500 & 1.38 $\pm$ 0.04 & 15.4 $\pm$ 1.9 & \textbf{1.83} $\pm$ \textbf{0.05} & \textbf{34.8} $\pm$ \textbf{2.7} \\
        \bottomrule
    \end{tabular}
    \label{tab:topk_ablation}
\end{table}

\subsection{Digit Addition Experiment Setups}
\label{sec:digit_addition_setup}
\textbf{Dataset creation:} We use ``two number addition" feature \cite{on_biology} in the format $p_1+p_2=$. We limit $p_1, p_2 \in [0,9]$ for a less costly experiment to run LLM Selection method. We extract the features' activation on the ``=" token over 100 datapoints ($[0,9]\times[0,9]$) and plot the activation pattern of each feature. We filter features that have the number of positive activation in the range of [5, 50] to avoid overly dense or dead features. 

\textbf{Setup:} We conduct our experiment on layer 19 MLP on Gemma-2-2b. For each target feature in the filtered set, we run the matching baselines to find the most similar source feature within the same filtered set that is not the target. Note that, since this experiment is conducted on a single layer, we cannot apply Attribution Patching \cite{attrib} for this setup. We then create blind ranking API for all of the methods and let 10 volunteers to rate over 100 random target features.

\textbf{Blind ranking API:} The API example is in Figure \ref{fig:blind_ranking}. There are two choices, ``valid" means that the source feature has the same meaning with the target feature, and ``best" means that the source feature is the best selection among all blind ranking matches. 

\begin{figure}
    \centering
    \includegraphics[width=0.8\linewidth]{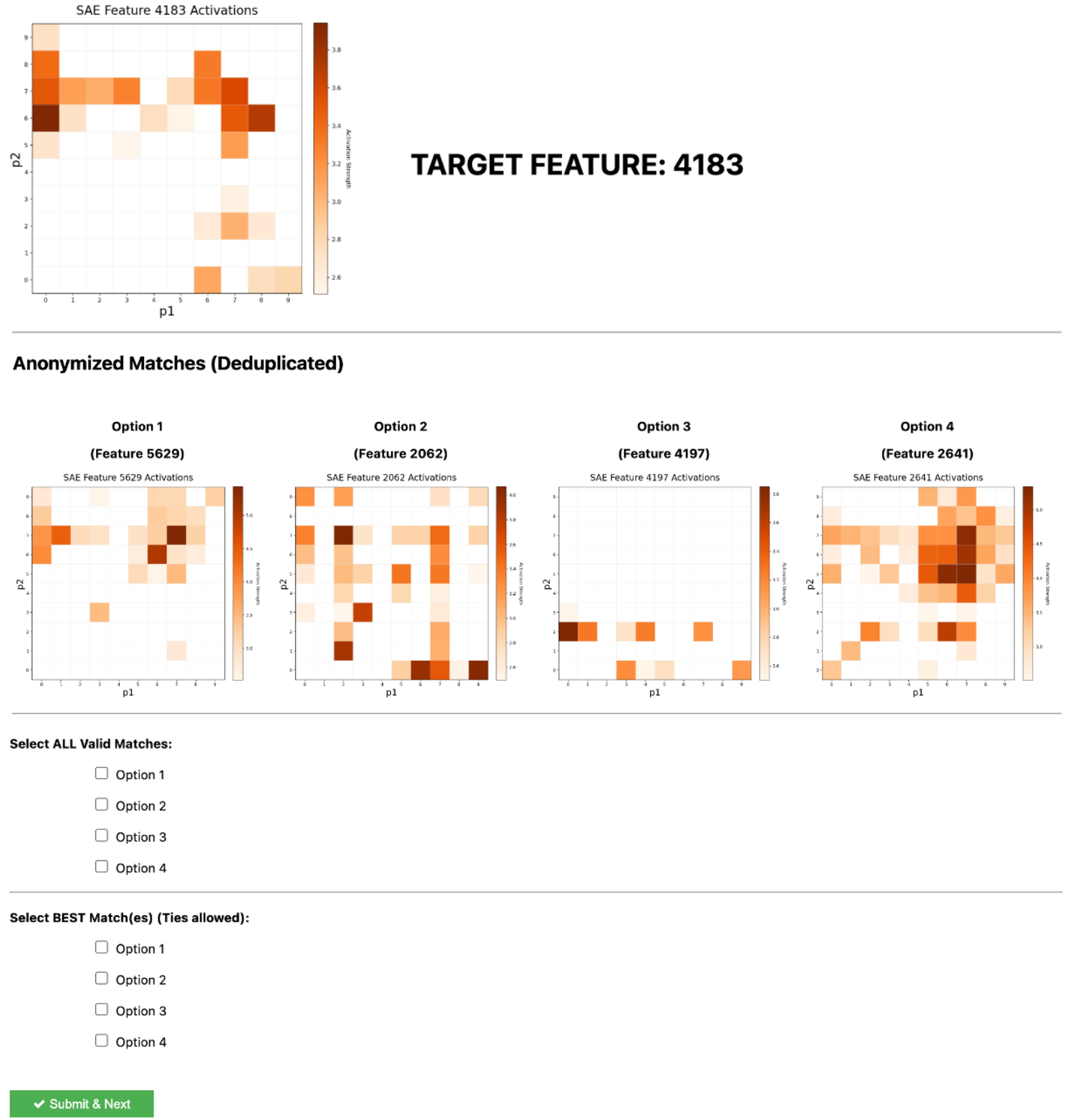}
    \caption{Blind ranking API for ``Digit Addition" experiment (Section \ref{sec:digit_addition}). There are two options of ``Valid" and ``Best" for a valid matching and best overall pick respectively. The volunteers can choose multiple choices or skip if no feature candidate matches.}
    \label{fig:blind_ranking}
\end{figure}

\subsection{Automated Circuit Compression Setups}
\label{sec:automated_circuit_compress_setup}
We use three prompts shown on Neuropedia for Gemma-2-2b, namely: \textit{``<BOS>Fact: the capital of the state containing Dallas is Austin"}, \textit{``<BOS>A rhyming couplet: he saw a carrot and had to grab it, his hunger was like a starving rabbit"}, and \textit{``<BOS>The opposite of small is large"}. Following the use of corrupt prompt \cite{feature_circuit}, we use corrupt prompt of \textit{``Fact: the city of the country containing Paris is"} and not use corrupt prompt for the other two circuits (this means the corrupt activation of each feature is zero \cite{feature_circuit}). We compute the circuit using attribution patching to find important nodes, the number of nodes and supernodes are in Table \ref{tab:circuit_setup}.
\begin{table}[h]
    \caption{Exemplar circuit setups for Section \ref{application:circuit_compression}.}
    \centering
    \begin{tabular}{|l|c|c|c|c|}
        \toprule
        \textbf{Models} &  \multicolumn{2}{c|}{\textbf{Gemma-2-2b}} \\
        \midrule
        & Nodes & Supernodes \\
        \midrule
        ``Dallas" circuit & 50 & 7 \\
        ``Rhythm" circuit & 100 & 10\\
        ``Opposite" circuit & 40 & 9 \\
        \bottomrule
    \end{tabular}
    \label{tab:circuit_setup}
\end{table}

\subsection{Finding Modular Feature Group Setups}
\label{sec:modular_setup}
\textbf{Behavior clustering:} The following procedure follows \cite{feature_circuit}. We gather prompts by chunking the validation set of The PILE dataset \cite{thepile} into 1000 sequences, each with 128 tokens. We extract the token sequences that the model correctly predicts the next tokens with Cross Entropy loss less than 0.1 nats. For each prompt, we extract the last token representation (which is the token with low Cross Entropy loss) of every layers and concatenate into one vector. We then use these vectors to cluster using K-means clustering. We identify around 18000 prompt examples that are clustered into 500 groups.

\textbf{Finding Modular Feature:} We select a few interesting groups such as in Figure \ref{fig:html_tag_modular}, \ref{fig:incrementing_seq_modular}, and \ref{fig:apostrophe_modular}. For each prompt example in a group, we then find a circuit (using method in \cite{feature_circuit} with attribution patching to select important nodes) with 200 nodes and compress it down to 10 supernodes. For each node pairs $u, v$ in the circuit, we track the following Jaccard similarity: $\frac{number\_of\_co\_occurence}{number\_of\_u \: + \: number\_of\_v \: - \: number\_of\_co\_occurence}$ where $number\_of\_co\_occurence$ counts when $u$ and $v$ are in the same supernode, and $number\_of\_u$ tracks the number of $u$ appears in the circuit group. After this, we form an ``edge" matrix between the all feature node pairs, and we apply Spectral clustering to find the modular group. Optionally, we can filter out ``edge" between feature pairs that have low co-occurence rate before clustering to avoid unimportant pairs. 

\subsection{SAE Training}
\label{sec:sae_training}
We train TopK SAE \cite{topk} on The Pile training set for 500M tokens on GPT2 at ``pre-residual-stream" (before entering transformer blocks) of every layers. The learning rate is $10^{-4}$, batch size of 2048 tokens, and auxiliary loss coefficient is 0.03125 and use Adam Optimizer. We run on A100 GPU.

\subsection{LLM Prompt}
\label{sec:prompt}
\textbf{LLM Compressing Prompt:} 
{\itshape
"""

You are an expert AI interpretability researcher. Your task is to analyze a list of Sparse Autoencoder (SAE) features extracted from a transformer language model and group them into semantically meaningful clusters.

DATA PROVIDED:

- Feature ID: Unique identifier for the feature.

- Act Context: Shows text intervals where the feature fired (activating tokens). THIS IS THE GOLD STANDARD for interpretation. Look for patterns (e.g., tracking plural nouns, specific concepts, or syntax).

- Logits: Shows tokens the feature promotes or suppresses. Use this to confirm the feature's meaning.

YOUR TASK:

1. Analyze the "Act Context" and "Logits" for every feature provided.

2. Write a short, general 1-sentence caption explaining what each feature detects or represents.

3. Group features that track similar concepts, syntactical roles, or semantic meanings into the same cluster. 

4. Output the result STRICTLY as a JSON object matching the schema below. Do not include any markdown formatting outside the JSON block.

CRITICAL RULES:

- EVERY feature ID provided in the prompt MUST be assigned to exactly one cluster. Do not leave any feature behind.

- You MUST group the features into EXACTLY {n$\_$clusters} clusters. Number the cluster$\_$id from 0 to \{n$\_$clusters - 1\}.

- Each cluster has AT LEAST one feature.
"""
}

\textbf{LLM Captioning Prompt:}
{\itshape
"""

You are an AI interpretability assistant analyzing neural network features. You must respond in valid JSON.

Here are the highest activating text snippets for Feature \{feat$\_$idx\}:\{context$\_$str$\_$joined\}

Describe what concept, pattern, or semantic meaning this feature activates on in EXACTLY ONE SHORT SENTENCE. 

Do not add any conversational filler.

Output your response strictly as a JSON object with a single key 'caption'.
"""
}

\textbf{LLM Selection Prompt:}
{\itshape
"""

You are an AI interpretability assistant analyzing neural network features. You must respond in valid JSON.

You are given a full catalog of features and their single-sentence descriptions.

FULL FEATURE CATALOG:\{full$\_$catalog$\_$str\}

Your task is to find the top \{k\} most semantically related features for the following target features.

TARGET FEATURES TO MATCH: \{target$\_$str\}

RULES:

1. Return EXACTLY \{k\} matches for each target feature.

2. DO NOT match a target feature to itself.

3. Output your response STRICTLY as a JSON object where keys are the target feature IDs (as strings) and values are lists of \{k\} integer feature IDs.
"""
}

\section{Ablation Study: Why do we need feature as distribution?}
\label{sec:feat_as_dist}

In this section, we run ablation experiments showing that representing feature as distribution preserves more semantic information than using one vector. We implement the following baseline, which first project the activation hidden states to the reference space but then average the distribution into one singular vector. By comparing with this baseline, we will show that our method surpass representing via one vector, even when projecting onto the same manifold:

\textbf{Naive Manifold Projection (Naive).} This baseline projects features into the shared reference space $\mathcal{R}$ but ignores the full distributional geometry. Instead of computing the Wasserstein transport plan, it collapses each projected feature distribution into a single activation-weighted centroid. For a source feature $j \in \mathcal{F}^{(B)}$, the centroid is defined as the expected position of its activated contexts in the reference space:
\begin{equation}
    \bar{z}_j^{(B)} = \sum_{t \in \{t, \: a_{j, t}^{(B)} \geq 0\}} \widehat{w}_{j,t}^{(B)} z_t^{(\mathcal{R})}
\end{equation}
We similarly define the centroid $\bar{z}_i^{(A)}$ for a target feature $i \in \mathcal{F}^{(A)}$. The distance matrix is then computed by evaluating the ground cost $c(\cdot, \cdot)$ strictly between these centroids, discarding the multi-modal activation structure:
\begin{equation}
    \left(\widehat{D}_{\text{Naive}}\right)_{i,j} = c\left(\bar{z}_i^{(A)}, \bar{z}_j^{(B)}\right)
\end{equation}

We run the same experiments as in Section \ref{exps: feature matching}, \ref{sec:circuit_compression}, \ref{sec:digit_addition}; the results are in Table \ref{tab:feature_matching_compare_naive}, Table \ref{tab:circuit_compression_compare_naive}, and Figure \ref{fig:digit_compare_naive} respectively. We find that for feature matching and digit addition, our method have stronger performance on the majority of the setups. Similarly, our method is better than the naive method for circuit compression, where the gap is significant on Gemma-2-2B and on 200-node circuits for both models. This shows that represent features as distributions preserve more semantic information than representing with one vector, even in the setup of projecting to the same manifold.

\begin{table}[h]
    \footnotesize
    \centering
    \vspace{-2mm}
    \caption{Feature matching experiment results when comparing with Naive method.}
    \begin{tabular}{l|c|c|c|c|c|c|c|c}
        \toprule
        & \multicolumn{4}{c|}{\textbf{GPT2}} & \multicolumn{4}{c}{\textbf{Gemma-2-2B}} \\
        \textbf{Method} & \textbf{LLM Eval} $\uparrow$ & \textbf{Acc} $\uparrow$ & \textbf{CE} $\downarrow$ & \textbf{VE} $\uparrow$ & \textbf{LLM Eval} $\uparrow$ & \textbf{Acc} $\uparrow$ & \textbf{CE} $\downarrow$ & \textbf{VE} $\uparrow$ \\
        \midrule
        \multicolumn{1}{c|}{} & \multicolumn{4}{c|}{\textbf{Layer 5 match with Layer 6}} & \multicolumn{4}{c}{\textbf{Layer 12 match with Layer 13}} \\
        \midrule
        \textbf{Our} & \textbf{2.53} $\pm$ \textbf{0.04} & \textbf{68.2} $\pm$ \textbf{1.9} & \textbf{3.775} & \textbf{0.9866} & \textbf{2.32} $\pm$ \textbf{0.04} & \textbf{58.6} $\pm$ \textbf{2.0} & \textbf{2.74784} & \textbf{0.7196} \\
        Our Naive & 2.51 $\pm$ 0.04 & 66.4 $\pm$ 2.3 & 3.777 & 0.9865 & 2.22 $\pm$ 0.04 & 55.2 $\pm$ 2.1 & 2.74799 & 0.7183 \\
        \midrule
        \multicolumn{1}{c|}{} & \multicolumn{4}{c|}{\textbf{Layer 5 match with Layer 11}} & \multicolumn{4}{c}{\textbf{Layer 12 match with Layer 25}} \\
        \midrule
        \textbf{Our} & \textbf{1.56} $\pm$ \textbf{0.04} & \textbf{23.2} $\pm$ \textbf{2.2} & \textbf{3.806} & \textbf{0.9127} & 1.83 $\pm$ 0.04 & 33.6 $\pm$ 2.9 & 2.74809 & \textbf{0.9115} \\
        Our Naive & \textbf{1.56} $\pm$ \textbf{0.04} & 22.4 $\pm$ 2.3 & 3.822 & 0.9117 & \textbf{1.94} $\pm$ \textbf{0.02} & \textbf{39.6} $\pm$ \textbf{0.7} & \textbf{2.74806} & \textbf{0.9115} \\
        \midrule
        \multicolumn{1}{c|}{} & \multicolumn{4}{c|}{\textbf{Layer 0 match with Layer 11}} & \multicolumn{4}{c}{\textbf{Layer 0 match with Layer 25}} \\
        \midrule
        \textbf{Our} & 1.39 $\pm$ 0.03 & 17.2 $\pm$ 2.3 & 3.828 & 0.9111 & \textbf{1.83} $\pm$ \textbf{0.05} & \textbf{34.8} $\pm$ \textbf{2.1} & 2.74866 & \textbf{0.9094} \\
        Our Naive & \textbf{1.48} $\pm$ \textbf{0.01} & \textbf{19.6} $\pm$ \textbf{1.1} & \textbf{3.820} & \textbf{0.9116} & 1.78 $\pm$ 0.03 & 32.6 $\pm$ 1.6 & \textbf{2.74862} & \textbf{0.9094} \\
        \bottomrule
    \end{tabular}
    \vspace{-3mm}
    \label{tab:feature_matching_compare_naive}
\end{table}

\begin{table}[h]
    \centering
    \small
    \setlength{\tabcolsep}{4pt}
    \caption{Circuit compression result comparing with Naive method.}
    \begin{tabular}{l|c|c|c|c}
        \toprule
        & \multicolumn{2}{c|}{\textbf{50 nodes, 5 supernodes}} & \multicolumn{2}{c}{\textbf{200 nodes, 20 supernodes}} \\ 
        \textbf{Method} & \textbf{Gemma-2-2B} & \textbf{GPT2-small} & \textbf{Gemma-2-2B} & \textbf{GPT2-small} \\
        \midrule
        \textbf{Our} & \textbf{0.6047} $\pm$ \textbf{0.0144} & \textbf{0.6883} $\pm$ \textbf{0.0209} & \textbf{0.6404} $\pm$ \textbf{0.0069} & \textbf{0.7368} $\pm$ \textbf{0.0100} \\
        Our Naive & 0.5641 $\pm$ 0.0142 & 0.6702 $\pm$ 0.0207 & 0.6239 $\pm$ 0.0079 & 0.7097 $\pm$ 0.0094 \\
        \bottomrule
    \end{tabular}
    \label{tab:circuit_compression_compare_naive}
\end{table}

\begin{figure}
    \centering
    \includegraphics[width=0.6\linewidth]{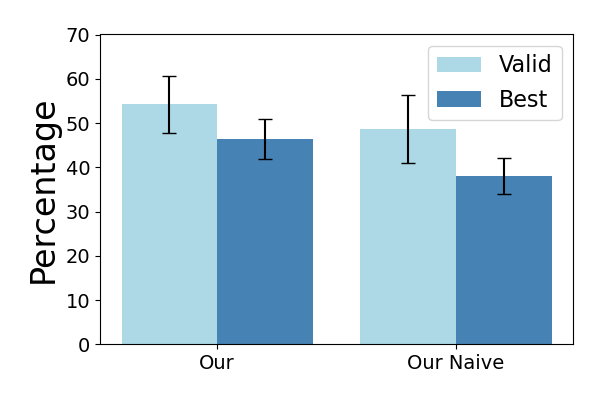}
    \caption{Digit addition comparing with Naive method. ``Valid" means the user agree with the matching, and ``Best" means the method has the best match among all methods.}
    \label{fig:digit_compare_naive}
\end{figure}

\section{Additional Theoretical Properties}
\label{sec:additional_theory}

\subsection{Activation Rescaling Invariance of Projected Feature Distributions}
\label{app:activation_rescaling_invariance}

\noindent
\textbf{Proposition~\ref{prop:activation_rescaling_invariance}
(Activation rescaling invariance of projected feature distributions).}
\emph{
Let $i\in\mathcal{F}^{(A)}$ be a target feature with activation function
$a_i^{(A)}:\mathcal{X}^{(A)}\to\mathbb{R}_{\ge 0}$. For any $\alpha>0$, define
$\widetilde a_i^{(A)}(x)=\alpha a_i^{(A)}(x)$. Under the deterministic
$\operatorname{TopK}$ convention, the projected empirical distribution induced
by $\widetilde a_i^{(A)}$ is identical to the one induced by $a_i^{(A)}$:
$\widetilde{\mu}_{i,T}^{(A\to\mathcal R)}
=
\widehat{\mu}_{i,T}^{(A\to\mathcal R)}$.
Consequently, for any source feature $j\in\mathcal{F}^{(B)}$ and any ground cost
$c:\mathcal R\times\mathcal R\to\mathbb{R}_{\ge 0}$,
$W_c(
\widetilde{\mu}_{i,T}^{(A\to\mathcal R)},
\widehat{\nu}_{j,T}^{(B\to\mathcal R)})
=
W_c(
\widehat{\mu}_{i,T}^{(A\to\mathcal R)},
\widehat{\nu}_{j,T}^{(B\to\mathcal R)})$.
}

\begin{proof}
Let $\mathcal I=\mathcal{I}_{i,T}^{(A),K}$ denote the top-$K$ index set selected
by the original activation values $a_i^{(A)}(x_t^{(A)})$, and let
$\widetilde{\mathcal I}=\widetilde{\mathcal{I}}_{i,T}^{(A),K}$ denote the
top-$K$ index set selected by the rescaled activation values
$\widetilde a_i^{(A)}(x_t^{(A)})=\alpha a_i^{(A)}(x_t^{(A)})$. We first show
that the selected indices are unchanged. For any token positions $p,q\in[T]$,
because $\alpha>0$,
\[
\begin{aligned}
a_i^{(A)}(x_p^{(A)})\ge a_i^{(A)}(x_q^{(A)})
&\Longleftrightarrow
\alpha a_i^{(A)}(x_p^{(A)})\ge \alpha a_i^{(A)}(x_q^{(A)})\\
&\Longleftrightarrow
\widetilde a_i^{(A)}(x_p^{(A)})\ge
\widetilde a_i^{(A)}(x_q^{(A)}).
\end{aligned}
\]
Thus every pairwise ordering relation among token activations is preserved.
Moreover, equality is also preserved:
\[
a_i^{(A)}(x_p^{(A)})=a_i^{(A)}(x_q^{(A)})
\Longleftrightarrow
\widetilde a_i^{(A)}(x_p^{(A)})=
\widetilde a_i^{(A)}(x_q^{(A)}).
\]
Therefore the rescaling neither changes the ranking nor creates or removes ties.
Since $\operatorname{TopK}$ uses a deterministic tie-breaking convention, the
same ordered top-$K$ candidates are selected before and after rescaling. Hence $\widetilde{\mathcal I}=\mathcal I.$ We next compare the normalized activation weights. By construction, for every
selected token position $t\in\mathcal I=\widetilde{\mathcal I}$,
\[
\begin{aligned}
\widetilde w_{i,t}^{(A)}
&=
\frac{\widetilde a_i^{(A)}(x_t^{(A)})}
{\sum_{s\in\widetilde{\mathcal I}}\widetilde a_i^{(A)}(x_s^{(A)})}\\
&=
\frac{\alpha a_i^{(A)}(x_t^{(A)})}
{\sum_{s\in\mathcal I}\alpha a_i^{(A)}(x_s^{(A)})}
\qquad
\text{since } \widetilde{\mathcal I}=\mathcal I\\
&=
\frac{\alpha a_i^{(A)}(x_t^{(A)})}
{\alpha\sum_{s\in\mathcal I}a_i^{(A)}(x_s^{(A)})}\\
&=
\frac{a_i^{(A)}(x_t^{(A)})}
{\sum_{s\in\mathcal I}a_i^{(A)}(x_s^{(A)})}
=
\widehat w_{i,t}^{(A)}.
\end{aligned}
\]
Thus, global positive rescaling preserves not only which token positions are
selected, but also the probability mass assigned to every selected token
position. It remains to compare the projected empirical distributions. In both
constructions, the token position $t$ is projected to the same support point
$z_t^{(\mathcal R)}\in\mathcal R$. Since the selected index set and the
normalized weights are both unchanged, we have
\[
\begin{aligned}
\widetilde{\mu}_{i,T}^{(A\to\mathcal R)}
&=
\sum_{t\in\widetilde{\mathcal I}}
\widetilde w_{i,t}^{(A)}\delta_{z_t^{(\mathcal R)}}\\
&=
\sum_{t\in\mathcal I}
\widetilde w_{i,t}^{(A)}\delta_{z_t^{(\mathcal R)}}
\qquad
\text{since } \widetilde{\mathcal I}=\mathcal I\\
&=
\sum_{t\in\mathcal I}
\widehat w_{i,t}^{(A)}\delta_{z_t^{(\mathcal R)}}
\qquad
\text{since } \widetilde w_{i,t}^{(A)}=\widehat w_{i,t}^{(A)}
\text{ for all } t\in\mathcal I\\
&=
\widehat{\mu}_{i,T}^{(A\to\mathcal R)}.
\end{aligned}
\]
Therefore the projected empirical distribution induced by the rescaled
activation function is exactly the same probability distribution as the original
projected empirical distribution. Finally, Wasserstein distance depends only on its input distributions and the ground
cost. Since the source distribution is identical in the two comparisons and the source
distribution $\widehat{\nu}_{j,T}^{(B\to\mathcal R)}$ is fixed,
\[
\begin{aligned}
W_c\!\left(
\widetilde{\mu}_{i,T}^{(A\to\mathcal R)},
\widehat{\nu}_{j,T}^{(B\to\mathcal R)}
\right)
&=
W_c\!\left(
\widehat{\mu}_{i,T}^{(A\to\mathcal R)},
\widehat{\nu}_{j,T}^{(B\to\mathcal R)}
\right).
\end{aligned}
\]
This proves both the invariance of the projected feature distribution and the
invariance of all downstream Wasserstein scores involving this source feature.
\end{proof}

\paragraph{Interpretation.}
This proposition shows that the projected activation-weighted representation
depends on the relative activation profile of a feature, not on its raw
activation scale. Positive global rescaling preserves the selected token
positions, normalized activation weights, projected empirical distribution, and
all downstream Wasserstein scores. This is important because SAE features from
different layers or training runs may have different activation scales, while
their semantic identity should be determined by where they activate strongly and
by their relative activation pattern.

\subsection{Flow Cost as a Refinement of Positional Cost}
\label{app:flow_cost_refinement}

We state the result in the shared reference space $\mathcal R$, where all
projected feature distributions are compared. In particular, one may take
$\mu=\widehat{\mu}_{i,T}^{(A\to\mathcal R)}$ and
$\nu=\widehat{\nu}_{j,T}^{(B\to\mathcal R)}$.

\noindent
\textbf{Proposition 2
(Flow cost refines positional cost).}
\emph{
Let $\phi:\mathcal R\to\mathbb R^p$ be a positional representation map. For
$m=1,\ldots,M$, let
$\overline u_\theta(\tau_m,\cdot):\mathcal R\to\mathbb R^s$ denote the
normalized flow field evaluated at time point $\tau_m$. For
$z,z'\in\mathcal R$, define the positional cost $c_{\mathrm{pos}}(z,z')
=
\|\phi(z)-\phi(z')\|_2^2,$ the directional cost $c_{\mathrm{dir}}(z,z')
=
\frac{1}{M}
\sum_{m=1}^{M}
\left\|
\overline u_\theta(\tau_m,z)
-
\overline u_\theta(\tau_m,z')
\right\|_2^2,$ and the composite flow cost
\[
c_{\mathrm{flow}}(z,z')
=
\lambda_x c_{\mathrm{pos}}(z,z')
+
\lambda_u c_{\mathrm{dir}}(z,z'),
\qquad
\lambda_x,\lambda_u\ge0.
\]
Then, for any $\mu,\nu\in\mathcal P(\mathcal R)$,
\[
W_{c_{\mathrm{flow}}}(\mu,\nu)
\ge
\lambda_x W_{c_{\mathrm{pos}}}(\mu,\nu),
\qquad
W_{c_{\mathrm{flow}}}(\mu,\nu)
\ge
\lambda_u W_{c_{\mathrm{dir}}}(\mu,\nu),
\]
and more strongly, $W_{c_{\mathrm{flow}}}(\mu,\nu)
\ge
\lambda_x W_{c_{\mathrm{pos}}}(\mu,\nu)
+
\lambda_u W_{c_{\mathrm{dir}}}(\mu,\nu).$
}

\begin{proof}
By the definition of Wasserstein cost with ground cost
$c_{\mathrm{flow}}$, and by the decomposition
$c_{\mathrm{flow}}=\lambda_x c_{\mathrm{pos}}+\lambda_u c_{\mathrm{dir}}$, we
have
\[
\begin{aligned}
W_{c_{\mathrm{flow}}}(\mu,\nu)
&=
\inf_{\pi\in\Pi(\mu,\nu)}
\mathbb E_{(z,z')\sim\pi}
\left[
c_{\mathrm{flow}}(z,z')
\right] \\
&=
\inf_{\pi\in\Pi(\mu,\nu)}
\mathbb E_{(z,z')\sim\pi}
\left[
\lambda_x c_{\mathrm{pos}}(z,z')
+
\lambda_u c_{\mathrm{dir}}(z,z')
\right] \\
&=
\inf_{\pi\in\Pi(\mu,\nu)}
\left\{
\lambda_x
\mathbb E_{\pi}
\left[
c_{\mathrm{pos}}(z,z')
\right]
+
\lambda_u
\mathbb E_{\pi}
\left[
c_{\mathrm{dir}}(z,z')
\right]
\right\}.
\end{aligned}
\]
For any fixed coupling $\pi\in\Pi(\mu,\nu)$, its expected positional cost is at
least the optimal positional transport cost, and its expected directional cost
is at least the optimal directional transport cost:
\[
\begin{aligned}
\mathbb E_{\pi}
\left[
c_{\mathrm{pos}}(z,z')
\right]
&\ge
\inf_{\gamma\in\Pi(\mu,\nu)}
\mathbb E_{\gamma}
\left[
c_{\mathrm{pos}}(z,z')
\right]
=
W_{c_{\mathrm{pos}}}(\mu,\nu),\\
\mathbb E_{\pi}
\left[
c_{\mathrm{dir}}(z,z')
\right]
&\ge
\inf_{\gamma\in\Pi(\mu,\nu)}
\mathbb E_{\gamma}
\left[
c_{\mathrm{dir}}(z,z')
\right]
=
W_{c_{\mathrm{dir}}}(\mu,\nu).
\end{aligned}
\]
Since $\lambda_x,\lambda_u\ge0$, multiplying these two inequalities preserves
their directions. Hence, for every coupling $\pi\in\Pi(\mu,\nu)$,
\[
\begin{aligned}
&\lambda_x
\mathbb E_{\pi}
\left[
c_{\mathrm{pos}}(z,z')
\right]
+
\lambda_u
\mathbb E_{\pi}
\left[
c_{\mathrm{dir}}(z,z')
\right]\\
&\qquad\ge
\lambda_x W_{c_{\mathrm{pos}}}(\mu,\nu)
+
\lambda_u W_{c_{\mathrm{dir}}}(\mu,\nu).
\end{aligned}
\]
The right-hand side is independent of $\pi$. Therefore, taking the infimum over
all couplings on the left-hand side gives
\[
\begin{aligned}
W_{c_{\mathrm{flow}}}(\mu,\nu)
&=
\inf_{\pi\in\Pi(\mu,\nu)}
\left\{
\lambda_x
\mathbb E_{\pi}
\left[
c_{\mathrm{pos}}(z,z')
\right]
+
\lambda_u
\mathbb E_{\pi}
\left[
c_{\mathrm{dir}}(z,z')
\right]
\right\}\\
&\ge
\lambda_x W_{c_{\mathrm{pos}}}(\mu,\nu)
+
\lambda_u W_{c_{\mathrm{dir}}}(\mu,\nu).
\end{aligned}
\]
This proves the stronger inequality. The two individual lower bounds follow
from nonnegativity:
\[
\begin{aligned}
W_{c_{\mathrm{flow}}}(\mu,\nu)
&\ge
\lambda_x W_{c_{\mathrm{pos}}}(\mu,\nu)
+
\lambda_u W_{c_{\mathrm{dir}}}(\mu,\nu)
\ge
\lambda_x W_{c_{\mathrm{pos}}}(\mu,\nu),\\
W_{c_{\mathrm{flow}}}(\mu,\nu)
&\ge
\lambda_x W_{c_{\mathrm{pos}}}(\mu,\nu)
+
\lambda_u W_{c_{\mathrm{dir}}}(\mu,\nu)
\ge
\lambda_u W_{c_{\mathrm{dir}}}(\mu,\nu).
\end{aligned}
\]
Thus the composite flow cost lower-bounds both weighted component Wasserstein
costs and also lower-bounds their sum.
\end{proof}

\paragraph{Interpretation.}
This proposition shows that the flow-aware ground cost is a refinement of
positional geometry in the shared reference space $\mathcal R$. The positional
term $c_{\mathrm{pos}}$ compares where two projected activation distributions
lie, while the directional term $c_{\mathrm{dir}}$ compares how the normalized
flow field behaves around those reference-space points. The composite cost
$c_{\mathrm{flow}}$ therefore adds a directional penalty on top of positional
matching. The inequality implies that matching under $c_{\mathrm{flow}}$ is at
least as strict as matching under either weighted component alone. In particular,
two projected feature distributions may be close positionally but far under the
flow-aware cost if their directional fields disagree. When $\lambda_u=0$, the
cost reduces to ordinary positional Wasserstein matching:
$c_{\mathrm{flow}}(z,z')=\lambda_x c_{\mathrm{pos}}(z,z')$.

\subsection{Stability of Wasserstein Feature Distance}
\label{app:wasserstein_stability}

We prove the stability result in the shared reference space. Throughout this
section, let $(\mathcal R,\rho)$ be the metric space where all projected feature
distributions live. Let
$\mu,\widetilde\mu,\nu,\widetilde\nu\in\mathcal P(\mathcal R)$, where
$(\mu,\nu)$ denotes the original pair of projected feature distributions and
$(\widetilde\mu,\widetilde\nu)$ denotes the perturbed pair. Define the relevant
supports
\[
\mathcal S_\mu
=
\operatorname{supp}(\mu)\cup\operatorname{supp}(\widetilde\mu),
\qquad
\mathcal S_\nu
=
\operatorname{supp}(\nu)\cup\operatorname{supp}(\widetilde\nu).
\]
A ground cost $c:\mathcal R\times\mathcal R\to\mathbb R_{\ge0}$ is
coordinate-wise Lipschitz on the relevant support if there exist constants
$L_x,L_y<\infty$ such that
\[
|c(z,z')-c(\bar z,z')|
\le
L_x\rho(z,\bar z),
\qquad
\forall z,\bar z\in\mathcal S_\mu,\ \forall z'\in\mathcal S_\nu,
\]
and
\[
|c(z,z')-c(z,\bar z')|
\le
L_y\rho(z',\bar z'),
\qquad
\forall z\in\mathcal S_\mu,\ \forall z',\bar z'\in\mathcal S_\nu.
\]

\noindent
\textbf{Theorem~\ref{thm:wasserstein_stability}
(Stability of Wasserstein feature distance).}
\emph{
Let $(\mathcal R,\rho)$ be the shared reference metric space. Let
$c:\mathcal R\times\mathcal R\to\mathbb R_{\ge0}$ be coordinate-wise Lipschitz
on the relevant support with constants $L_x,L_y$. Then, for any projected
feature distributions
$\mu,\widetilde\mu,\nu,\widetilde\nu\in\mathcal P(\mathcal R)$,
\[
\left|
W_c(\mu,\nu)
-
W_c(\widetilde\mu,\widetilde\nu)
\right|
\le
L_xW_\rho(\mu,\widetilde\mu)
+
L_yW_\rho(\nu,\widetilde\nu),
\]
where
$W_\rho(\mu,\widetilde\mu)
=
\inf_{\alpha\in\Pi(\mu,\widetilde\mu)}
\mathbb E_{(z,\widetilde z)\sim\alpha}[\rho(z,\widetilde z)]$
is the Wasserstein-1 distance induced by the base metric $\rho$.
}

\begin{proof}
We first prove the one-sided inequality
\[
W_c(\mu,\nu)-W_c(\widetilde\mu,\widetilde\nu)
\le
L_xW_\rho(\mu,\widetilde\mu)+L_yW_\rho(\nu,\widetilde\nu).
\]
Fix any $\eta>0$. By the definition of infimum, choose couplings
$\alpha_\eta\in\Pi(\mu,\widetilde\mu)$,
$\beta_\eta\in\Pi(\widetilde\nu,\nu)$, and
$\widetilde\pi_\eta\in\Pi(\widetilde\mu,\widetilde\nu)$ such that
\[
\begin{aligned}
\mathbb E_{\alpha_\eta}[\rho(z,\widetilde z)]
&\le
W_\rho(\mu,\widetilde\mu)+\eta,\\
\mathbb E_{\beta_\eta}[\rho(\widetilde z',z')]
&\le
W_\rho(\widetilde\nu,\nu)+\eta
=
W_\rho(\nu,\widetilde\nu)+\eta,\\
\mathbb E_{\widetilde\pi_\eta}[c(\widetilde z,\widetilde z')]
&\le
W_c(\widetilde\mu,\widetilde\nu)+\eta.
\end{aligned}
\]
The first coupling connects $\mu$ to $\widetilde\mu$, the middle coupling
connects $\widetilde\mu$ to $\widetilde\nu$, and the third coupling connects
$\widetilde\nu$ to $\nu$. Since these couplings share the required intermediate
marginals, the gluing lemma gives a joint distribution
$\Lambda_\eta\in\mathcal P(\mathcal R^4)$ over
$(z,\widetilde z,\widetilde z',z')$ such that
\[
\begin{aligned}
(z,\widetilde z)_\#\Lambda_\eta
&=
\alpha_\eta,
&
(\widetilde z,\widetilde z')_\#\Lambda_\eta
&=
\widetilde\pi_\eta,
&
(\widetilde z',z')_\#\Lambda_\eta
&=
\beta_\eta.
\end{aligned}
\]
Let $\pi_\eta=(z,z')_\#\Lambda_\eta$. Since the first marginal of
$\Lambda_\eta$ is $\mu$ and the fourth marginal is $\nu$, we have
$\pi_\eta\in\Pi(\mu,\nu)$. Hence
\[
\begin{aligned}
W_c(\mu,\nu)
&=
\inf_{\kappa\in\Pi(\mu,\nu)}
\mathbb E_{\kappa}[c(z,z')]\\
&\le
\mathbb E_{\pi_\eta}[c(z,z')]
=
\mathbb E_{\Lambda_\eta}[c(z,z')].
\end{aligned}
\]

For every $(z,\widetilde z,\widetilde z',z')$ in the support of
$\Lambda_\eta$, we have
$z,\widetilde z\in\mathcal S_\mu$ and
$z',\widetilde z'\in\mathcal S_\nu$. Therefore, by coordinate-wise
Lipschitzness,
\[
\begin{aligned}
c(z,z')-c(\widetilde z,\widetilde z')
&=
\left[
c(z,z')-c(\widetilde z,z')
\right]
+
\left[
c(\widetilde z,z')-c(\widetilde z,\widetilde z')
\right]\\
&\le
\left|
c(z,z')-c(\widetilde z,z')
\right|
+
\left|
c(\widetilde z,z')-c(\widetilde z,\widetilde z')
\right|\\
&\le
L_x\rho(z,\widetilde z)
+
L_y\rho(z',\widetilde z')\\
&=
L_x\rho(z,\widetilde z)
+
L_y\rho(\widetilde z',z').
\end{aligned}
\]
Equivalently,
\[
c(z,z')
\le
c(\widetilde z,\widetilde z')
+
L_x\rho(z,\widetilde z)
+
L_y\rho(\widetilde z',z').
\]
Taking expectation under $\Lambda_\eta$ and using its prescribed marginals gives
\[
\begin{aligned}
\mathbb E_{\Lambda_\eta}[c(z,z')]
&\le
\mathbb E_{\Lambda_\eta}[c(\widetilde z,\widetilde z')]
+
L_x\mathbb E_{\Lambda_\eta}[\rho(z,\widetilde z)]
+
L_y\mathbb E_{\Lambda_\eta}[\rho(\widetilde z',z')]\\
&=
\mathbb E_{\widetilde\pi_\eta}[c(\widetilde z,\widetilde z')]
+
L_x\mathbb E_{\alpha_\eta}[\rho(z,\widetilde z)]
+
L_y\mathbb E_{\beta_\eta}[\rho(\widetilde z',z')]\\
&\le
\left[
W_c(\widetilde\mu,\widetilde\nu)+\eta
\right]
+
L_x
\left[
W_\rho(\mu,\widetilde\mu)+\eta
\right]
+
L_y
\left[
W_\rho(\nu,\widetilde\nu)+\eta
\right]\\
&=
W_c(\widetilde\mu,\widetilde\nu)
+
L_xW_\rho(\mu,\widetilde\mu)
+
L_yW_\rho(\nu,\widetilde\nu)
+
(1+L_x+L_y)\eta.
\end{aligned}
\]
Combining this with
$W_c(\mu,\nu)\le\mathbb E_{\Lambda_\eta}[c(z,z')]$ yields
\[
\begin{aligned}
W_c(\mu,\nu)-W_c(\widetilde\mu,\widetilde\nu)
&\le
L_xW_\rho(\mu,\widetilde\mu)
+
L_yW_\rho(\nu,\widetilde\nu)
+
(1+L_x+L_y)\eta.
\end{aligned}
\]
Since $\eta>0$ was arbitrary, letting $\eta\downarrow0$ gives
\[
W_c(\mu,\nu)-W_c(\widetilde\mu,\widetilde\nu)
\le
L_xW_\rho(\mu,\widetilde\mu)
+
L_yW_\rho(\nu,\widetilde\nu).
\]

For the reverse direction, apply the same argument with
$(\mu,\nu)$ and $(\widetilde\mu,\widetilde\nu)$ interchanged. The relevant
supports are unchanged because they are unions of the same original and
perturbed supports. Hence
\[
\begin{aligned}
W_c(\widetilde\mu,\widetilde\nu)-W_c(\mu,\nu)
&\le
L_xW_\rho(\widetilde\mu,\mu)
+
L_yW_\rho(\widetilde\nu,\nu)\\
&=
L_xW_\rho(\mu,\widetilde\mu)
+
L_yW_\rho(\nu,\widetilde\nu),
\end{aligned}
\]
where the last equality uses symmetry of $W_\rho$. Combining the two one-sided
bounds gives
\[
\left|
W_c(\mu,\nu)-W_c(\widetilde\mu,\widetilde\nu)
\right|
\le
L_xW_\rho(\mu,\widetilde\mu)
+
L_yW_\rho(\nu,\widetilde\nu).
\]
This completes the proof.
\end{proof}

\paragraph{Interpretation.}
This theorem shows that Wasserstein distances between projected feature
distributions are stable under small perturbations in the shared reference
space. In our setting, one can instantiate
$\mu=\widehat{\mu}_{i,T}^{(A\to\mathcal R)}$ and
$\nu=\widehat{\nu}_{j,T}^{(B\to\mathcal R)}$. If the projected empirical
distribution of the target feature changes by
$W_\rho(\mu,\widetilde\mu)$ and the projected empirical distribution of the
source feature changes by $W_\rho(\nu,\widetilde\nu)$, then the Wasserstein
feature score changes by at most the corresponding Lipschitz-weighted sum. Thus,
the method is stable as long as the activation-induced projected distributions
remain close in the shared reference geometry.

\subsection{Metric-Specific Stability Corollaries}
\label{app:metric_specific_stability}

We now instantiate Theorem~\ref{thm:wasserstein_stability} for several commonly
used ground costs on the shared reference space $\mathcal R$. Throughout this
section, let $\mu,\widetilde\mu,\nu,\widetilde\nu\in\mathcal P(\mathcal R)$ be
projected feature distributions, and define
\[
\delta_\rho
:=
W_\rho(\mu,\widetilde\mu)
+
W_\rho(\nu,\widetilde\nu).
\]
Here, $\mu$ and $\nu$ denote the original target-side and source-side projected
feature distributions, while $\widetilde\mu$ and $\widetilde\nu$ denote their
perturbed versions. We use $z$ for a target-side reference-space point and $z'$
for a source-side reference-space point. Later, $\bar z$ and $\bar z'$ denote
perturbed reference-space points used to verify coordinate-wise Lipschitz
conditions.

Let
\[
\mathcal S_\mu
:=
\operatorname{supp}(\mu)\cup\operatorname{supp}(\widetilde\mu),
\qquad
\mathcal S_\nu
:=
\operatorname{supp}(\nu)\cup\operatorname{supp}(\widetilde\nu).
\]
Thus, $z\in\mathcal S_\mu$ ranges over target-side support points and
$z'\in\mathcal S_\nu$ ranges over source-side support points. If a ground cost
$c:\mathcal R\times\mathcal R\to\mathbb R_{\ge0}$ satisfies
\[
\left|
W_c(\mu,\nu)-W_c(\widetilde\mu,\widetilde\nu)
\right|
\le
K_c\delta_\rho,
\]
then $K_c$ is its raw worst-case stability constant. Raw constants are
scale-dependent: replacing $c$ by $\alpha c$ also replaces $K_c$ by
$\alpha K_c$. Therefore, when possible, we use an on-support range bound $B_c
\ge
\sup_{(z,z')\in\mathcal S_\mu\times\mathcal S_\nu}c(z,z'),$
and compare costs by the normalized stability constant
$\overline K_c:=K_c/B_c$. A smaller $\overline K_c$ means that, relative to the
numerical scale of the cost, the Wasserstein feature distance is less sensitive
to perturbations of the projected feature distributions.

\paragraph{Common support assumptions.}
Let $\phi:\mathcal R\to\mathbb R^p$ be a representation map on the shared
reference space. Assume that on the relevant support,
\[
\|\phi(z)\|_2\le R_\phi,\qquad
\forall z\in\mathcal S_\mu\cup\mathcal S_\nu,
\]
and, when cosine cost is used, assume additionally that
\[
\|\phi(z)\|_2\ge r_\phi>0,\qquad
\forall z\in\mathcal S_\mu\cup\mathcal S_\nu.
\]
Assume that $\phi$ is $L_\phi$-Lipschitz with respect to $\rho$:
\[
\|\phi(z)-\phi(\bar z)\|_2\le L_\phi\rho(z,\bar z),
\qquad
\forall z,\bar z\in\mathcal S_\mu\cup\mathcal S_\nu.
\]
For flow-direction costs, define
\[
q_m(z):=\overline u_\theta(\tau_m,z)
=
\frac{u_\theta(\tau_m,z)}
{\|u_\theta(\tau_m,z)\|_2+\varepsilon_{\mathrm{norm}}},
\qquad
m=1,\dots,M.
\]
Assume that each $q_m:\mathcal R\to\mathbb R^s$ is $L_u$-Lipschitz on the
relevant support:
\[
\|q_m(z)-q_m(\bar z)\|_2\le L_u\rho(z,\bar z),
\qquad
\forall m\in[M],
\]
and since $q_m$ is normalized, assume $\|q_m(z)\|_2\le 1$.

\begin{corollary}[Stability under Euclidean positional cost]
Consider $c_{\mathrm{Euc}}(z,z')=\|\phi(z)-\phi(z')\|_2$. Then
$\left|W_{c_{\mathrm{Euc}}}(\mu,\nu)-W_{c_{\mathrm{Euc}}}(\widetilde\mu,\widetilde\nu)\right|
\le L_\phi\delta_\rho$. Hence $K_{\mathrm{Euc}}=L_\phi$. Moreover,
$B_{\mathrm{Euc}}=2R_\phi$ and
$\overline K_{\mathrm{Euc}}=L_\phi/(2R_\phi)$.
\end{corollary}

\begin{proof}
We verify the coordinate-wise Lipschitz constants and the range bound. For the
first coordinate,
\[
\begin{aligned}
|c_{\mathrm{Euc}}(z,z')-c_{\mathrm{Euc}}(\bar z,z')|
&=
\left|
\|\phi(z)-\phi(z')\|_2
-
\|\phi(\bar z)-\phi(z')\|_2
\right|\\
&\le
\|(\phi(z)-\phi(z'))-(\phi(\bar z)-\phi(z'))\|_2\\
&=
\|\phi(z)-\phi(\bar z)\|_2\\
&\le
L_\phi\rho(z,\bar z).
\end{aligned}
\]
The same argument in the second coordinate gives
\[
\begin{aligned}
|c_{\mathrm{Euc}}(z,z')-c_{\mathrm{Euc}}(z,\bar z')|
&=
\left|
\|\phi(z)-\phi(z')\|_2
-
\|\phi(z)-\phi(\bar z')\|_2
\right|\\
&\le
\|(\phi(z)-\phi(z'))-(\phi(z)-\phi(\bar z'))\|_2\\
&=
\|\phi(z')-\phi(\bar z')\|_2
\le
L_\phi\rho(z',\bar z').
\end{aligned}
\]
Thus $c_{\mathrm{Euc}}$ is coordinate-wise Lipschitz with
$L_x=L_y=L_\phi$. Applying Theorem~\ref{thm:wasserstein_stability} gives
\[
\left|
W_{c_{\mathrm{Euc}}}(\mu,\nu)
-
W_{c_{\mathrm{Euc}}}(\widetilde\mu,\widetilde\nu)
\right|
\le
L_\phi W_\rho(\mu,\widetilde\mu)
+
L_\phi W_\rho(\nu,\widetilde\nu)
=
L_\phi\delta_\rho.
\]
For the range bound, for every $(z,z')\in\mathcal S_\mu\times\mathcal S_\nu$,
\[
\begin{aligned}
c_{\mathrm{Euc}}(z,z')
=
\|\phi(z)-\phi(z')\|_2
&\le
\|\phi(z)\|_2+\|\phi(z')\|_2\\
&\le
R_\phi+R_\phi
=
2R_\phi.
\end{aligned}
\]
Hence $B_{\mathrm{Euc}}=2R_\phi$ and
$\overline K_{\mathrm{Euc}}=K_{\mathrm{Euc}}/B_{\mathrm{Euc}}
=L_\phi/(2R_\phi)$.
\end{proof}

\paragraph{Interpretation.}
Euclidean cost is stable when the representation map $\phi$ is smooth. Its
normalized stability depends on $L_\phi/(2R_\phi)$.

\begin{corollary}[Stability under squared Euclidean positional cost]
Consider $c_{\mathrm{sq}}(z,z')=\|\phi(z)-\phi(z')\|_2^2$. Then
$\left|W_{c_{\mathrm{sq}}}(\mu,\nu)-W_{c_{\mathrm{sq}}}(\widetilde\mu,\widetilde\nu)\right|
\le 4R_\phi L_\phi\delta_\rho$. Hence $K_{\mathrm{sq}}=4R_\phi L_\phi$.
Moreover, $B_{\mathrm{sq}}=4R_\phi^2$ and
$\overline K_{\mathrm{sq}}=L_\phi/R_\phi$.
\end{corollary}

\begin{proof}
For the first coordinate, use
$\|a\|_2^2-\|b\|_2^2=\langle a+b,a-b\rangle$ with
$a=\phi(z)-\phi(z')$ and $b=\phi(\bar z)-\phi(z')$. Then
\[
\begin{aligned}
&|c_{\mathrm{sq}}(z,z')-c_{\mathrm{sq}}(\bar z,z')|\\
&=
\left|
\|\phi(z)-\phi(z')\|_2^2
-
\|\phi(\bar z)-\phi(z')\|_2^2
\right|\\
&=
\left|
\left\langle
\phi(z)-\phi(z')+\phi(\bar z)-\phi(z'),
\phi(z)-\phi(\bar z)
\right\rangle
\right|\\
&\le
\left(
\|\phi(z)-\phi(z')\|_2
+
\|\phi(\bar z)-\phi(z')\|_2
\right)
\|\phi(z)-\phi(\bar z)\|_2\\
&\le
\left(
\|\phi(z)\|_2+\|\phi(z')\|_2
+
\|\phi(\bar z)\|_2+\|\phi(z')\|_2
\right)
L_\phi\rho(z,\bar z)\\
&\le
(R_\phi+R_\phi+R_\phi+R_\phi)L_\phi\rho(z,\bar z)\\
&=
4R_\phi L_\phi\rho(z,\bar z).
\end{aligned}
\]
The same computation in the second coordinate gives
\[
\begin{aligned}
|c_{\mathrm{sq}}(z,z')-c_{\mathrm{sq}}(z,\bar z')|
&\le
\left(
\|\phi(z)-\phi(z')\|_2
+
\|\phi(z)-\phi(\bar z')\|_2
\right)
\|\phi(z')-\phi(\bar z')\|_2\\
&\le
4R_\phi L_\phi\rho(z',\bar z').
\end{aligned}
\]
Thus $c_{\mathrm{sq}}$ is coordinate-wise Lipschitz with
$L_x=L_y=4R_\phi L_\phi$, so Theorem~\ref{thm:wasserstein_stability} yields
\[
\left|
W_{c_{\mathrm{sq}}}(\mu,\nu)
-
W_{c_{\mathrm{sq}}}(\widetilde\mu,\widetilde\nu)
\right|
\le
4R_\phi L_\phi\delta_\rho.
\]
For the range bound,
\[
\begin{aligned}
c_{\mathrm{sq}}(z,z')
=
\|\phi(z)-\phi(z')\|_2^2
&\le
(\|\phi(z)\|_2+\|\phi(z')\|_2)^2\\
&\le
(2R_\phi)^2
=
4R_\phi^2.
\end{aligned}
\]
Hence $B_{\mathrm{sq}}=4R_\phi^2$ and
$\overline K_{\mathrm{sq}}=K_{\mathrm{sq}}/B_{\mathrm{sq}}
=(4R_\phi L_\phi)/(4R_\phi^2)=L_\phi/R_\phi$.
\end{proof}

\paragraph{Interpretation.}
Squared Euclidean cost has a larger raw stability constant than Euclidean cost.
After normalization, $\overline K_{\mathrm{sq}}=2\overline K_{\mathrm{Euc}}$.

\begin{corollary}[Stability under Mahalanobis cost]
Let $M_{\mathrm{Mah}}\in\mathbb R^{m\times p}$ and consider
$c_{\mathrm{Mah}}(z,z')=\|M_{\mathrm{Mah}}(\phi(z)-\phi(z'))\|_2$. Then
$\left|W_{c_{\mathrm{Mah}}}(\mu,\nu)-W_{c_{\mathrm{Mah}}}(\widetilde\mu,\widetilde\nu)\right|
\le \|M_{\mathrm{Mah}}\|_{\mathrm{op}}L_\phi\delta_\rho$. Hence
$K_{\mathrm{Mah}}=\|M_{\mathrm{Mah}}\|_{\mathrm{op}}L_\phi$. Moreover,
$B_{\mathrm{Mah}}=2R_\phi\|M_{\mathrm{Mah}}\|_{\mathrm{op}}$ and
$\overline K_{\mathrm{Mah}}=L_\phi/(2R_\phi)$.
\end{corollary}

\begin{proof}
For the first coordinate,
\[
\begin{aligned}
|c_{\mathrm{Mah}}(z,z')-c_{\mathrm{Mah}}(\bar z,z')|
&=
\left|
\|M_{\mathrm{Mah}}(\phi(z)-\phi(z'))\|_2
-
\|M_{\mathrm{Mah}}(\phi(\bar z)-\phi(z'))\|_2
\right|\\
&\le
\|M_{\mathrm{Mah}}(\phi(z)-\phi(z'))-M_{\mathrm{Mah}}(\phi(\bar z)-\phi(z'))\|_2\\
&=
\|M_{\mathrm{Mah}}(\phi(z)-\phi(\bar z))\|_2\\
&\le
\|M_{\mathrm{Mah}}\|_{\mathrm{op}}\|\phi(z)-\phi(\bar z)\|_2\\
&\le
\|M_{\mathrm{Mah}}\|_{\mathrm{op}}L_\phi\rho(z,\bar z).
\end{aligned}
\]
Similarly,
\[
\begin{aligned}
|c_{\mathrm{Mah}}(z,z')-c_{\mathrm{Mah}}(z,\bar z')|
&\le
\|M_{\mathrm{Mah}}(\phi(z')-\phi(\bar z'))\|_2\\
&\le
\|M_{\mathrm{Mah}}\|_{\mathrm{op}}L_\phi\rho(z',\bar z').
\end{aligned}
\]
Therefore $L_x=L_y=\|M_{\mathrm{Mah}}\|_{\mathrm{op}}L_\phi$, and
Theorem~\ref{thm:wasserstein_stability} gives
\[
\left|
W_{c_{\mathrm{Mah}}}(\mu,\nu)
-
W_{c_{\mathrm{Mah}}}(\widetilde\mu,\widetilde\nu)
\right|
\le
\|M_{\mathrm{Mah}}\|_{\mathrm{op}}L_\phi\delta_\rho.
\]
For the range bound,
\[
\begin{aligned}
c_{\mathrm{Mah}}(z,z')
&=
\|M_{\mathrm{Mah}}(\phi(z)-\phi(z'))\|_2\\
&\le
\|M_{\mathrm{Mah}}\|_{\mathrm{op}}\|\phi(z)-\phi(z')\|_2\\
&\le
\|M_{\mathrm{Mah}}\|_{\mathrm{op}}(\|\phi(z)\|_2+\|\phi(z')\|_2)\\
&\le
2R_\phi\|M_{\mathrm{Mah}}\|_{\mathrm{op}}.
\end{aligned}
\]
Thus $B_{\mathrm{Mah}}=2R_\phi\|M_{\mathrm{Mah}}\|_{\mathrm{op}}$ and
$\overline K_{\mathrm{Mah}}
=(\|M_{\mathrm{Mah}}\|_{\mathrm{op}}L_\phi)/(2R_\phi\|M_{\mathrm{Mah}}\|_{\mathrm{op}})
=L_\phi/(2R_\phi)$.
\end{proof}

\paragraph{Interpretation.}
Raw Mahalanobis stability is amplified by $\|M_{\mathrm{Mah}}\|_{\mathrm{op}}$,
but this trivial scaling cancels after range normalization.

\begin{corollary}[Stability under squared Mahalanobis cost]
Let $M_{\mathrm{Mah}}\in\mathbb R^{m\times p}$ and consider
$c_{\mathrm{sqMah}}(z,z')=\|M_{\mathrm{Mah}}(\phi(z)-\phi(z'))\|_2^2$. Then
$\left|W_{c_{\mathrm{sqMah}}}(\mu,\nu)-W_{c_{\mathrm{sqMah}}}(\widetilde\mu,\widetilde\nu)\right|
\le 4R_\phi\|M_{\mathrm{Mah}}\|_{\mathrm{op}}^2L_\phi\delta_\rho$. Hence
$K_{\mathrm{sqMah}}=4R_\phi\|M_{\mathrm{Mah}}\|_{\mathrm{op}}^2L_\phi$. Moreover,
$B_{\mathrm{sqMah}}=4R_\phi^2\|M_{\mathrm{Mah}}\|_{\mathrm{op}}^2$ and
$\overline K_{\mathrm{sqMah}}=L_\phi/R_\phi$.
\end{corollary}

\begin{proof}
For the first coordinate, again using
$\|a\|_2^2-\|b\|_2^2=\langle a+b,a-b\rangle$ with
$a=M_{\mathrm{Mah}}(\phi(z)-\phi(z'))$ and
$b=M_{\mathrm{Mah}}(\phi(\bar z)-\phi(z'))$, we get
\[
\begin{aligned}
&|c_{\mathrm{sqMah}}(z,z')-c_{\mathrm{sqMah}}(\bar z,z')|\\
&=
\left|
\|M_{\mathrm{Mah}}(\phi(z)-\phi(z'))\|_2^2
-
\|M_{\mathrm{Mah}}(\phi(\bar z)-\phi(z'))\|_2^2
\right|\\
&\le
\left(
\|M_{\mathrm{Mah}}(\phi(z)-\phi(z'))\|_2
+
\|M_{\mathrm{Mah}}(\phi(\bar z)-\phi(z'))\|_2
\right)
\|M_{\mathrm{Mah}}(\phi(z)-\phi(\bar z))\|_2\\
&\le
\left(
\|M_{\mathrm{Mah}}\|_{\mathrm{op}}\|\phi(z)-\phi(z')\|_2
+
\|M_{\mathrm{Mah}}\|_{\mathrm{op}}\|\phi(\bar z)-\phi(z')\|_2
\right)
\|M_{\mathrm{Mah}}\|_{\mathrm{op}}\|\phi(z)-\phi(\bar z)\|_2\\
&\le
\left(
2R_\phi\|M_{\mathrm{Mah}}\|_{\mathrm{op}}
+
2R_\phi\|M_{\mathrm{Mah}}\|_{\mathrm{op}}
\right)
\|M_{\mathrm{Mah}}\|_{\mathrm{op}}L_\phi\rho(z,\bar z)\\
&=
4R_\phi\|M_{\mathrm{Mah}}\|_{\mathrm{op}}^2L_\phi\rho(z,\bar z).
\end{aligned}
\]
The same argument in the second coordinate gives
\[
|c_{\mathrm{sqMah}}(z,z')-c_{\mathrm{sqMah}}(z,\bar z')|
\le
4R_\phi\|M_{\mathrm{Mah}}\|_{\mathrm{op}}^2L_\phi\rho(z',\bar z').
\]
Therefore $c_{\mathrm{sqMah}}$ is coordinate-wise Lipschitz with
$L_x=L_y=4R_\phi\|M_{\mathrm{Mah}}\|_{\mathrm{op}}^2L_\phi$, giving the stated
stability bound. For the range,
\[
\begin{aligned}
c_{\mathrm{sqMah}}(z,z')
&=
\|M_{\mathrm{Mah}}(\phi(z)-\phi(z'))\|_2^2\\
&\le
\|M_{\mathrm{Mah}}\|_{\mathrm{op}}^2\|\phi(z)-\phi(z')\|_2^2\\
&\le
\|M_{\mathrm{Mah}}\|_{\mathrm{op}}^2(2R_\phi)^2
=
4R_\phi^2\|M_{\mathrm{Mah}}\|_{\mathrm{op}}^2.
\end{aligned}
\]
Thus
$\overline K_{\mathrm{sqMah}}
=
(4R_\phi\|M_{\mathrm{Mah}}\|_{\mathrm{op}}^2L_\phi)/
(4R_\phi^2\|M_{\mathrm{Mah}}\|_{\mathrm{op}}^2)
=
L_\phi/R_\phi$.
\end{proof}

\paragraph{Interpretation.}
Squared Mahalanobis behaves like squared Euclidean after normalization.

\begin{corollary}[Stability under cosine cost]
Define $\psi(z)=\phi(z)/\|\phi(z)\|_2$ and consider
$c_{\mathrm{cos}}(z,z')=1-\langle\psi(z),\psi(z')\rangle$. Then
$\left|W_{c_{\mathrm{cos}}}(\mu,\nu)-W_{c_{\mathrm{cos}}}(\widetilde\mu,\widetilde\nu)\right|
\le (2L_\phi/r_\phi)\delta_\rho$. Hence
$K_{\mathrm{cos}}=2L_\phi/r_\phi$. Since
$c_{\mathrm{cos}}(z,z')\in[0,2]$, we may take
$B_{\mathrm{cos}}=2$ and $\overline K_{\mathrm{cos}}=L_\phi/r_\phi$.
\end{corollary}

\begin{proof}
First, for any $u,v$ with $\|u\|_2,\|v\|_2\ge r_\phi$,
\[
\begin{aligned}
\left\|
\frac{u}{\|u\|_2}
-
\frac{v}{\|v\|_2}
\right\|_2
&=
\left\|
\frac{u-v}{\|u\|_2}
+
v\left(\frac{1}{\|u\|_2}-\frac{1}{\|v\|_2}\right)
\right\|_2\\
&\le
\frac{\|u-v\|_2}{\|u\|_2}
+
\|v\|_2
\frac{\left|\|v\|_2-\|u\|_2\right|}
{\|u\|_2\|v\|_2}\\
&\le
\frac{\|u-v\|_2}{r_\phi}
+
\frac{\left|\|v\|_2-\|u\|_2\right|}{r_\phi}\\
&\le
\frac{\|u-v\|_2}{r_\phi}
+
\frac{\|u-v\|_2}{r_\phi}
=
\frac{2\|u-v\|_2}{r_\phi}.
\end{aligned}
\]
Applying this with $u=\phi(z)$ and $v=\phi(\bar z)$ gives
\[
\|\psi(z)-\psi(\bar z)\|_2
\le
\frac{2\|\phi(z)-\phi(\bar z)\|_2}{r_\phi}
\le
\frac{2L_\phi}{r_\phi}\rho(z,\bar z).
\]
Hence, for the first coordinate,
\[
\begin{aligned}
|c_{\mathrm{cos}}(z,z')-c_{\mathrm{cos}}(\bar z,z')|
&=
\left|
1-\langle\psi(z),\psi(z')\rangle
-
\left(1-\langle\psi(\bar z),\psi(z')\rangle\right)
\right|\\
&=
\left|
\langle\psi(\bar z)-\psi(z),\psi(z')\rangle
\right|\\
&\le
\|\psi(\bar z)-\psi(z)\|_2\|\psi(z')\|_2\\
&=
\|\psi(\bar z)-\psi(z)\|_2\\
&\le
\frac{2L_\phi}{r_\phi}\rho(z,\bar z).
\end{aligned}
\]
The second coordinate is identical. Thus
$L_x=L_y=2L_\phi/r_\phi$, and Theorem~\ref{thm:wasserstein_stability} gives
the stability bound. Finally, since $\psi(z)$ and $\psi(z')$ are unit vectors,
\[
-1\le \langle\psi(z),\psi(z')\rangle\le 1
\quad\Longrightarrow\quad
0\le c_{\mathrm{cos}}(z,z')\le 2.
\]
Thus $B_{\mathrm{cos}}=2$ and
$\overline K_{\mathrm{cos}}=(2L_\phi/r_\phi)/2=L_\phi/r_\phi$.
\end{proof}

\paragraph{Interpretation.}
Cosine cost becomes sensitive when representation norms approach zero; this is
captured by the factor \(1/r_\phi\).

\begin{corollary}[Stability under pure flow-direction cost]
Consider
$c_{\mathrm{dir}}(z,z')=M^{-1}\sum_{m=1}^{M}\|q_m(z)-q_m(z')\|_2$. Then
$\left|W_{c_{\mathrm{dir}}}(\mu,\nu)-W_{c_{\mathrm{dir}}}(\widetilde\mu,\widetilde\nu)\right|
\le L_u\delta_\rho$. Hence $K_{\mathrm{dir}}=L_u$. Since
$c_{\mathrm{dir}}(z,z')\le2$, we may take $B_{\mathrm{dir}}=2$ and
$\overline K_{\mathrm{dir}}=L_u/2$.
\end{corollary}

\begin{proof}
For the first coordinate,
\[
\begin{aligned}
&|c_{\mathrm{dir}}(z,z')-c_{\mathrm{dir}}(\bar z,z')|\\
&=
\left|
\frac1M\sum_{m=1}^{M}\|q_m(z)-q_m(z')\|_2
-
\frac1M\sum_{m=1}^{M}\|q_m(\bar z)-q_m(z')\|_2
\right|\\
&\le
\frac1M\sum_{m=1}^{M}
\left|
\|q_m(z)-q_m(z')\|_2
-
\|q_m(\bar z)-q_m(z')\|_2
\right|\\
&\le
\frac1M\sum_{m=1}^{M}
\|(q_m(z)-q_m(z'))-(q_m(\bar z)-q_m(z'))\|_2\\
&=
\frac1M\sum_{m=1}^{M}
\|q_m(z)-q_m(\bar z)\|_2\\
&\le
\frac1M\sum_{m=1}^{M}
L_u\rho(z,\bar z)
=
L_u\rho(z,\bar z).
\end{aligned}
\]
The same argument gives
\[
|c_{\mathrm{dir}}(z,z')-c_{\mathrm{dir}}(z,\bar z')|
\le
L_u\rho(z',\bar z').
\]
Thus $L_x=L_y=L_u$, and Theorem~\ref{thm:wasserstein_stability} gives
\[
\left|
W_{c_{\mathrm{dir}}}(\mu,\nu)
-
W_{c_{\mathrm{dir}}}(\widetilde\mu,\widetilde\nu)
\right|
\le
L_u\delta_\rho.
\]
For the range,
\[
\begin{aligned}
c_{\mathrm{dir}}(z,z')
&=
\frac1M\sum_{m=1}^{M}\|q_m(z)-q_m(z')\|_2\\
&\le
\frac1M\sum_{m=1}^{M}
(\|q_m(z)\|_2+\|q_m(z')\|_2)\\
&\le
\frac1M\sum_{m=1}^{M}2
=
2.
\end{aligned}
\]
Hence $B_{\mathrm{dir}}=2$ and
$\overline K_{\mathrm{dir}}=L_u/2$.
\end{proof}

\paragraph{Interpretation.}
Pure flow-direction stability is controlled only by the smoothness \(L_u\) of
the normalized velocity field.

\begin{corollary}[Stability under squared pure flow-direction cost]
Consider
$c_{\mathrm{dir}}^{(2)}(z,z')=M^{-1}\sum_{m=1}^{M}\|q_m(z)-q_m(z')\|_2^2$. Then
$\left|W_{c_{\mathrm{dir}}^{(2)}}(\mu,\nu)-W_{c_{\mathrm{dir}}^{(2)}}(\widetilde\mu,\widetilde\nu)\right|
\le4L_u\delta_\rho$. Hence $K_{\mathrm{dir}}^{(2)}=4L_u$. Since
$c_{\mathrm{dir}}^{(2)}(z,z')\le4$, we may take
$B_{\mathrm{dir}}^{(2)}=4$ and $\overline K_{\mathrm{dir}}^{(2)}=L_u$.
\end{corollary}

\begin{proof}
For each $m$, using again
$\|a\|_2^2-\|b\|_2^2=\langle a+b,a-b\rangle$ with
$a=q_m(z)-q_m(z')$ and $b=q_m(\bar z)-q_m(z')$, we obtain
\[
\begin{aligned}
&\left|
\|q_m(z)-q_m(z')\|_2^2
-
\|q_m(\bar z)-q_m(z')\|_2^2
\right|\\
&=
\left|
\left\langle
q_m(z)-q_m(z')+q_m(\bar z)-q_m(z'),
q_m(z)-q_m(\bar z)
\right\rangle
\right|\\
&\le
\left(
\|q_m(z)-q_m(z')\|_2
+
\|q_m(\bar z)-q_m(z')\|_2
\right)
\|q_m(z)-q_m(\bar z)\|_2\\
&\le
\left(
\|q_m(z)\|_2+\|q_m(z')\|_2
+
\|q_m(\bar z)\|_2+\|q_m(z')\|_2
\right)
L_u\rho(z,\bar z)\\
&\le
(1+1+1+1)L_u\rho(z,\bar z)
=
4L_u\rho(z,\bar z).
\end{aligned}
\]
Averaging over \(m\) gives the first-coordinate Lipschitz bound:
\[
\begin{aligned}
|c_{\mathrm{dir}}^{(2)}(z,z')-c_{\mathrm{dir}}^{(2)}(\bar z,z')|
&\le
\frac1M\sum_{m=1}^{M}4L_u\rho(z,\bar z)
=
4L_u\rho(z,\bar z).
\end{aligned}
\]
The second coordinate is identical. Thus $L_x=L_y=4L_u$, and the stability
bound follows from Theorem~\ref{thm:wasserstein_stability}. For the range,
\[
\begin{aligned}
c_{\mathrm{dir}}^{(2)}(z,z')
&=
\frac1M\sum_{m=1}^{M}\|q_m(z)-q_m(z')\|_2^2\\
&\le
\frac1M\sum_{m=1}^{M}(\|q_m(z)\|_2+\|q_m(z')\|_2)^2\\
&\le
\frac1M\sum_{m=1}^{M}4
=
4.
\end{aligned}
\]
Hence $B_{\mathrm{dir}}^{(2)}=4$ and
$\overline K_{\mathrm{dir}}^{(2)}=4L_u/4=L_u$.
\end{proof}

\paragraph{Interpretation.}
Squaring the directional discrepancy increases the normalized stability constant
from \(L_u/2\) to \(L_u\).

\begin{corollary}[Stability under composite unsquared flow cost]
Consider
$c_{\mathrm{comp}}(z,z')=\lambda_xc_{\mathrm{Euc}}(z,z')+\lambda_uc_{\mathrm{dir}}(z,z')$
with $\lambda_x,\lambda_u\ge0$. Then
$\left|W_{c_{\mathrm{comp}}}(\mu,\nu)-W_{c_{\mathrm{comp}}}(\widetilde\mu,\widetilde\nu)\right|
\le(\lambda_xL_\phi+\lambda_uL_u)\delta_\rho$. Hence
$K_{\mathrm{comp}}=\lambda_xL_\phi+\lambda_uL_u$. Moreover,
$B_{\mathrm{comp}}=2(\lambda_xR_\phi+\lambda_u)$ and
$\overline K_{\mathrm{comp}}
=(\lambda_xL_\phi+\lambda_uL_u)/(2(\lambda_xR_\phi+\lambda_u))$.
\end{corollary}

\begin{proof}
For the first coordinate, using nonnegativity of \(\lambda_x,\lambda_u\) and the
coordinate-wise bounds already proved,
\[
\begin{aligned}
&|c_{\mathrm{comp}}(z,z')-c_{\mathrm{comp}}(\bar z,z')|\\
&=
\left|
\lambda_x[c_{\mathrm{Euc}}(z,z')-c_{\mathrm{Euc}}(\bar z,z')]
+
\lambda_u[c_{\mathrm{dir}}(z,z')-c_{\mathrm{dir}}(\bar z,z')]
\right|\\
&\le
\lambda_x|c_{\mathrm{Euc}}(z,z')-c_{\mathrm{Euc}}(\bar z,z')|
+
\lambda_u|c_{\mathrm{dir}}(z,z')-c_{\mathrm{dir}}(\bar z,z')|\\
&\le
\lambda_xL_\phi\rho(z,\bar z)
+
\lambda_uL_u\rho(z,\bar z)\\
&=
(\lambda_xL_\phi+\lambda_uL_u)\rho(z,\bar z).
\end{aligned}
\]
The same argument in the second coordinate gives the same constant. Hence
\[
L_x=L_y=\lambda_xL_\phi+\lambda_uL_u,
\]
and Theorem~\ref{thm:wasserstein_stability} gives the claimed stability bound.
For the range,
\[
\begin{aligned}
c_{\mathrm{comp}}(z,z')
&=
\lambda_xc_{\mathrm{Euc}}(z,z')+\lambda_uc_{\mathrm{dir}}(z,z')\\
&\le
\lambda_x(2R_\phi)+\lambda_u(2)
=
2(\lambda_xR_\phi+\lambda_u).
\end{aligned}
\]
Thus
\[
\overline K_{\mathrm{comp}}
=
\frac{\lambda_xL_\phi+\lambda_uL_u}
{2(\lambda_xR_\phi+\lambda_u)}.
\]
\end{proof}

\paragraph{Interpretation.}
The composite unsquared cost interpolates between positional stability and
directional stability.

\begin{corollary}[Stability under composite squared flow cost]
Consider
$c_{\mathrm{flow}}(z,z')=\lambda_xc_{\mathrm{sq}}(z,z')+\lambda_uc_{\mathrm{dir}}^{(2)}(z,z')$
with $\lambda_x,\lambda_u\ge0$. Then
$\left|W_{c_{\mathrm{flow}}}(\mu,\nu)-W_{c_{\mathrm{flow}}}(\widetilde\mu,\widetilde\nu)\right|
\le(4\lambda_xR_\phi L_\phi+4\lambda_uL_u)\delta_\rho$. Hence
$K_{\mathrm{flow}}=4\lambda_xR_\phi L_\phi+4\lambda_uL_u$. Moreover,
$B_{\mathrm{flow}}=4(\lambda_xR_\phi^2+\lambda_u)$ and
$\overline K_{\mathrm{flow}}
=(\lambda_xR_\phi L_\phi+\lambda_uL_u)/(\lambda_xR_\phi^2+\lambda_u)$.
\end{corollary}

\begin{proof}
For the first coordinate,
\[
\begin{aligned}
&|c_{\mathrm{flow}}(z,z')-c_{\mathrm{flow}}(\bar z,z')|\\
&=
\left|
\lambda_x[c_{\mathrm{sq}}(z,z')-c_{\mathrm{sq}}(\bar z,z')]
+
\lambda_u[c_{\mathrm{dir}}^{(2)}(z,z')-c_{\mathrm{dir}}^{(2)}(\bar z,z')]
\right|\\
&\le
\lambda_x|c_{\mathrm{sq}}(z,z')-c_{\mathrm{sq}}(\bar z,z')|
+
\lambda_u|c_{\mathrm{dir}}^{(2)}(z,z')-c_{\mathrm{dir}}^{(2)}(\bar z,z')|\\
&\le
\lambda_x(4R_\phi L_\phi)\rho(z,\bar z)
+
\lambda_u(4L_u)\rho(z,\bar z)\\
&=
(4\lambda_xR_\phi L_\phi+4\lambda_uL_u)\rho(z,\bar z).
\end{aligned}
\]
The second coordinate is identical. Hence
$L_x=L_y=4\lambda_xR_\phi L_\phi+4\lambda_uL_u$, and the stability bound follows
from Theorem~\ref{thm:wasserstein_stability}. For the range,
\[
\begin{aligned}
c_{\mathrm{flow}}(z,z')
&=
\lambda_xc_{\mathrm{sq}}(z,z')+\lambda_uc_{\mathrm{dir}}^{(2)}(z,z')\\
&\le
\lambda_x(4R_\phi^2)+\lambda_u(4)
=
4(\lambda_xR_\phi^2+\lambda_u).
\end{aligned}
\]
Therefore,
\[
\overline K_{\mathrm{flow}}
=
\frac{4\lambda_xR_\phi L_\phi+4\lambda_uL_u}
{4(\lambda_xR_\phi^2+\lambda_u)}
=
\frac{\lambda_xR_\phi L_\phi+\lambda_uL_u}
{\lambda_xR_\phi^2+\lambda_u}.
\]
\end{proof}

\paragraph{Interpretation.}
The squared flow cost matches squared-distance geometries, but it has larger raw
stability constants than the corresponding unsquared composite cost.

\begin{table}[t]
\centering
\caption{Raw and normalized worst-case stability constants for common ground
costs on the shared reference space $\mathcal R$. The raw constant $K_c$ is
scale-dependent. The normalized constant $\overline K_c=K_c/B_c$ removes the
trivial cost-scaling effect using an on-support range bound $B_c$.}
\label{tab:normalized_stability_constants}
\begin{tabular}{llll}
\toprule
Cost & Raw constant $K_c$ & Range bound $B_c$ & Normalized constant $\overline K_c$ \\
\midrule
Euclidean &
$L_\phi$ &
$2R_\phi$ &
$\frac{L_\phi}{2R_\phi}$ \\

Squared Euclidean &
$4R_\phi L_\phi$ &
$4R_\phi^2$ &
$\frac{L_\phi}{R_\phi}$ \\

Mahalanobis &
$\|A\|_{\mathrm{op}}L_\phi$ &
$2R_\phi\|A\|_{\mathrm{op}}$ &
$\frac{L_\phi}{2R_\phi}$ \\

Squared Mahalanobis &
$4R_\phi\|A\|_{\mathrm{op}}^2L_\phi$ &
$4R_\phi^2\|A\|_{\mathrm{op}}^2$ &
$\frac{L_\phi}{R_\phi}$ \\

Cosine &
$\frac{2L_\phi}{r_\phi}$ &
$2$ &
$\frac{L_\phi}{r_\phi}$ \\

Pure flow-direction &
$L_u$ &
$2$ &
$\frac{L_u}{2}$ \\

Squared flow-direction &
$4L_u$ &
$4$ &
$L_u$ \\

Composite unsquared flow &
$\lambda_xL_\phi+\lambda_uL_u$ &
$2(\lambda_xR_\phi+\lambda_u)$ &
$\frac{\lambda_xL_\phi+\lambda_uL_u}
{2(\lambda_xR_\phi+\lambda_u)}$ \\

Composite squared flow &
$4\lambda_xR_\phi L_\phi+4\lambda_uL_u$ &
$4(\lambda_xR_\phi^2+\lambda_u)$ &
$\frac{\lambda_xR_\phi L_\phi+\lambda_uL_u}
{\lambda_xR_\phi^2+\lambda_u}$ \\
\bottomrule
\end{tabular}
\end{table}

\paragraph{Scale-invariant comparison for feature matching.}
Even normalized constants do not fully determine which cost is better for
feature matching, because the matching guarantee depends on both the stability
constant and the Wasserstein matching margin. For a ground cost $c$, define the
margin-normalized robustness ratio
$\chi_i(c):=K_c\delta_\rho/\Delta_i^c$, where $\Delta_i^c$ is the Wasserstein
matching margin under $c$. The nearest-neighbor match is certified whenever
$\chi_i(c)<1/2$. This ratio is invariant to positive cost rescaling: if
$c'=\alpha c$ with $\alpha>0$, then
\[
K_{c'}=\alpha K_c,\qquad
\Delta_i^{c'}=\alpha\Delta_i^c,\qquad
\chi_i(c')=\frac{\alpha K_c\delta_\rho}{\alpha\Delta_i^c}=\chi_i(c).
\]
Therefore, for feature matching, the meaningful comparison is not the raw
constant $K_c$, but the margin-normalized ratio $\chi_i(c)$.

\paragraph{Euclidean versus pure flow-direction stability.}
Using normalized constants, pure flow-direction cost is more stable than
Euclidean positional cost only when
$\overline K_{\mathrm{dir}}<\overline K_{\mathrm{Euc}}$. Since
\[
\overline K_{\mathrm{dir}}=\frac{L_u}{2},\qquad
\overline K_{\mathrm{Euc}}=\frac{L_\phi}{2R_\phi},
\]
this condition is equivalent to $L_u<L_\phi/R_\phi$. Thus, pure flow-direction
cost is more stable only when the normalized velocity field varies more smoothly
than the representation geometry after accounting for the representation scale
$R_\phi$.

\subsection{Stability Under Ground-Cost Perturbation}
\label{app:ground_cost_perturbation}

\noindent
\textbf{Theorem~\ref{thm:ground_cost_perturbation}
(Stability under ground-cost perturbation).}
\label{thm:ground_cost_perturbation}
\emph{
Let $\mu,\nu\in\mathcal P(\mathcal R)$ be projected feature distributions, and
let $c,\widetilde c:
\mathcal R\times\mathcal R
\to
\mathbb R_{\ge 0}$ be two ground costs. Define the relevant transport support by $\mathcal S_{\mu,\nu}
:=
\operatorname{supp}(\mu)\times\operatorname{supp}(\nu).$ Suppose $c$ and $\widetilde c$ are uniformly close on this support:
\[
\varepsilon_c
:=
\sup_{(z,z')\in\mathcal S_{\mu,\nu}}
\left|
c(z,z')-\widetilde c(z,z')
\right|
<
\infty.
\]
Then $\left|
W_c(\mu,\nu)
-
W_{\widetilde c}(\mu,\nu)
\right|
\le
\varepsilon_c.$
}

\begin{proof}
For any coupling $\pi\in\Pi(\mu,\nu)$, the support of $\pi$ is contained in
$\operatorname{supp}(\mu)\times\operatorname{supp}(\nu)$. Hence,
\[
\left|
\mathbb E_{\pi}[c(z,z')]
-
\mathbb E_{\pi}[\widetilde c(z,z')]
\right|
\le
\mathbb E_{\pi}
\left[
|c(z,z')-\widetilde c(z,z')|
\right]
\le
\varepsilon_c.
\]
Equivalently, for every $\pi\in\Pi(\mu,\nu)$,
\[
\mathbb E_{\pi}[c(z,z')]
\le
\mathbb E_{\pi}[\widetilde c(z,z')]
+
\varepsilon_c,
\qquad
\mathbb E_{\pi}[\widetilde c(z,z')]
\le
\mathbb E_{\pi}[c(z,z')]
+
\varepsilon_c.
\]

We first show
\[
W_c(\mu,\nu)-W_{\widetilde c}(\mu,\nu)\le \varepsilon_c.
\]
Let $\eta>0$, and choose an $\eta$-optimal coupling
$\widetilde\pi_\eta\in\Pi(\mu,\nu)$ for $W_{\widetilde c}(\mu,\nu)$, so that
\[
\mathbb E_{\widetilde\pi_\eta}[\widetilde c(z,z')]
\le
W_{\widetilde c}(\mu,\nu)+\eta.
\]
Since $\widetilde\pi_\eta$ is feasible for $W_c(\mu,\nu)$,
\[
\begin{aligned}
W_c(\mu,\nu)
&\le
\mathbb E_{\widetilde\pi_\eta}[c(z,z')]\\
&\le
\mathbb E_{\widetilde\pi_\eta}[\widetilde c(z,z')]
+
\varepsilon_c\\
&\le
W_{\widetilde c}(\mu,\nu)+\eta+\varepsilon_c.
\end{aligned}
\]
Taking $\eta\downarrow 0$ gives
\[
W_c(\mu,\nu)-W_{\widetilde c}(\mu,\nu)\le \varepsilon_c.
\]

The reverse inequality follows identically. Let $\pi_\eta\in\Pi(\mu,\nu)$ be an
$\eta$-optimal coupling for $W_c(\mu,\nu)$. Then
\[
\begin{aligned}
W_{\widetilde c}(\mu,\nu)
&\le
\mathbb E_{\pi_\eta}[\widetilde c(z,z')]\\
&\le
\mathbb E_{\pi_\eta}[c(z,z')]
+
\varepsilon_c\\
&\le
W_c(\mu,\nu)+\eta+\varepsilon_c.
\end{aligned}
\]
Taking $\eta\downarrow 0$ yields
\[
W_{\widetilde c}(\mu,\nu)-W_c(\mu,\nu)\le \varepsilon_c.
\]
Combining the two one-sided bounds,
\[
\left|
W_c(\mu,\nu)
-
W_{\widetilde c}(\mu,\nu)
\right|
\le
\varepsilon_c.
\]
\end{proof}

\paragraph{Interpretation.}
This theorem isolates the effect of changing the ground cost while keeping the
projected feature distributions fixed. In our method, the distributions
$\widehat{\mu}_{i,T}^{(A\to\mathcal R)}$ and
$\widehat{\nu}_{j,T}^{(B\to\mathcal R)}$ live in the same shared reference space
$\mathcal R$, but the cost used to compare their support points may be
Euclidean, flow-aware, learned, or approximated for computational reasons. The
theorem says that if the approximate cost $\widetilde c$ is uniformly close to
the intended cost $c$ on the actually transported support, then the resulting
Wasserstein score changes by at most the same amount. Thus, learned or
approximate ground costs are safe to use when their error is controlled on the
empirical support of the projected distributions.

\subsection{Matching Recovery Under Wasserstein Margin}
\label{app:matching_recovery}

\paragraph{Definitions.}
Fix a target feature $i\in\mathcal F^{(A)}$ and a source feature set
$\mathcal F^{(B)}$. For each source feature $j\in\mathcal F^{(B)}$, define the
oracle Wasserstein score in the shared reference space $\mathcal R$ by $D_{\mathcal R}(i,j)
=
W_c\!\left(
\mu_i^{(A\to\mathcal R)},
\nu_j^{(B\to\mathcal R)}
\right),$ where $\mu_i^{(A\to\mathcal R)}$ and $\nu_j^{(B\to\mathcal R)}$ denote the
population projected feature distributions of the target and source features.
The empirical score computed from finite projected activation distributions is
\[
\widehat D_{\mathcal R}(i,j)
=
W_c\!\left(
\widehat\mu_{i,T}^{(A\to\mathcal R)},
\widehat\nu_{j,T}^{(B\to\mathcal R)}
\right).
\]
Let $j_i^\star
=
\arg\min_{j\in\mathcal F^{(B)}}D_{\mathcal R}(i,j)$ be the oracle nearest-neighbor match of target feature $i$. We assume that this
minimizer is unique. The Wasserstein matching margin is
\[
\Delta_i
=
\min_{j\neq j_i^\star}
\left[
D_{\mathcal R}(i,j)
-
D_{\mathcal R}(i,j_i^\star)
\right].
\]
Thus, $\Delta_i$ measures how much closer the oracle match $j_i^\star$ is than
the nearest competing source feature. Finally, define the uniform empirical
score perturbation by
\[
\epsilon_i
=
\sup_{j\in\mathcal F^{(B)}}
\left|
\widehat D_{\mathcal R}(i,j)
-
D_{\mathcal R}(i,j)
\right|.
\]
This quantity upper-bounds the largest score error over all candidate source
features for the fixed target feature $i$.

\noindent
\textbf{Theorem~\ref{thm:matching_recovery}
(Matching recovery under Wasserstein margin).}
\emph{
Fix $i\in\mathcal F^{(A)}$ and suppose $j_i^\star$ is unique with
$\Delta_i>0$. If
\[
\sup_{j\in\mathcal F^{(B)}}
\left|
\widehat D_{\mathcal R}(i,j)
-
D_{\mathcal R}(i,j)
\right|
\le
\epsilon_i
\qquad
\text{and}
\qquad
\Delta_i>2\epsilon_i,
\]
then
\[
\arg\min_{j\in\mathcal F^{(B)}}
\widehat D_{\mathcal R}(i,j)
=
j_i^\star.
\]
}

\begin{proof}
Because $j_i^\star$ is the unique oracle minimizer, every competing source
$j\neq j_i^\star$ satisfies
\[
D_{\mathcal R}(i,j)-D_{\mathcal R}(i,j_i^\star)\ge \Delta_i.
\]
The uniform perturbation condition gives, for every $j\in\mathcal F^{(B)}$,
\[
-\epsilon_i
\le
\widehat D_{\mathcal R}(i,j)-D_{\mathcal R}(i,j)
\le
\epsilon_i.
\]
Equivalently,
\[
\widehat D_{\mathcal R}(i,j)\ge D_{\mathcal R}(i,j)-\epsilon_i,
\qquad
\widehat D_{\mathcal R}(i,j)\le D_{\mathcal R}(i,j)+\epsilon_i.
\]
Now fix any competing source $j\neq j_i^\star$. We compare its empirical score
with the empirical score of the oracle match by adding and subtracting the
oracle scores:
\[
\begin{aligned}
&
\widehat D_{\mathcal R}(i,j)
-
\widehat D_{\mathcal R}(i,j_i^\star)\\
&=
\left[
\widehat D_{\mathcal R}(i,j)
-
D_{\mathcal R}(i,j)
\right] +
\left[
D_{\mathcal R}(i,j)
-
D_{\mathcal R}(i,j_i^\star)
\right]
+
\left[
D_{\mathcal R}(i,j_i^\star)
-
\widehat D_{\mathcal R}(i,j_i^\star)
\right]\\
&\ge
\left[-\epsilon_i\right]
+
\Delta_i
+
\left[-\epsilon_i\right]\\
&=
\Delta_i-2\epsilon_i.
\end{aligned}
\]
If $\Delta_i>2\epsilon_i$, then
\[
\widehat D_{\mathcal R}(i,j)
-
\widehat D_{\mathcal R}(i,j_i^\star)
>
0,
\qquad
\forall j\neq j_i^\star.
\]
Hence every competing source has strictly larger empirical score than
$j_i^\star$:
\[
\widehat D_{\mathcal R}(i,j)
>
\widehat D_{\mathcal R}(i,j_i^\star),
\qquad
\forall j\neq j_i^\star.
\]
Therefore $j_i^\star$ remains the unique empirical nearest neighbor:
\[
\arg\min_{j\in\mathcal F^{(B)}}
\widehat D_{\mathcal R}(i,j)
=
j_i^\star.
\]
\end{proof}

\paragraph{Interpretation.}
This theorem converts score stability into decision stability. The margin
$\Delta_i$ measures how much the oracle best match is separated from the nearest
competing source. The perturbation $\epsilon_i$ controls how much any empirical
score can deviate from its oracle value. The factor $2$ appears because, in the
worst case, the empirical score of the true match can increase by $\epsilon_i$
while the score of a competitor can decrease by $\epsilon_i$. Thus the oracle
gap can shrink by at most $2\epsilon_i$.

\paragraph{Combined distribution and cost perturbations.}
The theorem above is stated in terms of an abstract score perturbation
$\epsilon_i$. We now derive such a perturbation bound when both the projected
feature distributions and the ground cost are perturbed. Fix a target feature
$i\in\mathcal F^{(A)}$ and a source feature $j\in\mathcal F^{(B)}$. The oracle
score is
\[
D_{\mathcal R}(i,j)
=
W_c\!\left(
\mu_i^{(A\to\mathcal R)},
\nu_j^{(B\to\mathcal R)}
\right).
\]
Suppose the empirical computation uses projected empirical distributions and an
approximate ground cost $\widetilde c:\mathcal R\times\mathcal R\to\mathbb R_{\ge0}$.
Define the perturbed empirical score by
\[
\widehat D_{\mathcal R}^{\widetilde c}(i,j)
=
W_{\widetilde c}\!\left(
\widehat\mu_{i,T}^{(A\to\mathcal R)},
\widehat\nu_{j,T}^{(B\to\mathcal R)}
\right).
\]
The distribution perturbation error for pair $(i,j)$ is
\[
\delta_{i,j}
:=
L_x
W_\rho\!\left(
\mu_i^{(A\to\mathcal R)},
\widehat\mu_{i,T}^{(A\to\mathcal R)}
\right)
+
L_y
W_\rho\!\left(
\nu_j^{(B\to\mathcal R)},
\widehat\nu_{j,T}^{(B\to\mathcal R)}
\right),
\]
where $L_x,L_y$ are the coordinate-wise Lipschitz constants of the reference
ground cost $c$ on the relevant support. The cost perturbation error is
\[
\varepsilon_{c,i,j}
:=
\sup_{(z,z')\in
\operatorname{supp}(\widehat\mu_{i,T}^{(A\to\mathcal R)})
\times
\operatorname{supp}(\widehat\nu_{j,T}^{(B\to\mathcal R)})}
\left|
c(z,z')-\widetilde c(z,z')
\right|.
\]
This is the maximum pointwise discrepancy between the reference cost and the
approximate cost on the empirical supports actually used by the OT problem.
Define the worst-case combined perturbation for target feature $i$ by
\[
\bar\epsilon_i
:=
\max_{j\in\mathcal F^{(B)}}
\left[
\delta_{i,j}
+
\varepsilon_{c,i,j}
\right].
\]

\begin{corollary}[Stable matching under distribution and cost perturbations]
\label{cor:stable_matching_measure_cost}
Let $j_i^\star$ be the unique oracle nearest-neighbor match under
$D_{\mathcal R}(i,j)$, and let $\Delta_i$ be its Wasserstein matching margin. If
\[
\Delta_i>2\bar\epsilon_i,
\]
then nearest-neighbor matching with the perturbed empirical scores recovers the
oracle match:
\[
\arg\min_{j\in\mathcal F^{(B)}}
\widehat D_{\mathcal R}^{\widetilde c}(i,j)
=
j_i^\star.
\]
\end{corollary}

\begin{proof}
Fix any source feature $j\in\mathcal F^{(B)}$. We decompose the total score
error into a distribution perturbation term and a cost perturbation term by
adding and subtracting the intermediate score computed on empirical
distributions with the reference cost $c$:
\[
\begin{aligned}
&
\left|
D_{\mathcal R}(i,j)
-
\widehat D_{\mathcal R}^{\widetilde c}(i,j)
\right|\\
&=
\left|
W_c\!\left(
\mu_i^{(A\to\mathcal R)},
\nu_j^{(B\to\mathcal R)}
\right)
-
W_{\widetilde c}\!\left(
\widehat\mu_{i,T}^{(A\to\mathcal R)},
\widehat\nu_{j,T}^{(B\to\mathcal R)}
\right)
\right|\\
&\le
\left|
W_c\!\left(
\mu_i^{(A\to\mathcal R)},
\nu_j^{(B\to\mathcal R)}
\right)
-
W_c\!\left(
\widehat\mu_{i,T}^{(A\to\mathcal R)},
\widehat\nu_{j,T}^{(B\to\mathcal R)}
\right)
\right|\\
&\quad+
\left|
W_c\!\left(
\widehat\mu_{i,T}^{(A\to\mathcal R)},
\widehat\nu_{j,T}^{(B\to\mathcal R)}
\right)
-
W_{\widetilde c}\!\left(
\widehat\mu_{i,T}^{(A\to\mathcal R)},
\widehat\nu_{j,T}^{(B\to\mathcal R)}
\right)
\right|.
\end{aligned}
\]
By Theorem~\ref{thm:wasserstein_stability}, the first term is bounded by
\[
\begin{aligned}
&
\left|
W_c\!\left(
\mu_i^{(A\to\mathcal R)},
\nu_j^{(B\to\mathcal R)}
\right)
-
W_c\!\left(
\widehat\mu_{i,T}^{(A\to\mathcal R)},
\widehat\nu_{j,T}^{(B\to\mathcal R)}
\right)
\right|\\
&\le
L_x
W_\rho\!\left(
\mu_i^{(A\to\mathcal R)},
\widehat\mu_{i,T}^{(A\to\mathcal R)}
\right)
+
L_y
W_\rho\!\left(
\nu_j^{(B\to\mathcal R)},
\widehat\nu_{j,T}^{(B\to\mathcal R)}
\right)\\
&=
\delta_{i,j}.
\end{aligned}
\]

It remains to bound the cost perturbation term. Let
\[
\widehat\mu=\widehat\mu_{i,T}^{(A\to\mathcal R)},
\qquad
\widehat\nu=\widehat\nu_{j,T}^{(B\to\mathcal R)}
\]
inside this paragraph only. By definition of $\varepsilon_{c,i,j}$, for every
coupling $\pi\in\Pi(\widehat\mu,\widehat\nu)$,
\[
\begin{aligned}
\left|
\mathbb E_{\pi}[c(z,z')]
-
\mathbb E_{\pi}[\widetilde c(z,z')]
\right|
&=
\left|
\int_{\mathcal R\times\mathcal R}
\left[
c(z,z')-\widetilde c(z,z')
\right]
d\pi(z,z')
\right|\\
&\le
\int_{\mathcal R\times\mathcal R}
\left|
c(z,z')-\widetilde c(z,z')
\right|
d\pi(z,z')\\
&\le
\varepsilon_{c,i,j}.
\end{aligned}
\]
Hence, for every $\pi\in\Pi(\widehat\mu,\widehat\nu)$,
\[
\mathbb E_{\pi}[c(z,z')]
\le
\mathbb E_{\pi}[\widetilde c(z,z')]
+
\varepsilon_{c,i,j},
\qquad
\mathbb E_{\pi}[\widetilde c(z,z')]
\le
\mathbb E_{\pi}[c(z,z')]
+
\varepsilon_{c,i,j}.
\]
Taking the infimum of the first inequality over
$\pi\in\Pi(\widehat\mu,\widehat\nu)$ gives
\[
\begin{aligned}
W_c(\widehat\mu,\widehat\nu)
&=
\inf_{\pi\in\Pi(\widehat\mu,\widehat\nu)}
\mathbb E_{\pi}[c(z,z')]\\
&\le
\inf_{\pi\in\Pi(\widehat\mu,\widehat\nu)}
\left\{
\mathbb E_{\pi}[\widetilde c(z,z')]
+
\varepsilon_{c,i,j}
\right\}\\
&=
W_{\widetilde c}(\widehat\mu,\widehat\nu)
+
\varepsilon_{c,i,j}.
\end{aligned}
\]
Taking the infimum of the second inequality gives the reverse direction:
\[
\begin{aligned}
W_{\widetilde c}(\widehat\mu,\widehat\nu)
&=
\inf_{\pi\in\Pi(\widehat\mu,\widehat\nu)}
\mathbb E_{\pi}[\widetilde c(z,z')]\\
&\le
\inf_{\pi\in\Pi(\widehat\mu,\widehat\nu)}
\left\{
\mathbb E_{\pi}[c(z,z')]
+
\varepsilon_{c,i,j}
\right\}\\
&=
W_c(\widehat\mu,\widehat\nu)
+
\varepsilon_{c,i,j}.
\end{aligned}
\]
Combining the two one-sided inequalities yields
\[
\left|
W_c(\widehat\mu,\widehat\nu)
-
W_{\widetilde c}(\widehat\mu,\widehat\nu)
\right|
\le
\varepsilon_{c,i,j}.
\]
Therefore,
\[
\begin{aligned}
\left|
D_{\mathcal R}(i,j)
-
\widehat D_{\mathcal R}^{\widetilde c}(i,j)
\right|
&\le
\delta_{i,j}
+
\varepsilon_{c,i,j}\\
&\le
\bar\epsilon_i,
\qquad
\forall j\in\mathcal F^{(B)}.
\end{aligned}
\]
Thus the perturbed empirical scores satisfy the uniform perturbation condition
of Theorem~\ref{thm:matching_recovery} with
$\epsilon_i=\bar\epsilon_i$. Since $\Delta_i>2\bar\epsilon_i$, applying
Theorem~\ref{thm:matching_recovery} gives
\[
\arg\min_{j\in\mathcal F^{(B)}}
\widehat D_{\mathcal R}^{\widetilde c}(i,j)
=
j_i^\star.
\]
\end{proof}

\paragraph{Interpretation.}
This corollary separates the two perturbation sources that appear in practice.
The first source is distributional: finite top-$K$ sampling, corpus variation, or
SAE retraining may change the projected feature distributions. This contributes
the term $\delta_{i,j}$. The second source is geometric: the ground cost may be
approximate, learned, or computed from a proxy representation. This contributes
the term $\varepsilon_{c,i,j}$. The result shows that these errors add at the
score level. Matching remains stable whenever the oracle Wasserstein margin
$\Delta_i$ is larger than twice the worst-case combined error over source
features.

\subsection{Matching recovery under Wasserstein margin}
\label{app:matching_recovery}

\paragraph{Definitions.}
Fix a target feature $i\in\mathcal{F}^{(A)}$ and a source feature set
$\mathcal{F}^{(B)}$. For each source feature $j\in\mathcal{F}^{(B)}$, define the
oracle Wasserstein score in the shared reference space $\mathcal R$ by
\[
D_{\mathcal R}(i,j)
=
W_c\!\left(
\mu_i^{(A\to\mathcal R)},
\nu_j^{(B\to\mathcal R)}
\right),
\]
and define the empirical Wasserstein score by
\[
\widehat D_{\mathcal R}(i,j)
=
W_c\!\left(
\widehat{\mu}_{i,T}^{(A\to\mathcal R)},
\widehat{\nu}_{j,T}^{(B\to\mathcal R)}
\right).
\]
Let
\[
j_i^\star
=
\arg\min_{j\in\mathcal{F}^{(B)}}D_{\mathcal R}(i,j)
\]
be the oracle nearest-neighbor match of target feature $i$, and suppose this
minimizer is unique. Define the Wasserstein matching margin by
\[
\Delta_i
=
\min_{j\neq j_i^\star}
\left[
D_{\mathcal R}(i,j)-D_{\mathcal R}(i,j_i^\star)
\right].
\]
Finally, define the uniform empirical score perturbation by
\[
\epsilon_i
=
\sup_{j\in\mathcal{F}^{(B)}}
\left|
\widehat D_{\mathcal R}(i,j)-D_{\mathcal R}(i,j)
\right|.
\]

\noindent
\textbf{Theorem~\ref{thm:matching_recovery}
(Matching recovery under Wasserstein margin).}
\emph{
Fix $i\in\mathcal F^{(A)}$ and suppose $j_i^\star$ is unique with $\Delta_i>0$. If $\sup_{j\in\mathcal F^{(B)}}|\widehat D_{\mathcal R}(i,j)-D_{\mathcal R}(i,j)|\le\epsilon_i$ and $\Delta_i>2\epsilon_i$, then $\arg\min_{j\in\mathcal F^{(B)}}\widehat D_{\mathcal R}(i,j)=j_i^\star$.
}

\begin{proof}
Since $j_i^\star$ is the unique minimizer of the oracle scores, for every
$j\neq j_i^\star$, $D_{\mathcal R}(i,j)-D_{\mathcal R}(i,j_i^\star)
\ge
\Delta_i.$
By the uniform perturbation bound, for every $j\in\mathcal{F}^{(B)}$,
\[
\left|
\widehat D_{\mathcal R}(i,j)-D_{\mathcal R}(i,j)
\right|
\le
\epsilon_i.
\]
Therefore, for every $j\neq j_i^\star$,
\[
\widehat D_{\mathcal R}(i,j)
\ge
D_{\mathcal R}(i,j)-\epsilon_i,
\qquad
\widehat D_{\mathcal R}(i,j_i^\star)
\le
D_{\mathcal R}(i,j_i^\star)+\epsilon_i.
\]
Subtracting the two inequalities gives
\[
\begin{aligned}
\widehat D_{\mathcal R}(i,j)-\widehat D_{\mathcal R}(i,j_i^\star)
&\ge
\left(D_{\mathcal R}(i,j)-\epsilon_i\right)
-
\left(D_{\mathcal R}(i,j_i^\star)+\epsilon_i\right)\\
&=
D_{\mathcal R}(i,j)-D_{\mathcal R}(i,j_i^\star)-2\epsilon_i\\
&\ge
\Delta_i-2\epsilon_i.
\end{aligned}
\]
If $\Delta_i>2\epsilon_i$, then
\[
\widehat D_{\mathcal R}(i,j)
>
\widehat D_{\mathcal R}(i,j_i^\star),
\qquad
\forall j\neq j_i^\star.
\]
Thus $j_i^\star$ is the unique minimizer of the empirical scores:
\[
\arg\min_{j\in\mathcal{F}^{(B)}}\widehat D_{\mathcal R}(i,j)
=
j_i^\star.
\]
\end{proof}

\paragraph{Interpretation.}
This theorem converts score stability into decision stability. The Wasserstein
stability theorem controls how much the scores can move; the margin condition
ensures that this movement is not large enough to change the nearest-neighbor
decision. The factor $2$ appears because the oracle best match may become worse
by at most $\epsilon_i$, while a competing match may become better by at most
$\epsilon_i$.

\paragraph{Combined distribution and cost perturbations.}
The matching recovery theorem is stated in terms of an abstract score
perturbation $\epsilon_i$. We now show how such a perturbation bound follows
when both the projected feature distributions and the ground cost are perturbed.

Fix a target feature $i\in\mathcal F^{(A)}$ and a source feature
$j\in\mathcal F^{(B)}$. Recall the oracle Wasserstein score
\[
D_{\mathcal R}(i,j)
=
W_c\!\left(
\mu_i^{(A\to\mathcal R)},
\nu_j^{(B\to\mathcal R)}
\right).
\]
When both the projected distributions and the ground cost are perturbed, define the
approximate empirical score
\[
\widehat D^{\widetilde c}_{\mathcal R}(i,j)
=
W_{\widetilde c}\!\left(
\widehat{\mu}_{i,T}^{(A\to\mathcal R)},
\widehat{\nu}_{j,T}^{(B\to\mathcal R)}
\right).
\]
Here, $c$ is the reference ground cost, while $\widetilde c$ is an approximate
or learned ground cost. For each pair $(i,j)$, define the projected-distribution perturbation error
\[
\delta_{i,j}
:=
L_x
W_\rho\!\left(
\mu_i^{(A\to\mathcal R)},
\widehat{\mu}_{i,T}^{(A\to\mathcal R)}
\right)
+
L_y
W_\rho\!\left(
\nu_j^{(B\to\mathcal R)},
\widehat{\nu}_{j,T}^{(B\to\mathcal R)}
\right),
\]
where $L_x$ and $L_y$ are the coordinate-wise Lipschitz constants of $c$ on the
relevant support. Define the ground-cost perturbation error by
\[
\varepsilon_{c,i,j}
:=
\sup_{(x,y)\in
\operatorname{supp}(\widehat{\mu}_{i,T}^{(A\to\mathcal R)})
\times
\operatorname{supp}(\widehat{\nu}_{j,T}^{(B\to\mathcal R)})}
\left|
c(x,y)-\widetilde c(x,y)
\right|.
\]
Finally, define the worst-case combined perturbation for target feature $i$ by
\[
\bar\epsilon_i
:=
\max_{j\in\mathcal F^{(B)}}
\left[
\delta_{i,j}
+
\varepsilon_{c,i,j}
\right].
\]

\begin{corollary}[Stable matching under distribution and cost perturbations]
\label{cor:stable_matching_measure_cost}
Let $j_i^\star$ be the unique oracle nearest-neighbor match under
$D_{\mathcal R}(i,j)$, and let $\Delta_i$ be its Wasserstein matching margin. If
\[
\Delta_i>2\bar\epsilon_i,
\]
then nearest-neighbor matching with the perturbed empirical scores recovers the
oracle match:
\[
\arg\min_{j\in\mathcal F^{(B)}}
\widehat D^{\widetilde c}_{\mathcal R}(i,j)
=
j_i^\star.
\]
\end{corollary}

\begin{proof}
For each source feature $j\in\mathcal F^{(B)}$, add and subtract the
intermediate score that uses the perturbed distributions but the original ground
cost: $W_c\!\left(
\widehat{\mu}_{i,T}^{(A\to\mathcal R)},
\widehat{\nu}_{j,T}^{(B\to\mathcal R)}
\right).$ Then
\[
\begin{aligned}
&
\left|
D_{\mathcal R}(i,j)
-
\widehat D^{\widetilde c}_{\mathcal R}(i,j)
\right|\\
&=
\left|
W_c\!\left(
\mu_i^{(A\to\mathcal R)},
\nu_j^{(B\to\mathcal R)}
\right)
-
W_{\widetilde c}\!\left(
\widehat{\mu}_{i,T}^{(A\to\mathcal R)},
\widehat{\nu}_{j,T}^{(B\to\mathcal R)}
\right)
\right|\\
&\le
\left|
W_c\!\left(
\mu_i^{(A\to\mathcal R)},
\nu_j^{(B\to\mathcal R)}
\right)
-
W_c\!\left(
\widehat{\mu}_{i,T}^{(A\to\mathcal R)},
\widehat{\nu}_{j,T}^{(B\to\mathcal R)}
\right)
\right|\\
&\quad+
\left|
W_c\!\left(
\widehat{\mu}_{i,T}^{(A\to\mathcal R)},
\widehat{\nu}_{j,T}^{(B\to\mathcal R)}
\right)
-
W_{\widetilde c}\!\left(
\widehat{\mu}_{i,T}^{(A\to\mathcal R)},
\widehat{\nu}_{j,T}^{(B\to\mathcal R)}
\right)
\right|.
\end{aligned}
\]
By Theorem~\ref{thm:wasserstein_stability}, the first term is bounded by
\[
L_x
W_\rho\!\left(
\mu_i^{(A\to\mathcal R)},
\widehat{\mu}_{i,T}^{(A\to\mathcal R)}
\right)
+
L_y
W_\rho\!\left(
\nu_j^{(B\to\mathcal R)},
\widehat{\nu}_{j,T}^{(B\to\mathcal R)}
\right)
=
\delta_{i,j}.
\]
By Theorem 2, the second term is bounded by $\varepsilon_{c,i,j}.$ Therefore,
\[
\left|
D_{\mathcal R}(i,j)
-
\widehat D^{\widetilde c}_{\mathcal R}(i,j)
\right|
\le
\delta_{i,j}
+
\varepsilon_{c,i,j}
\le
\bar\epsilon_i,
\qquad
\forall j\in\mathcal F^{(B)}.
\]
Thus the perturbed empirical scores satisfy the uniform perturbation condition
of Theorem~\ref{thm:matching_recovery} with $\epsilon_i=\bar\epsilon_i$.
Since $\Delta_i>2\bar\epsilon_i$, Theorem~\ref{thm:matching_recovery} gives
\[
\arg\min_{j\in\mathcal F^{(B)}}
\widehat D^{\widetilde c}_{\mathcal R}(i,j)
=
j_i^\star.
\]
\end{proof}

\paragraph{Interpretation.}
This corollary combines the two perturbation sources that arise in practice.
The projected empirical distributions may differ from the oracle projected distributions because
they are estimated from finite top-$K$ activations, and the ground cost may also
be approximate, learned, or flow-aware. The result says that these errors add:
the total score error is controlled by the distribution perturbation term
$\delta_{i,j}$ plus the cost perturbation term $\varepsilon_{c,i,j}$. Matching is
stable whenever the oracle Wasserstein margin $\Delta_i$ is larger than twice
the worst-case combined error. Thus, a match is reliable when the projected
distributions are accurate, the cost approximation is accurate on the empirical
transport support, and the best source feature is separated from its nearest
competitor by a nontrivial margin.


\subsection{Voronoi Semantic Cell Recovery}
\label{app:voronoi_recovery}

We state the recovery result directly in the shared reference space
$(\mathcal R,\rho)$, where all projected feature distributions live. This avoids
introducing an additional semantic map and makes the grouping analysis align
with the method, which assigns circuit nodes using distances between projected
feature distributions in $\mathcal R$.

\paragraph{Semantic centers and cells.}
Let $\xi_1,\ldots,\xi_m\in\mathcal R$ be semantic centers. Here, $m$ is the
number of semantic supernodes in the compressed circuit, and $\xi_k$ is the
representative center of the $k$-th semantic supernode. These centers induce a
Voronoi partition of the shared reference space:
\[
\mathsf{Cell}_k
:=
\left\{
z\in\mathcal R:
\rho(z,\xi_k)
\le
\rho(z,\xi_r),
\ \forall r\in[m]
\right\}.
\]
Thus, a point $z\in\mathcal R$ belongs to $\mathsf{Cell}_k$ if it is at least as
close to $\xi_k$ as to every other semantic center.

\paragraph{Population and empirical projected feature distributions.}
Fix a target feature $i\in\mathcal F^{(A)}$. Let
$\mu_i^{(A\to\mathcal R)}\in\mathcal P(\mathcal R)$ denote the population
projected feature distribution, and let
$\widehat\mu_{i,T}^{(A\to\mathcal R)}\in\mathcal P(\mathcal R)$ denote the
empirical projected feature distribution constructed from $T$ token positions.
In our construction,
\[
\widehat\mu_{i,T}^{(A\to\mathcal R)}
=
\sum_{t\in\mathcal I_{i,T}^{(A),K}}
\widehat w_{i,t}^{(A)}
\delta_{z_t^{(\mathcal R)}}.
\]
The target-layer activation function selects the top-$K$ token indices
$\mathcal I_{i,T}^{(A),K}$ and normalized weights $\widehat w_{i,t}^{(A)}$,
while the support points $z_t^{(\mathcal R)}$ lie in the shared reference space.

\paragraph{Semantic-center assignment scores.}
Since a feature is represented by a distribution rather than a single point, we
assign it to a semantic center by its expected distance to that center. For any
projected distribution $\eta\in\mathcal P(\mathcal R)$, define
\[
\mathcal Q_k(\eta)
:=
\mathbb E_{z\sim\eta}
\left[
\rho(z,\xi_k)
\right].
\]
Here, $\mathcal Q_k(\eta)$ is the average distance from the feature distribution
$\eta$ to the $k$-th semantic center. The population and empirical assignments
of feature $i$ are
\[
k_i^\star
:=
\arg\min_{k\in[m]}
\mathcal Q_k\!\left(\mu_i^{(A\to\mathcal R)}\right),
\qquad
\widehat k_{i,T}
:=
\arg\min_{k\in[m]}
\mathcal Q_k\!\left(\widehat\mu_{i,T}^{(A\to\mathcal R)}\right).
\]
For the empirical distribution, the score has the finite form
\[
\mathcal Q_k\!\left(\widehat\mu_{i,T}^{(A\to\mathcal R)}\right)
=
\sum_{t\in\mathcal I_{i,T}^{(A),K}}
\widehat w_{i,t}^{(A)}
\rho\!\left(z_t^{(\mathcal R)},\xi_k\right).
\]
Thus, empirical assignment compares the weighted average distance from the
feature's selected activation contexts to each semantic center.

\paragraph{Voronoi assignment margin.}
Assume that the population assignment $k_i^\star$ is unique. Define the Voronoi
assignment margin by
\[
\Gamma_i
:=
\min_{r\neq k_i^\star}
\left[
\mathcal Q_r\!\left(\mu_i^{(A\to\mathcal R)}\right)
-
\mathcal Q_{k_i^\star}\!\left(\mu_i^{(A\to\mathcal R)}\right)
\right].
\]
The condition $\Gamma_i>0$ means that the population projected distribution of
feature $i$ is strictly closer, in expected distance, to its assigned semantic
center $\xi_{k_i^\star}$ than to every competing center.

\noindent
\textbf{Theorem~\ref{thm:voronoi_recovery}
(Voronoi semantic cell recovery).}
\emph{
Suppose the population semantic assignment $k_i^\star$ is unique with margin
$\Gamma_i>0$. If
\[
W_\rho\!\left(
\widehat\mu_{i,T}^{(A\to\mathcal R)},
\mu_i^{(A\to\mathcal R)}
\right)
<
\frac{\Gamma_i}{2},
\]
then the empirical semantic assignment recovers the population assignment:
$\widehat k_{i,T}=k_i^\star$.
}

\begin{proof}
For each semantic center $\xi_k$, define the point-to-center distance function
$h_k:\mathcal R\to\mathbb R_{\ge0}$ by $h_k(z)=\rho(z,\xi_k)$. The first step is
to show that $h_k$ is $1$-Lipschitz with respect to $\rho$. For any
$z,z'\in\mathcal R$, the reverse triangle inequality gives
\[
\begin{aligned}
|h_k(z)-h_k(z')|
&=
\left|
\rho(z,\xi_k)-\rho(z',\xi_k)
\right| \\
&\le
\rho(z,z').
\end{aligned}
\]
Thus, moving a point by at most $\rho(z,z')$ can change its distance to any fixed
semantic center by at most $\rho(z,z')$.

We next lift this pointwise Lipschitz property to distributions. Let
$\eta,\eta'\in\mathcal P(\mathcal R)$ be any two projected feature
distributions, and let $\gamma\in\Pi(\eta,\eta')$ be any coupling between them.
Because the first marginal of $\gamma$ is $\eta$ and the second marginal is
$\eta'$, we can rewrite the score difference under the same coupling:
\[
\begin{aligned}
\mathcal Q_k(\eta)-\mathcal Q_k(\eta')
&=
\int_{\mathcal R} h_k(z)\,d\eta(z)
-
\int_{\mathcal R} h_k(z')\,d\eta'(z') \\
&=
\int_{\mathcal R\times\mathcal R}
\left[
h_k(z)-h_k(z')
\right]
\,d\gamma(z,z').
\end{aligned}
\]
Therefore,
\[
\begin{aligned}
\left|
\mathcal Q_k(\eta)-\mathcal Q_k(\eta')
\right|
&=
\left|
\int_{\mathcal R\times\mathcal R}
\left[
h_k(z)-h_k(z')
\right]
\,d\gamma(z,z')
\right| \\
&\le
\int_{\mathcal R\times\mathcal R}
\left|
h_k(z)-h_k(z')
\right|
\,d\gamma(z,z') \\
&\le
\int_{\mathcal R\times\mathcal R}
\rho(z,z')
\,d\gamma(z,z').
\end{aligned}
\]
Since the above inequality holds for every coupling
$\gamma\in\Pi(\eta,\eta')$, taking the infimum over couplings gives
\[
\left|
\mathcal Q_k(\eta)-\mathcal Q_k(\eta')
\right|
\le
W_\rho(\eta,\eta').
\]
Apply this bound to
$\eta=\widehat\mu_{i,T}^{(A\to\mathcal R)}$ and
$\eta'=\mu_i^{(A\to\mathcal R)}$. Define the empirical distribution error
\[
\varepsilon_i
:=
W_\rho\!\left(
\widehat\mu_{i,T}^{(A\to\mathcal R)},
\mu_i^{(A\to\mathcal R)}
\right).
\]
Then, for every semantic center $k\in[m]$,
\[
\left|
\mathcal Q_k\!\left(\widehat\mu_{i,T}^{(A\to\mathcal R)}\right)
-
\mathcal Q_k\!\left(\mu_i^{(A\to\mathcal R)}\right)
\right|
\le
\varepsilon_i.
\]

Now fix any competing center $r\neq k_i^\star$. We compare the empirical score
of the competitor $r$ with the empirical score of the population winner
$k_i^\star$. Add and subtract the corresponding population scores:
\[
\begin{aligned}
&
\mathcal Q_r\!\left(\widehat\mu_{i,T}^{(A\to\mathcal R)}\right)
-
\mathcal Q_{k_i^\star}\!\left(\widehat\mu_{i,T}^{(A\to\mathcal R)}\right) \\
&=
\left[
\mathcal Q_r\!\left(\widehat\mu_{i,T}^{(A\to\mathcal R)}\right)
-
\mathcal Q_r\!\left(\mu_i^{(A\to\mathcal R)}\right)
\right] \\
&\quad+
\left[
\mathcal Q_r\!\left(\mu_i^{(A\to\mathcal R)}\right)
-
\mathcal Q_{k_i^\star}\!\left(\mu_i^{(A\to\mathcal R)}\right)
\right] \\
&\quad+
\left[
\mathcal Q_{k_i^\star}\!\left(\mu_i^{(A\to\mathcal R)}\right)
-
\mathcal Q_{k_i^\star}\!\left(\widehat\mu_{i,T}^{(A\to\mathcal R)}\right)
\right].
\end{aligned}
\]
Each of the three terms has a direct interpretation: the first is the empirical
perturbation of the competitor score, the second is the population Voronoi gap,
and the third is the empirical perturbation of the winning score. By the
definition of $\Gamma_i$ and the uniform score perturbation bound above,
\[
\begin{aligned}
\mathcal Q_r\!\left(\widehat\mu_{i,T}^{(A\to\mathcal R)}\right)
-
\mathcal Q_{k_i^\star}\!\left(\widehat\mu_{i,T}^{(A\to\mathcal R)}\right)
&\ge
\left[-\varepsilon_i\right]
+
\Gamma_i
+
\left[-\varepsilon_i\right] \\
&=
\Gamma_i-2\varepsilon_i.
\end{aligned}
\]
If $\varepsilon_i<\Gamma_i/2$, then $\Gamma_i-2\varepsilon_i>0$. Hence, for
every competing center $r\neq k_i^\star$,
\[
\mathcal Q_r\!\left(\widehat\mu_{i,T}^{(A\to\mathcal R)}\right)
>
\mathcal Q_{k_i^\star}\!\left(\widehat\mu_{i,T}^{(A\to\mathcal R)}\right).
\]
Therefore $k_i^\star$ remains the unique minimizer of the empirical assignment
scores:
\[
\begin{aligned}
\widehat k_{i,T}
&=
\arg\min_{k\in[m]}
\mathcal Q_k\!\left(\widehat\mu_{i,T}^{(A\to\mathcal R)}\right) \\
&=
k_i^\star.
\end{aligned}
\]
This proves the claim.
\end{proof}

\paragraph{Eventual exact recovery.}
Assume the conditions of Theorem~\ref{thm:voronoi_recovery} and suppose
\[
W_\rho\!\left(
\widehat\mu_{i,T}^{(A\to\mathcal R)},
\mu_i^{(A\to\mathcal R)}
\right)
\longrightarrow 0
\qquad
\text{as }T\to\infty.
\]
If $\Gamma_i>0$, then there exists $T_i<\infty$ such that
$\widehat k_{i,T}=k_i^\star$ for all $T\ge T_i$.

\begin{proof}
Since the Wasserstein estimation error converges to zero and $\Gamma_i/2>0$,
there exists $T_i<\infty$ such that, for every $T\ge T_i$,
\[
\begin{aligned}
W_\rho\!\left(
\widehat\mu_{i,T}^{(A\to\mathcal R)},
\mu_i^{(A\to\mathcal R)}
\right)
&<
\frac{\Gamma_i}{2}.
\end{aligned}
\]
The condition of Theorem~\ref{thm:voronoi_recovery} is therefore satisfied for
all $T\ge T_i$, so
\[
\widehat k_{i,T}=k_i^\star,
\qquad
\forall T\ge T_i.
\]
\end{proof}

\paragraph{High-probability Voronoi recovery.}
Assume that for some rate function $r_T(\delta)$,
\[
\mathbb P
\left(
W_\rho\!\left(
\widehat\mu_{i,T}^{(A\to\mathcal R)},
\mu_i^{(A\to\mathcal R)}
\right)
\le
r_T(\delta)
\right)
\ge
1-\delta.
\]
If $r_T(\delta)<\Gamma_i/2$, then $\mathbb P(\widehat k_{i,T}=k_i^\star)\ge1-\delta.$

\begin{proof}
Define the high-probability event
\[
\mathcal E_T
:=
\left\{
W_\rho\!\left(
\widehat\mu_{i,T}^{(A\to\mathcal R)},
\mu_i^{(A\to\mathcal R)}
\right)
\le
r_T(\delta)
\right\}.
\]
By assumption, $\mathbb P(\mathcal E_T)\ge1-\delta$. On the event
$\mathcal E_T$, the rate condition gives
\[
\begin{aligned}
W_\rho\!\left(
\widehat\mu_{i,T}^{(A\to\mathcal R)},
\mu_i^{(A\to\mathcal R)}
\right)
&\le
r_T(\delta)
<
\frac{\Gamma_i}{2}.
\end{aligned}
\]
Thus Theorem~\ref{thm:voronoi_recovery} implies
$\widehat k_{i,T}=k_i^\star$ on $\mathcal E_T$. Therefore,
\[
\begin{aligned}
\mathbb P(\widehat k_{i,T}=k_i^\star)
&\ge
\mathbb P(\mathcal E_T)\\
&\ge
1-\delta.
\end{aligned}
\]
\end{proof}

\paragraph{Interpretation.}
This theorem shows that semantic grouping is stable when the empirical projected
feature distribution remains sufficiently close to the population projected
distribution. The margin $\Gamma_i$ measures how much smaller the assignment
score of the correct semantic center is than the score of the nearest competing
center. Since each assignment score $\mathcal Q_k$ changes by at most
$W_\rho(\widehat\mu_{i,T}^{(A\to\mathcal R)},\mu_i^{(A\to\mathcal R)})$, the
gap between the correct and competing centers can shrink by at most twice this
amount. Thus recovery is controlled by the margin-normalized error
\[
\frac{
W_\rho\!\left(
\widehat\mu_{i,T}^{(A\to\mathcal R)},
\mu_i^{(A\to\mathcal R)}
\right)
}{
\Gamma_i
}.
\]
For circuit compression, this means that large-margin features can be grouped
confidently into semantic supernodes, while small-margin features are close to a
semantic boundary and should be treated as ambiguous or left uncompressed.

\paragraph{Rates under the reference-space Voronoi recovery theorem.}
The recovery condition depends on how accurately the empirical projected
feature distribution approximates its population counterpart in the shared
reference space. Specifically, Theorem~\ref{thm:voronoi_recovery} shows that
semantic-cell recovery is guaranteed whenever
\[
W_\rho\!\left(
\widehat\mu_{i,n}^{(A\to\mathcal R)},
\mu_i^{(A\to\mathcal R)}
\right)
<
\frac{\Gamma_i}{2}.
\]
Thus, any high-probability convergence bound of the form
\[
\mathbb P\!\left(
W_\rho\!\left(
\widehat\mu_{i,n}^{(A\to\mathcal R)},
\mu_i^{(A\to\mathcal R)}
\right)
\le
r_n(\delta)
\right)
\ge
1-\delta
\]
implies recovery with probability at least $1-\delta$ whenever
$r_n(\delta)<\Gamma_i/2$. Here, $n$ denotes the effective number of samples used to construct the empirical
projected distribution. If the empirical distribution is built from all sampled
contexts, then $n=T$; if it is built only from the selected top-$K$ contexts,
then $n=K$ unless an additional top-$K$ population convergence argument is used.

\paragraph{Rate in one dimension.}
If the projected feature distribution has effective dimension $d=1$, empirical
Wasserstein convergence typically satisfies
\[
W_\rho(\widehat\mu_{i,T},\mu_i)
=
O_p(T^{-1/2}).
\]
A high-probability form can be written as $W_\rho(\widehat\mu_{i,T},\mu_i)
\le
C_1
\sqrt{\frac{\log(1/\delta)}{T}}$
with probability at least $1-\delta$. Therefore, Voronoi recovery is guaranteed
whenever
\[
\begin{aligned}
C_1
\sqrt{\frac{\log(1/\delta)}{T}}
&<
\frac{\Gamma_i}{2}\\
\Longleftrightarrow\qquad
T
&>
\frac{4C_1^2}{\Gamma_i^2}
\log\frac1\delta.
\end{aligned}
\]
Thus, for $d=1$, $T
=
O\!\left(
\frac{1}{\Gamma_i^2}
\log\frac1\delta
\right).$

\paragraph{Rate in two dimensions.}
If the projected feature distribution has effective dimension $d=2$, empirical
Wasserstein convergence has the critical logarithmic rate
\[
W_\rho(\widehat\mu_{i,T},\mu_i)
=
O_p\!\left(
\frac{\log T}{\sqrt T}
\right).
\]
A high-probability form can be written as
\[
W_\rho(\widehat\mu_{i,T},\mu_i)
\le
C_2
\frac{\log T+\log(1/\delta)}{\sqrt T}.
\]
Thus, semantic-cell recovery is guaranteed whenever
\[
\begin{aligned}
C_2
\frac{\log T+\log(1/\delta)}{\sqrt T}
&<
\frac{\Gamma_i}{2}\\
\Longleftrightarrow\qquad
\frac{\sqrt T}{\log T+\log(1/\delta)}
&>
\frac{2C_2}{\Gamma_i}.
\end{aligned}
\]
Equivalently, up to logarithmic factors, the sufficient sample size scales as $T
=
\widetilde O\!\left(
\frac{1}{\Gamma_i^2}
\right).$

\paragraph{Rate in dimension $d>2$.}
If the projected feature distribution has effective dimension $d>2$, empirical
Wasserstein convergence typically satisfies $W_\rho(\widehat\mu_{i,T},\mu_i)
=
O_p(T^{-1/d}).$
A high-probability bound may be written as
\[
W_\rho(\widehat\mu_{i,T},\mu_i)
\le
C_d T^{-1/d}
+
C_\delta
\sqrt{\frac{\log(1/\delta)}{T}}.
\]
For $d>2$, the dominant term is typically $T^{-1/d}$, while the
high-probability concentration term scales as $T^{-1/2}$ up to logarithmic
factors. Since $T^{-1/2}$ decays faster than $T^{-1/d}$ when $d>2$, the
dominant sample-complexity scaling is governed by $T^{-1/d}$. Ignoring the
lower-order concentration term, recovery is guaranteed whenever
\[
\begin{aligned}
C_d T^{-1/d}
&<
\frac{\Gamma_i}{2}\\
\Longleftrightarrow\qquad
T^{1/d}
&>
\frac{2C_d}{\Gamma_i}\\
\Longleftrightarrow\qquad
T
&>
\left(
\frac{2C_d}{\Gamma_i}
\right)^d.
\end{aligned}
\]
Thus, for $d>2$, $T
=
O\!\left[
\left(
\frac{1}{\Gamma_i}
\right)^d
\right].$

\paragraph{Finite-support rate.}
If the population projected feature distribution is supported on at most $S$
reference-space atoms,
\[
\mu_i
=
\sum_{s=1}^{S}p_s\delta_{u_s},
\]
and the empirical distribution has the same support with empirical weights
$\widehat p_s$, then
\[
W_\rho(\widehat\mu_{i,T},\mu_i)
\le
D_{\max}
\|\widehat p-p\|_1,
\qquad
D_{\max}
:=
\max_{a,b\in[S]}
\rho(u_a,u_b).
\]
By multinomial concentration,
\[
\|\widehat p-p\|_1
=
O_p\!\left(
\sqrt{
\frac{S+\log(1/\delta)}{T}
}
\right).
\]
Hence,
\[
W_\rho(\widehat\mu_{i,T},\mu_i)
=
O_p\!\left(
D_{\max}
\sqrt{
\frac{S+\log(1/\delta)}{T}
}
\right).
\]
Voronoi recovery is guaranteed if
\[
\begin{aligned}
D_{\max}
\sqrt{
\frac{S+\log(1/\delta)}{T}
}
&<
\frac{\Gamma_i}{2}\\
\Longleftrightarrow\qquad
T
&>
\frac{4D_{\max}^2}{\Gamma_i^2}
\left(
S+\log\frac1\delta
\right).
\end{aligned}
\]
Thus the finite-support sample complexity scales as $T
=
O\!\left(
\frac{D_{\max}^2}{\Gamma_i^2}
\left[
S+\log\frac1\delta
\right]
\right).$

\begin{table}[t]
\centering
\caption{Typical empirical Wasserstein convergence rates and sufficient sample
complexities for Voronoi semantic recovery in the shared reference space. Here
$\Gamma_i$ is the Voronoi assignment margin, $\delta\in(0,1)$ is the failure
probability, and $T$ is the number of samples used to construct the empirical
projected distribution.}
\label{tab:voronoi_convergence_rates}
\begin{tabular}{lll}
\toprule
Setting & Wasserstein rate $r(T)$ & Sufficient sample complexity \\
\midrule
$d=1$ &
$O\!\left(T^{-1/2}\right)$ &
$O\!\left(\frac{1}{\Gamma_i^2}\log\frac1\delta\right)$ \\

$d=2$ &
$O\!\left(\frac{\log T}{\sqrt T}\right)$ &
$\widetilde O\!\left(\frac{1}{\Gamma_i^2}\log\frac1\delta\right)$ \\

$d>2$ &
$O\!\left(T^{-1/d}\right)$ &
$O\!\left(\frac{1}{\Gamma_i^d}\right)$ \\

Finite support $S$ &
$O\!\left(D_{\max}\sqrt{\frac{S+\log(1/\delta)}{T}}\right)$ &
$O\!\left(
\frac{D_{\max}^2}{\Gamma_i^2}
\left(S+\log\frac1\delta\right)
\right)$ \\
\bottomrule
\end{tabular}
\end{table}

\subsection{Synthetic Verification of Voronoi Semantic-Cell Recovery}
\label{sec:voronoi_synthetic_verification}

\begin{figure}[htp]
    \centering
    \includegraphics[width=1\linewidth]{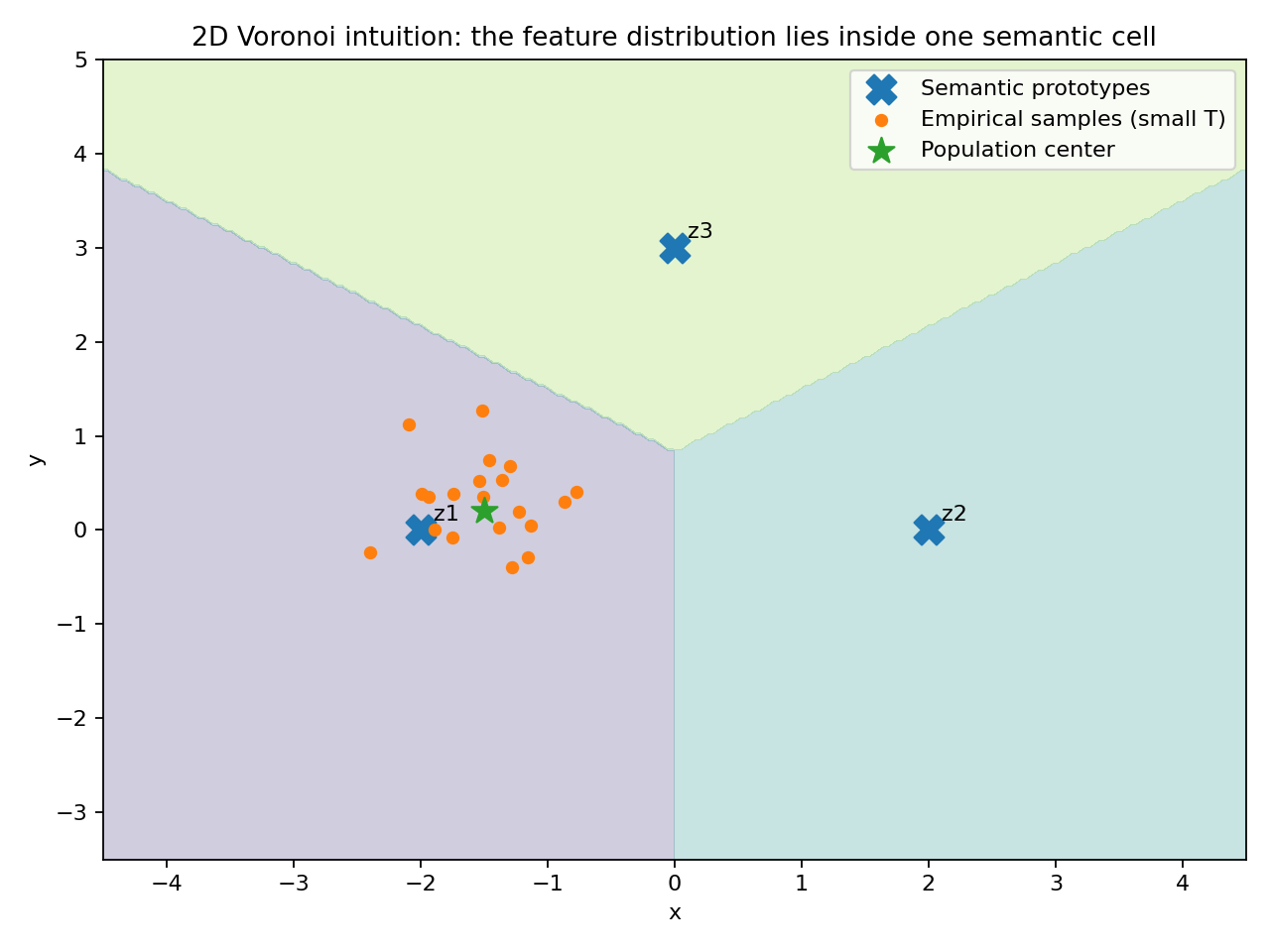}
    \caption{
\textbf{Voronoi-based semantic cell intuition.}
We visualize a synthetic semantic representation space with three prototypes
$z_1,z_2,z_3$, which induce a Voronoi partition of the space. The orange points
represent empirical samples from a feature distribution, while the star denotes
the population center. In this example, the distribution lies inside the
Voronoi cell of $z_1$, meaning that its semantic assignment is stable as long as
sampling or representation perturbations do not move the empirical distribution
across the decision boundary. This figure illustrates the geometric intuition
behind the Voronoi recovery theorem: if a feature distribution is separated from
competing cells by a positive margin, then sufficiently accurate empirical
estimates recover the same semantic cell.
}
    \label{fig:voronoi_intuition}
\end{figure}

We provide a synthetic experiment to verify the mechanism of
Theorem~\ref{thm:voronoi_recovery}. The goal is not to evaluate feature
matching performance, but to isolate the geometric statement of the theorem:
if the empirical projected feature distribution remains within the
Wasserstein stability radius of the population projected distribution, then the
semantic-cell assignment is preserved.

\paragraph{Setup.}
We consider the shared reference space $\mathcal R=\mathbb R$ with base metric
$\rho(z,z')=|z-z'|$. The population projected feature distribution is
$\mu=\mathcal N(m,\sigma^2)$ with $m=-1$ and $\sigma=0.6$. We use two semantic
centers $\xi_1=-2$ and $\xi_2=2$. For any projected distribution
$\eta\in\mathcal P(\mathcal R)$, the assignment score to center $\xi_k$ is
\[
\mathcal Q_k(\eta)
=
\mathbb E_{z\sim\eta}
\left[
\rho(z,\xi_k)
\right].
\]
The population semantic assignment is therefore
\[
k^\star
=
\arg\min_{k\in\{1,2\}}
\mathcal Q_k(\mu),
\qquad
\Gamma
=
\mathcal Q_{2}(\mu)-\mathcal Q_{1}(\mu),
\]
where $\Gamma>0$ means that the population distribution is assigned to
$\xi_1$ with positive Voronoi margin. Since the experiment is one-dimensional,
the population scores can be computed in closed form using
\[
\mathcal Q_k(\mu)
=
\mathbb E_{Z\sim\mathcal N(m,\sigma^2)}
|Z-\xi_k|.
\]
For each sample size $T$, we draw samples
$z_1,\ldots,z_T\sim\mu$ and construct the empirical distribution
$\widehat\mu_T=T^{-1}\sum_{t=1}^{T}\delta_{z_t}$. The empirical assignment is
\[
\widehat k_T
=
\arg\min_{k\in\{1,2\}}
\mathcal Q_k(\widehat\mu_T),
\qquad
\mathcal Q_k(\widehat\mu_T)
=
\frac1T
\sum_{t=1}^{T}
|z_t-\xi_k|.
\]
Theorem~\ref{thm:voronoi_recovery} predicts that
\[
W_\rho(\widehat\mu_T,\mu)<\frac{\Gamma}{2}
\quad
\Longrightarrow
\quad
\widehat k_T=k^\star.
\]
In one dimension, we estimate
\[
W_\rho(\widehat\mu_T,\mu)
=
\int_{\mathbb R}
\left|
\widehat F_T(t)-F(t)
\right|dt,
\]
where $\widehat F_T$ is the empirical CDF and $F$ is the Gaussian CDF.

\paragraph{Voronoi-cell intuition.}
Figure~\ref{fig:voronoi_intuition} illustrates the geometric meaning of the
semantic-cell recovery theorem. The semantic centers partition the reference
space into Voronoi cells, and the population feature distribution lies inside
one cell with positive separation from the nearest boundary. The key implication
is that semantic grouping is not determined by exact sample locations, but by
whether the empirical distribution remains inside the same margin-defined
region. Thus, small sampling fluctuations are harmless when the population
distribution is well separated from competing semantic centers. This is the
geometric reason why large-margin features can be compressed into the same
supernode confidently.


\paragraph{Convergence of assignment scores.}
Figure~\ref{fig:voronoi_score_convergence} shows that the empirical assignment
scores $\mathcal Q_1(\widehat\mu_T)$ and $\mathcal Q_2(\widehat\mu_T)$
concentrate around their population values $\mathcal Q_1(\mu)$ and
$\mathcal Q_2(\mu)$ as $T$ increases. Since
$\mathcal Q_1(\mu)<\mathcal Q_2(\mu)$, the population assignment is
$k^\star=1$. The empirical assignment becomes stable once the empirical score
gap remains positive:
\[
\mathcal Q_2(\widehat\mu_T)-\mathcal Q_1(\widehat\mu_T)>0.
\]
\begin{figure}[htp]
    \centering
    \includegraphics[width=1\linewidth]{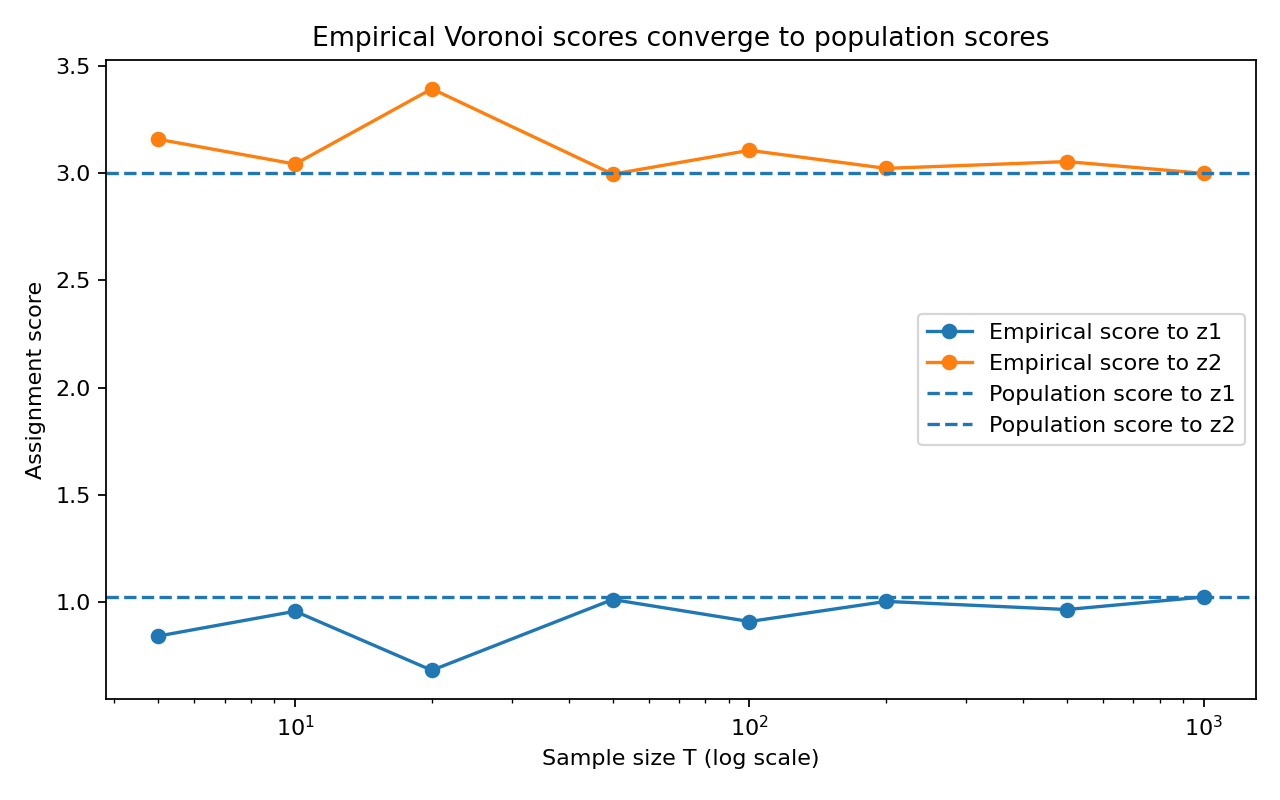}
    \caption{
\textbf{Convergence of empirical Voronoi assignment scores.}
For each semantic center $\xi_k\in\mathcal R$, the assignment score is
$\mathcal Q_k(\mu)=\int_{\mathcal R}\rho(z,\xi_k)\,d\mu(z)$, and the empirical
score $\mathcal Q_k(\widehat\mu)$ is computed from samples. The dashed
horizontal lines show the population scores $\mathcal Q_1(\mu)$ and
$\mathcal Q_2(\mu)$, while the solid curves show the empirical scores as the
sample size $T$ increases. The empirical scores concentrate around their
population values, and the score to the correct semantic center $\xi_1$ remains
smaller than the score to the competing semantic center $\xi_2$. This supports
the theorem's mechanism: semantic recovery follows once the empirical score
perturbation is small relative to the population assignment margin.
}
    \label{fig:voronoi_score_convergence}
\end{figure}

\begin{figure}[htp]
    \centering
    \includegraphics[width=1\linewidth]{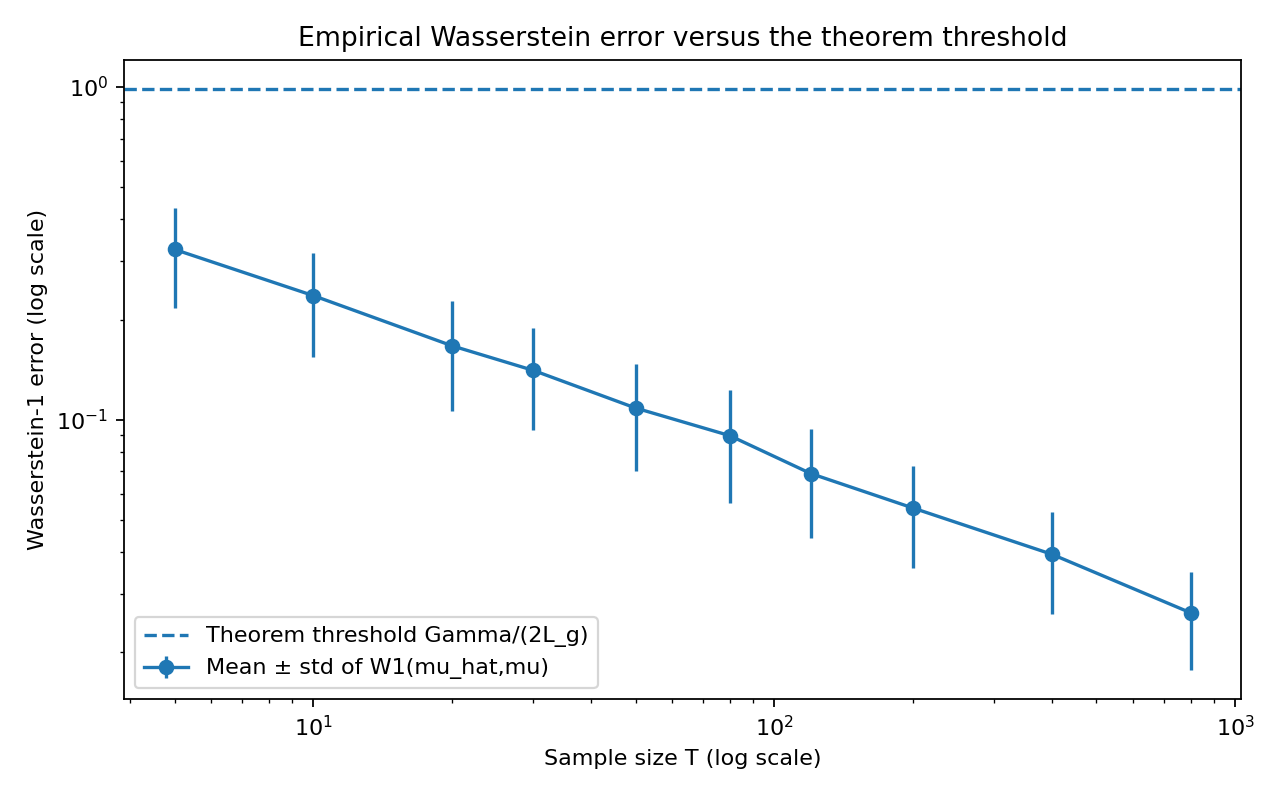}
    \caption{
\textbf{Empirical Wasserstein error versus the Voronoi recovery threshold.}
The blue curve shows the empirical Wasserstein error
$W_\rho(\widehat\mu_{i,T}^{(A\to\mathcal R)},\mu_i^{(A\to\mathcal R)})$ as a
function of sample size $T$, averaged over Monte Carlo trials with error bars
indicating one standard deviation. The dashed horizontal line is the theoretical
recovery threshold $\Gamma_i/2$, where $\Gamma_i$ is the population Voronoi
assignment margin of feature $i$. The Voronoi recovery theorem guarantees
correct semantic-cell recovery whenever
$W_\rho(\widehat\mu_{i,T}^{(A\to\mathcal R)},\mu_i^{(A\to\mathcal R)})
<\Gamma_i/2$. Thus, the plot visualizes the certificate condition: as $T$
grows, the empirical projected distribution converges toward the population
projected distribution, making the sufficient recovery condition easier to
satisfy.
}
    \label{fig:voronoi_wasserstein_threshold}
\end{figure}

\paragraph{Verification of the recovery certificate.}
Figure~\ref{fig:voronoi_wasserstein_threshold} compares the empirical
Wasserstein error $W_\rho(\widehat\mu_T,\mu)$ against the theorem threshold
$\Gamma/2$. As sample size increases, the Wasserstein error decreases. Once
\[
W_\rho(\widehat\mu_T,\mu)<\frac{\Gamma}{2},
\]
Theorem~\ref{thm:voronoi_recovery} certifies that the empirical semantic-cell
assignment must equal the population assignment. This verifies the theorem's
central mechanism: semantic recovery follows from Wasserstein convergence plus
a positive assignment margin.

\paragraph{Empirical margin convergence.}
Figure~\ref{fig:voronoi_margin_convergence} plots the empirical margin
\[
\widehat\Gamma_T
=
\mathcal Q_2(\widehat\mu_T)-\mathcal Q_1(\widehat\mu_T)
\]
against the population margin
\[
\Gamma
=
\mathcal Q_2(\mu)-\mathcal Q_1(\mu).
\]
The empirical margin concentrates around $\Gamma$ as $T$ increases. This gives
a direct empirical view of why the assignment stabilizes: the empirical
distribution may fluctuate at small sample sizes, but the gap between the
assigned center and the nearest competing center converges to the positive
population margin.

\begin{figure}[htp]
    \centering
    \includegraphics[width=1\linewidth]{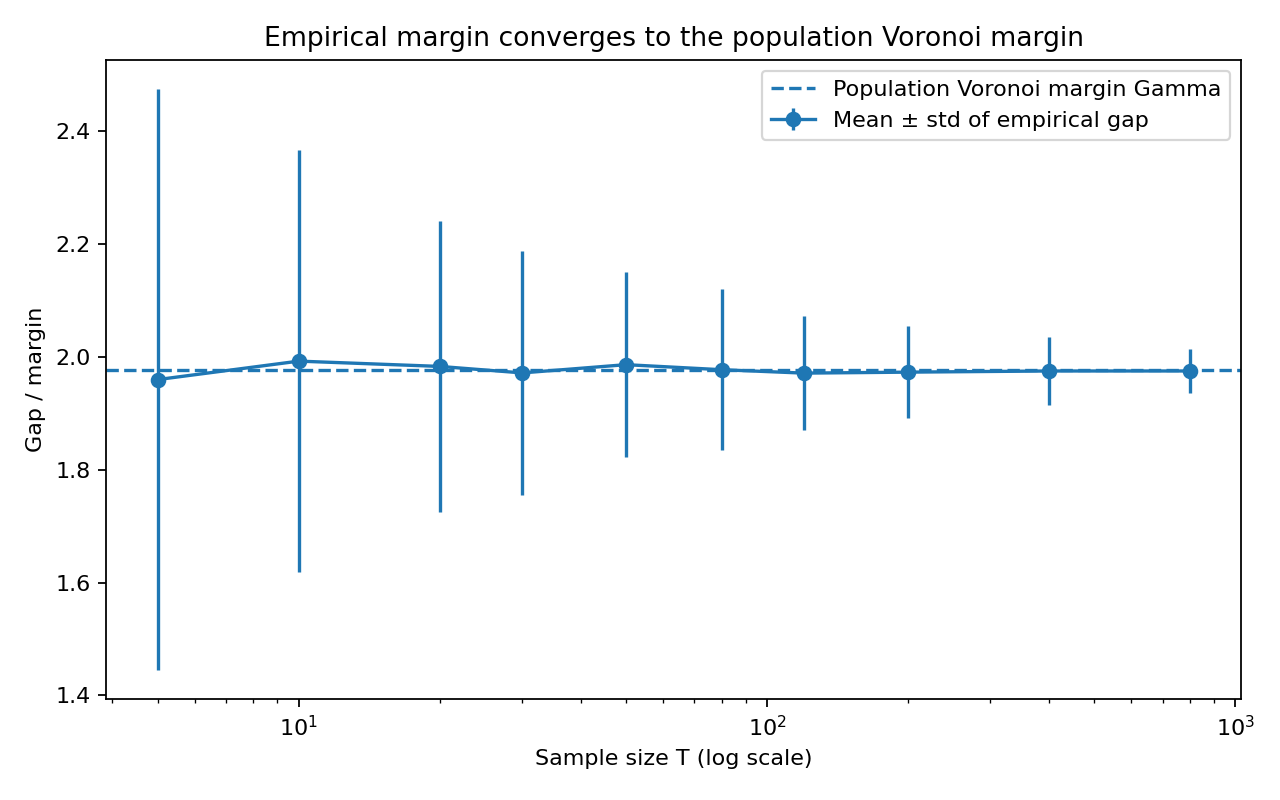}
    \caption{
\textbf{Convergence of the empirical Voronoi margin.}
The dashed horizontal line shows the population Voronoi margin
$\Gamma_i=\mathcal Q_2(\mu_i^{(A\to\mathcal R)})
-\mathcal Q_1(\mu_i^{(A\to\mathcal R)})$, where $\xi_1$ is the correct semantic
center and $\xi_2$ is the nearest competing semantic center. The solid curve
shows the empirical margin
$\mathcal Q_2(\widehat\mu_{i,T}^{(A\to\mathcal R)})
-\mathcal Q_1(\widehat\mu_{i,T}^{(A\to\mathcal R)})$ averaged over Monte Carlo
trials, with error bars indicating one standard deviation. As the sample size
$T$ increases, the empirical margin concentrates around the population margin.
This illustrates why semantic assignment becomes stable with more samples: once
the empirical margin remains positive and sufficiently far from zero, the
correct Voronoi cell is preserved.
}
    \label{fig:voronoi_margin_convergence}
\end{figure}

\paragraph{Conclusion.}
The synthetic experiment confirms the qualitative behavior predicted by
Theorem~\ref{thm:voronoi_recovery}. Larger sample sizes reduce
$W_\rho(\widehat\mu_T,\mu)$, the empirical assignment scores converge to their
population scores, and the empirical margin converges to the population
Voronoi margin. Therefore, once the empirical Wasserstein error falls below the
stability radius $\Gamma/2$, the semantic-cell assignment is certified to
recover the population assignment. This supports the use of margin-aware
semantic grouping for circuit compression: features with large Voronoi margins
are stable under finite-sample perturbations, while features near semantic
boundaries require more samples or should be treated as ambiguous.


\end{document}